\documentclass[11pt]{article}

\usepackage{ragged2e}
\usepackage{lmodern}
\usepackage[utf8]{inputenc} 
\usepackage[T1]{fontenc}    
\usepackage{amsfonts}       
\usepackage{nicefrac}       
\usepackage{microtype}      
\usepackage{fullpage}
\usepackage{appendix}
\usepackage{mystyle}

\usepackage{setspace}
\title{Contrastive UCB: Provably Efficient Contrastive Self-Supervised Learning in Online Reinforcement Learning}

\begin{document}



\author{Shuang Qiu\thanks{University of Chicago. 
Email: \texttt{qiush@umich.edu}.} 
       \quad
        Lingxiao Wang\thanks{Northwestern University.
Email: \texttt{lingxiaowang2022@u.northwestern.edu}.}    	\quad
		Chenjia Bai\thanks{Shanghai AI Laboratory. 
    Email: \texttt{baichenjia255@gmail.com}.}
	\quad    
       	Zhuoran Yang\thanks{Yale University.
    Email: \texttt{zhuoran.yang@yale.edu}.}
    \quad 
    Zhaoran Wang\thanks{
   Northwestern University. 
	Email: \texttt{zhaoranwang@gmail.com}.}
}

\maketitle

\begin{abstract}
In view of its power in extracting feature representation, contrastive self-supervised learning has been successfully integrated into the practice of (deep) reinforcement learning (RL), leading to efficient policy learning in various applications. 
Despite its tremendous empirical successes, the understanding of contrastive learning for RL  remains elusive. 
To narrow such a gap, we study how RL can be empowered by contrastive learning in a class of Markov decision processes (MDPs) and Markov games (MGs) with low-rank transitions. For both models, we propose to extract the correct feature representations of the low-rank model by minimizing a contrastive loss. Moreover, under the online setting, we propose novel upper confidence bound (UCB)-type algorithms that incorporate such a contrastive loss with online RL algorithms for MDPs or MGs. 
We further theoretically prove that our algorithm recovers the true representations and simultaneously achieves sample efficiency in learning the optimal policy and Nash equilibrium in MDPs and MGs.  We also provide empirical studies to demonstrate the efficacy of 
the UCB-based contrastive learning method for RL. 
To the best of our knowledge, we provide the first provably efficient online RL algorithm that incorporates contrastive learning for representation learning. Our codes are available at \url{https://github.com/Baichenjia/Contrastive-UCB}.
\end{abstract}

\section{Introduction}

Deep reinforcement learning (DRL) has achieved great empirical successes in various real-world decision-making problems (e.g., \citet{mnih2015human,silver2016mastering,silver2017mastering,sallab2017deep,sutton2018reinforcement,silver2018general,vinyals2019grandmaster}). A key to the success of DRL is the superior representation power of the neural networks, which extracts the effective information from raw input pixel states. Nevertheless, learning such effective representation of states typically demands millions of interactions with the environment, which limits the usefulness of RL algorithms in domains where the interaction with environments is expensive or prohibitive, such as healthcare \citep{yu2021reinforcement} and autonomous driving \citep{kiran2021deep}.

To improve the sample efficiency of RL algorithms, recent works propose to learn low-dimensional representations of the states via solving auxiliary problems \citep{jaderberg2016reinforcement,hafner2019dream,hafner2019learning,gelada2019deepmdp,franccois2019combined,bellemare2019geometric,srinivas2020curl,zhang2020learning,liu2021return,yang2021representation,stooke2021decoupling}. Among the recent breakthroughs in representation learning for RL, contrastive self-supervised learning gains popularity for its superior empirical performance \citep{oord2018representation,sermanet2018time,dwibedi2018learning,anand2019unsupervised,schwarzer2020data,srinivas2020curl,liu2021return}. A typical paradigm for such contrastive RL is to construct an auxiliary contrastive loss for representation learning, add it to the loss function in RL, and deploy an RL algorithm with the learned representation being the state and action input. However, the theoretical underpinnings of such an enterprise remain elusive. To summarize, we raise the following question:


\begin{center}
\emph{Can contrastive self-supervised learning provably improve the sample efficiency of RL via representation learning?}
\end{center}
To answer such a question, we face two challenges. First, in terms of the algorithm design, it remains unclear how to integrate contrastive self-supervised learning into provably efficient online exploration strategies, such as exploration with the upper confidence bound (UCB), in a principled fashion. 
Second, in terms of theoretical analysis, 
it also remains unclear how to analyze the sample complexity of such an integration of self-supervised learning and RL. 
Specifically, to establish theoretical guarantees for such an approach, we need to (i) characterize the accuracy of the representations learned by minimizing a contrastive loss computed based on adaptive data collected in RL, and 
(ii) understand how the error of representation learning affects the efficiency of exploration.  
In this work, we take an initial step towards tackling such challenges by proposing a reinforcement learning algorithm where the representations are learned via temporal contrastive self-supervised learning \citep{oord2018representation,sermanet2018time}. 
Specifically, our algorithm iteratively solves a temporal contrastive loss to obtain the state-action representations and then constructs a UCB bonus based on such representations to explore in a provably efficient way. 
As for theoretical results, we prove that the proposed algorithm provably recovers the true representations under the low-rank MDP setting. Moreover, we show that our algorithm achieves a $\tilde\cO(1/\varepsilon^2)$ sample complexity for attaining the $\varepsilon$-approximate optimal value function, where $\tilde\cO(\cdot)$ hides logarithmic factors. 
Therefore, our theory shows that contrastive self-supervised learning yields accurate representation in RL, and these learned representations provably enables efficient exploration.
In addition to theoretical guarantees, we also provide numerical experiments to empirically demonstrate the efficacy of our algorithm. 
Furthermore,  we extend the algorithm and theory to the zero-sum MG under the low-rank setting, a multi-agent extension of MDPs to a competitive environment. 
Specifically, in the competitive setting, our algorithm constructs upper and lower confidence bounds (ULCB) of the value functions based on the representations learned via contrastive learning. 
We prove that the proposed approach achieves an $\tilde{\cO}(1/\varepsilon^2)$ sample complexity to attain an $\varepsilon$-approximate Nash equilibrium. 
To the best of our knowledge, we propose the first provably efficient online RL algorithms that employ contrastive learning for representation learning. 
Our major contributions are summarized as follows:

\vspace{5pt}
\noindent\textbf{Contribution.}  
Our contributions are three-fold. First, We show that contrastive self-supervised learning recovers the underlying true transition dynamics, which reveals the benefit of incorporating representation learning into RL in a provable way.
Second, we propose the first provably efficient exploration strategy incorporated with contrastive self-supervised learning. Our proposed UCB-based method is readily adapted to existing representation learning methods for RL, which then demonstrates improvements over previous empirical results as shown in our experiments.
Finally, we extend our results to the zero-sum MG, which reveals a potential direction of utilizing the contrastive self-supervised learning for multi-agent RL.

\vspace{5pt}
\noindent\textbf{Related Work.}
Our work is closely related to the line of research on RL with low-rank transition kernels, which assumes that the transition dynamics take the form of an inner product of two unknown feature vectors for the current state-action pair and the next state (see Assumption \ref{assump:low-rank} for details) \citep{jiang2017contextual,agarwal2020flambe,uehara2021representation}. 
In contrast, as a special case of the low-rank model, linear MDPs have a similar form of structures but with an extra assumption that the linear representation is known a priori \citep{du2019good,yang2019sample,jin2020provably,xie2020learning,ayoub2020model,cai2020provably,yang2020reinforcement,chenalmost,zhou2021nearly,zhou2021provably}. Our work focuses on the more challenging low-rank setting and aims to recover the unknown state-action representation via contrastive self-supervised learning. Our theory is motivated by the recent progress in low-rank MDPs \citep{agarwal2020flambe,uehara2021representation}, which show that the transition dynamics can be effectively recovered via maximum likelihood estimation (MLE). In contrast, our work recovers the representation via contrastive self-supervised learning. 
Upon acceptance of our work, we notice a concurrent work \citep{zhang2022making} studies contrastive learning in RL on linear MDPs.

There is a large amount of literature studying contrastive learning in RL empirically. To improve the sample efficiency of RL, previous empirical works leverages different types of information for representation learning, e.g.,  temporal information \citep{sermanet2018time,dwibedi2018learning,oord2018representation,anand2019unsupervised,schwarzer2020data}, local spatial structure\citep{anand2019unsupervised}, image augmentation\citep{srinivas2020curl}, and return feedback\citep{liu2021return}.  Our work follows the utilization of contrastive learning for RL to extract temporal information. Similar to our work, recent work by \citet{misra2020kinematic} shows that contrastive learning provably recovers the latent embedding under the restrictive Block MDP setting \citep{du2019provably}. In contrast, our work analyzes contrastive learning in RL under the more general low-rank setting, which includes Block MDP as a special case \citep{agarwal2020flambe} for both MDPs and MGs.

\section{Preliminaries}
In this section, we introduce the backgrounds of single-agent MDPs, zero-sum MGs, and the low-rank assumption.
 
\vspace{5pt}
\noindent\textbf{Single-Agent MDP.} An episodic single-agent MDP is defined by $(\cS, \cA, H , r, \PP)$, where $\cS$ is the state space, $\cA$ is the action space, $H$ is the length of an episode, $r = \{r_h\}_{h=1}^H$ is the reward function with $r_h: \cS\times\cA\mapsto[0,1]$, and $\PP=\{\PP_h\}_{h=1}^H$ denotes the transition model with $\PP_h(s'|s,a)$ being the probability density of an agent transitioning to $s'\in\cS$ from state $s\in \cS$ after taking action $a\in\cA$ at the step $h$. Specifically,  $\cS$ can be an infinite state space\footnote{We assume that the volume (Lebesgue measure) of the infinite state space $\cS$ satisfies $\mathrm{Vol}(\cS) \leq c$, where $\mathrm{Vol}(\cdot)$ denotes the volume of a space. WOLG, we let $c = 1$ for simplicity. } and the action space $\cA$ is assumed to be finite with the size of $|\cA|$. A deterministic policy is denoted as $\pi = \{\pi_h\}_{h=1}^H$ where $\pi_h: \cS\mapsto \cA$ is the map from the agent's state $s$ to an action $a$ at the $h$-th step.
We further denote the policy learned at the $k$-th episode by $\pi^k = \{\pi_h^k\}_{h=1}^H$. For simplicity, assume the initial state is fixed as $s_1^k = s_1$ for any episode $k$.
 
For the single-agent MDP, for any $(s,a) \in \cS\times \cA$, we define the associated Q-function and value function as $Q_h^\pi(s,a) = \EE[\sum_{h'=h}^H r_{h'}(s_{h'},a_{h'}) \given s_h=s, a_h=a,\pi, \PP]$ and $V_h^\pi(s) = \EE[\sum_{h'=h}^H r_{h'}(s_{h'},a_{h'}) \given \allowbreak s_h=s,\pi, \PP]$. Then, we further have the Bellman equation as $Q_h^\pi(s,a) = r_h(s,a) + \PP_hV_{h+1}^\pi(s,a)$ and $V_h^\pi(s) = Q_h^\pi(s,\pi_h(s))$ where, for the ease of notation, we denote $\PP_hV(s,a) = \int_{s'}\PP_h(s'|s,a)V(s')\mathrm{d}s'$ for any value function $V$. Moreover, we define the \emph{optimal policy} as $\pi^*:=\argmax_{\pi} V_1^\pi(s_1)$. We say a policy $\pi$ is an \emph{$\varepsilon$-approximate optimal policy} if 
\begin{align*}
V_1^{\pi^*}(s_1) - V_1^{\pi}(s_1) \leq \varepsilon.
\end{align*}

\noindent\textbf{Zero-Sum Markov Game.} Our work further studies the zero-sum two-player Markov game that can be defined by $(\cS, \cA, \cB, H, r, \PP)$, where $\cS$ is the infinite state space with $\mathrm{Vol}(\cS) \leq 1$, $\cA$ and $\cB$ are the finite action spaces for two players with the sizes of $|\cA|$ and $|\cB|$, $H$ is the length of an episode, $r = \{r_h\}_{h=1}^H$ is the reward function with $r_h: \cS\times\cA\times \cB\mapsto[-1,1]$, and $\PP=\{\PP_h\}_{h=1}^H$ denotes the transition model with $\PP_h(s'|s,a,b)$ being the probability density of the two players transitioning to $s'\in\cS$ from state $s\in \cS$ after taking action $a\in\cA$ and $b\in\cB$ at step $h$. The policies of the two players are denoted as $\pi = \{\pi_h\}_{h=1}^H$ and $\nu = \{\nu_h\}_{h=1}^H$, where $\pi_h(a|s)$ and $\nu_h(b|s)$ are the probabilities of taking actions $a\in \cA$ or $b\in \cB$ at the state $s\in \cS$. Moreover, we denote $\sigma=\{\sigma_h\}_{h=1}^H$ as a joint policy, where $\sigma_h(a,b|s)$  is the probability of taking actions $a\in \cA$ and $b\in \cB$ at the state $s\in \cS$. Note that the actions $a$ and $b$ are not necessarily mutually independent conditioned on state $s$.  One special case of a joint policy is the product of a policy pair $\pi \times \nu$. Here we also assume the initial state is fixed as $s_1^k = s_1$ for any episode $k$. The Markov game is a multi-agent extension of the MDP model under a competitive environment.

For any $(s,a,b) \in \cS\times \cA\times \cB$ and joint policy $\sigma$, we define the Q-function and value function as $Q_h^\sigma(s,a,b) = \EE[\sum_{h'=h}^H r_{h'}(s_{h'},a_{h'},b_{h'}) \given s_h=s,  a_h=a, b_h=b , \sigma, \PP]$ and $V_h^\sigma(s) = \EE[\sum_{h'=h}^H r_{h'}(s_{h'},a_{h'},b_{h'}) \allowbreak \given s_h=s,\sigma, \PP]$. We have the Bellman equation as 
$Q_h^{\sigma}(s,a,b) \allowbreak = r_h(s,a,b) + \PP_hV_{h+1}^{\sigma}(s,a,b)$ and $V_h^{\sigma}(s) = \langle \sigma_h(\cdot,\cdot|s), Q_h^{\sigma}(s,\cdot,\cdot) \rangle$.
We denote $\PP_hV(s,a,b) = \int_{s'}\PP_h(s'|s,a,b)V(s')\mathrm{d}s'$ for any value function $V$. We say $(\pi^\dag,\nu^\dag)$ is a \emph{Nash equilibrium (NE)} if it is a solution to the max-min optimization problem $\max_\pi \min_\nu V_1^{\pi,\nu}(s_1)$. Then, $(\pi,\nu)$ is an \emph{$\varepsilon$-approximate NE} if  it satisfies
\begin{align*}
\max_{\pi'} V_1^{\pi',\nu}(s_1) - \min_{\nu'} V_1^{\pi,\nu'}(s_1) \leq \varepsilon.
\end{align*}
In addition, we denote $\mathrm{br}(\cdot)$ as the best response, which is defined as $\mathrm{br}(\nu) =\argmax_{\pi} V_1^{\pi, \nu}(s_1)$ and $\mathrm{br}(\pi) =\argmin_{\nu} V_1^{\pi, \nu}(s_1)$.

\vspace{5pt}
\noindent\textbf{Low-Rank Transition Kernel.} 
In this paper, we consider the low-rank structures with the dimension $d$ \citep{jiang2017contextual,agarwal2020flambe,uehara2021representation} for both single-agent MDPs and Markov games, in which the transition model admits the structure in the following assumption. To unify both settings, with a slight abuse of notation, we let $\cZ := \cS\times\cA$ for single-agent MDPs and $\cZ := \cS\times\cA \times \cB$ for Markov games. 
\begin{assumption}[Low-Rank Transition Kernel]\label{assump:low-rank} Assuming there exist two unknown maps $\psi^*:\cS\mapsto\RR^d$ and $\phi^*:\cZ\mapsto\RR^d$, 
the true transition kernel admits the following low-rank decomposition for all $h\in [H]$, $(z,s')\in \cZ \times\cS$,
    \begin{align*}
        \PP_h(s'|z) = \psi_h^*(s')^\top\phi_h^*(z),
    \end{align*}
    where $\|\phi_h^*(z)\|_2\leq 1$ and $\|\psi_h^*(s')\|_2\leq \sqrt{d}$.
\end{assumption}

\begin{remark}
In contrast to linear MDPs \citep{jin2020provably} or linear Markov games \citep{xie2020learning} where $\phi_h^*$ is known \emph{a priori}, we adopt the more challenging setting that both $\psi_h^*$ and $\phi_h^*$ are unknown and hence should be identified via contrastive learning. Moreover, our work also extends the scenario of low-rank transition model from single-agent RL \citep{jiang2017contextual,agarwal2020flambe,uehara2021representation} to the multi-agent competitive RL.
\end{remark}


\section{Contrastive Learning for Single-Agent MDP}

\subsection{Algorithm} \label{sec:alg-mdp}

 \begin{algorithm}[t]\caption{Online Contrastive RL for Single-Agent MDPs} \label{alg:contrastive_RL}
  \setstretch{1.1}
    \begin{small}   
	\begin{algorithmic}[1]
		\State {\bfseries Initialize:} $\pi_h^0(a|s) = 1/|\cA|, \forall (s,a)\in \cS\times\cA$. $\cD_h^0 = \emptyset, \forall h\in [H]$. $\delta > 0$, $\beta > 0$, and $\varepsilon > 0$.  
		\For{episode $k=1,\ldots,K$}   
		    
		   	\State Let $V_{H+1}^{k}(\cdot) = \boldsymbol 0$ and $Q_{H+1}^{k}(\cdot, \cdot) = \boldsymbol 0$
	        \State Collect bonus data $\{\tilde{\cD}_h^k = \{(\tilde{s}_h^\tau,\tilde{a}_h^\tau)\}_{\tau=1}^k\}_{h=1}^H$ and contrastive training data $\{\cD_h^k\}_{h=1}^H$ by Alg. \ref{alg:sample}.

		        \For{step $h=H, H-1,\ldots, 1$} 
                \State Obtain $\tilde{\phi}_h^k$ and $\tilde{\psi}_h^k$ by solving \eqref{eq:solve-contra-loss} with $\cD_h^k$. 
                \State Normalize $\tilde{\phi}_h^k$ and $\tilde{\psi}_h^k$ by \eqref{eq:normalization} to obtain $\hat{\phi}_h^k$ and $\hat{\psi}_h^k$.
                \State Estimate $\PP_h$ by $\hat{\PP}_h^k(\cdot|\cdot,\cdot) = \hat{\psi}_h^k(\cdot)^\top\hat{\phi}_h^k(\cdot,\cdot)$.
				\State $\hat{\Sigma}_h^k=\frac{1}{k}\sum_{\tau=1}^k\hat{\phi}^k_h(\tilde{s}_h^\tau,\tilde{a}_h^\tau)\hat{\phi}^k_h(\tilde{s}_h^\tau,\tilde{a}_h^\tau)^\top + \lambda_k I$.
					\State Bonus $\beta_h^k(\cdot,\cdot) = \min\{\gamma_k \|\hat{\phi}^k_h(\cdot,\cdot)\|_{(\hat{\Sigma}_h^k)^{-1}}, 2H\}$.
				\State $\overline{Q}_h^k(\cdot,\cdot) =   (r_h + \beta_h^k +\hat{\PP}_h^k\overline{V}_{h+1}^k ) (\cdot, \cdot)$.
				
				\State $\overline{V}_h^k(\cdot) =  \max_{a\in \cA} \overline{Q}_h^k(\cdot, a)$.
				\State $\pi_h^k(\cdot) =  \argmax_{a\in \cA} \overline{Q}_h^k(\cdot, a)$.
	            \EndFor 	
    \EndFor             
	\end{algorithmic}\label{alg:contrastive}
\end{small}
\end{algorithm}

\textbf{Algorithmic Framework.} We propose an online UCB-type contrastive RL algorithm, Contrastive UCB, for MDPs in Algorithm \ref{alg:contrastive}. At the $k$-th episode, we execute the learned policy from the last round to collect the datasets $\{\tilde{\cD}_h^k\}_{h=1}^H$ and $\{\cD_h^k\}_{h=1}^H$ as bonus construction data and the contrastive learning data according to the sampling strategy in Algorithm \ref{alg:sample}. Specifically, the contrastive learning sample is composed of positive and negative data points. At 
a state-action pair $(s_h, a_h)$ that is sampled independently following a certain distribution formed by the current policy and the true transition, with probability $1/2$, we collect the positive transition data point as $(s_h,a_h,s_{h+1}, 1)$ with $s_{h+1}\sim \PP_h(\cdot|s_h,a_h)$ and a label $y=1$. On the other hand, with probability $1/2$, we generate the negative transition data point as $(s_h,a_h,s_{h+1}^{-}, 0)$ with $s_{h+1}^{-}\sim \cPS(\cdot)$ and a label $y=0$, where $\cPS(\cdot)$ is a designed negative sampling distribution. Given the data sample for contrastive learning $\{\cD_h^k\}_{h=1}^H$, we propose to solve the minimization problem \eqref{eq:solve-contra-loss} at each step $h$ with $\cL_h(\psi, \phi ;\cD_h^k)$ denoting the contrastive loss defined in \eqref{eq:contra-loss} to learn the low-rank representation $\tilde{\phi}_h^k$ and $\tilde{\psi}_h^k$. More detailed implementation of data sampling and the contrastive loss will be elaborated below. According to our analysis in Section \ref{sec:mdp-theory}, the true transition kernel $\PP_h(s'|s,a)$ can be well approximated by the learned representation $\tilde{\phi}_h^k(s')^\top\tilde{\psi}_h^k(s,a)\cPS(s')$. However, such learned features are not guaranteed to satisfy the relation $\int_{s'\in \cS}\tilde{\phi}_h^k(s')^\top\tilde{\psi}_h^k(s,a)\cPS(s')\mathrm{d}s' = 1$ or $\tilde{\phi}_h^k(\cdot)^\top\tilde{\psi}_h^k(s,a)\cPS(\cdot)$ may not be a distribution over $\cS$. Thus, we further normalize learned representations by  
\begin{align}
\begin{aligned}\label{eq:normalization}
    &\hat{\psi}_h^k(s') :=\cPS(s')  \tilde{\psi}_h^k (s'),\\
    &\hat{\phi}_h^k(z) :=   \tilde{\phi}_h^k(z)\big/\textstyle \int_{s'\in\cS}\cPS(s') \tilde{\phi}_h^k(z)^\top \tilde{\psi}_h^k (s')\mathrm{d}s',
\end{aligned}
\end{align}
where $z = (s,a)$. Then, we obtain an approximated transition kernel $\hat{\PP}_h^k(\cdot|s,a) := \hat{\psi}_h^k(\cdot)^\top\hat{\phi}_h^k(s,a)$. Our analysis in Section \ref{sec:mdp-theory} shows that $\hat{\PP}_h^k(\cdot|s,a)$ lies in a probability simplex and can well approximate the true transition $\PP_h(\cdot|s,a)$.

Simultaneously, we construct the UCB bonus term $\beta_h^k$ with the learned representation $\hat{\phi}_h^k$  and the empirical covariance matrix $\hat{\Sigma}_h^k$ using the bonus construction data sampled online via Algorithm \ref{alg:sample}. Then, with the estimated transition $\hat{\PP}_h^k$ and the UCB bonus term $\beta_h^k$, we obtain a UCB estimation of the Q-function and value function in Line 11 and Line 12. The policy $\pi_h^k$ is then the greedy policy corresponding to the estimated Q-function $\overline{Q}_h^k$. 
\begin{remark}
To focus our analysis on the contrastive learning for the transition dynamics, we only consider the setting where the reward function $r_h(\cdot, \cdot)$ is known. One might further modify the proposed algorithm to the unknown reward setting under the linear reward function assumption by considering to minimize a square loss with observed rewards as the regression target to learn the parameters. The corresponding analysis would then take the statistical error of such a procedure into consideration. 
\end{remark}

\noindent\textbf{Dataset for Contrastive Learning.} 
For our algorithm, we make the following assumption for the negative sampling distribution $\cPS(\cdot)$.
\begin{assumption} [Negative Sampling Distribution] \label{assump:negative} Let $\cPS(\cdot)$ be a distribution over $\cS$. The distribution $\cPS(\cdot)$ satisfies $\inf_{s\in\cS}\cPS(s) \geq \CS> 0$ for a constant $\CS$. 
\end{assumption}

The detailed sampling scheme for the contrastive learning dataset is presented in Algorithm \ref{alg:sample} in Appendix. Here we provide a brief idea of this algorithm. Letting $d^{\pi}_h(\cdot)$ be the state distribution at step $h$ under the true transition $\PP$ and a policy $\pi$, we define two state-action distributions induced by $\pi$ and $\PP$ at step $h$ as $\tilde{d}^{\pi}_h(s,a) = d^{\pi}_h(s) \Unif(a)$ and $\breve{d}^{\pi}_h(s,a) = \tilde{d}^{\pi}_{h-1}(s',a')\PP_{h-1}(s|s',a')\Unif(a)$, where $\Unif(a) =1/|\cA|$. Then, at each round $k$, we sample the temporal data as follows: 
\begin{itemize}[leftmargin=*,itemsep=0pt,parsep=1pt]
\item Sample $(\tilde{s}_h^k, \tilde{a}_h^k)\sim\tilde{d}^{\pi^{k-1}}_h(\cdot,\cdot)$ for all $h\in [H]$ and $(\breve{s}_h^k, \breve{a}_h^k)\sim\breve{d}^{\pi^{k-1}}_h(\cdot,\cdot)$ for all $h\geq2$.
\item For each $(\tilde{s}_h^k, \tilde{a}_h^k)$ or $(\breve{s}_h^k, \breve{a}_h^k)$, generate a label $y\in\{0,1\}$ from a Bernoulli distribution $\Ber(1/2)$ independently.
\item Sample the next state from the true transition as $\tilde{s}_{h+1}^k \sim \PP_h(\cdot|\tilde{s}_h^k, \tilde{a}_h^k)$ or $\breve{s}_{h+1}^k \sim \PP_h(\cdot|\breve{s}_h^k, \breve{a}_h^k)$ when the associated labels are $1$ and sample negative transition data points by $\tilde{s}_{h+1}^{k,-} \sim  \cPS(\cdot)$ or $\breve{s}_{h+1}^{k,-} \sim  \cPS(\cdot)$ if labels are $0$.
\item Given the dataset $\cD_h^{k-1}$ from the last round, add the new transition data with labels, i.e., $(\tilde{s}_h^k,\tilde{a}_h^k,\tilde{s}_{h+1}^k, 1)$ or $(\tilde{s}_h^k,\tilde{a}_h^k,\tilde{s}_{h+1}^{k,-}, 0)$ and $(\breve{s}_h^k, \breve{a}_h^k,\breve{s}_{h+1}^k,1)$ or $(\breve{s}_h^k, \breve{a}_h^k,\breve{s}_{h+1}^{k,-},0)$,  into it  to compose a new set $\cD_h^k$.  
\end{itemize}

In addition, we also build a dataset $\tilde{\cD}_h^k$ via Algorithm \ref{alg:sample} for the construction of the UCB bonus term in Algorithm \ref{alg:contrastive}, where $\tilde{\cD}_h^k$ is composed of the present and historical state-action pairs sampled from $\tilde{d}^{\pi^{k'}}_h(\cdot,\cdot)$ for all $k'\in [0,k-1]$. Algorithm \ref{alg:sample} illustrates how to sample the above data by interacting with the environment in an online manner, which can also guarantee the data points are mutually independent within $\cD_h^k$ and $\tilde{\cD}_h^k$.

\vspace{5pt}
\noindent\textbf{Contrastive Loss.}
Given the dataset $\{\cD_h^k\}_{h=1}^H$ for contrastive learning, we further define the following contrastive loss for each step $h\in [H]$
\begin{align}
\begin{aligned}\label{eq:contra-loss}
    \cL_h(\psi, \phi ;\cD_h^k)&:= \EE_{\cD_h^k} \big[y\log(1+1/\psi(s')^\top\phi(z)) + (1-y)\log(1+\psi(s')^\top\phi(z))\big],
\end{aligned}
\end{align}
where $z = (s,a)$ and $\EE_{\cD_h^k}$ indicates taking average over all $(s,a,s',y)$ in the collected contrastive training dataset $\cD_h^k$. Here $\phi$ and $\psi$ are two functions lying in the function classes $\Phi$ and $\Phi$ as defined below. Letting $\cZ = \cS\times\cA$, we define:
\begin{definition}[Function Class] \label{def:func-class}
Let $\cF:=\{\psi(\cdot)^\top\phi(\cdot,\cdot) :  \psi\in\Psi, \phi\in\Phi\}$ be a function class where $\Psi:=\{\phi:\cS\mapsto\RR^d\}$ and $\Phi:=\{\psi:\cZ\mapsto\RR^d\}$ are two finite function classes. For any $\psi\in \Psi$, $\sup_{s\in\cS}\|\psi(s)\|_2 \leq \sqrt{d}/\CS$. And for any $\phi\in \Phi$, $\sup_{s\in\cS}\|\phi(z)\|_2 \leq 1$. The cardinality of $\cF$ is $|\cF| =|\Psi|\cdot |\Phi|$.
\end{definition}
The fundamental idea of designing  \eqref{eq:contra-loss} is to consider a negative log-likelihood loss for the probability $\Pr_h(y|s,a,s'):=\Big(\frac{f_h(s,a,s')}{1+f_h(s,a,s')}\Big)^{y}\Big(\frac{1}{1+f_h(s,a,s')}\Big)^{1-y}$ where $f_h(s,a,s')=\psi(s')^\top\phi(s,a)$ and $\Pr_h$ denote the associated probability at step $h$. Then \eqref{eq:contra-loss} is equivalent to $\cL_h(\psi, \phi ;\cD_h^k)=-\EE_{\cD_h^k}[\log \Pr_h(y|s,a,s')]$. Thus, to learn the contrastive feature representation, we seek to solve the following problem of contrastive loss minimization
\begin{align}\label{eq:solve-contra-loss}
 \big(\tilde{\psi}_h^k, \tilde{\phi}_h^k\big)= \argmin_{\psi\in \Psi, \phi\in \Phi}   \cL_h(\psi, \phi ;\cD_h^k).
\end{align}
According to Lemma \ref{lem:opt-contra-loss} in Appendix, letting $z=(s,a)$, the learning target of the above minimization problem is 
\begin{align}\label{eq:targt}
f^*_h(z,s') = \PP_h(s'|z)/\cPS(s').
\end{align}
Since $\PP_h(s'|z) = \psi_h^*(s')^\top\phi_h^*(z)$ with $\|\phi_h^*(z)\|_2\leq 1$ and $\|\psi_h^*(s')\|_2\leq \sqrt{d}$ as in Assumption \ref{assump:low-rank}, by Definition \ref{def:func-class}, we know $f^*_h\in \cF$, i.e., $\psi_h^*(\cdot)/\cPS(\cdot)\in \Psi$ and $\phi_h^*(\cdot) \in \Phi$.

\begin{remark}
The parameter $C^{-}_{\mathcal{S}}$ in Assumption \ref{assump:negative} captures the fundamental difficulty of contrastive learning in RL by characterizing how large the function class (Definition \ref{def:func-class}) should be to include the underlying true density ratio in \eqref{eq:targt}. Technically, it also guarantees that the problem is mathematically well-defined. In particular, the true density ratio \eqref{eq:targt} has non-zero denominator $\mathcal{P}_\mathcal{S}^-(s),\forall s\in \mathcal{S}$ if the parameter $C^{-}_{\mathcal{S}}$ is positive.
\end{remark}

\begin{remark}
One can further extend the setting of the finite function class to the infinite function class setting by utilizing the covering argument as in \citet{van2000applications,uehara2021pessimistic} such that the terms depending on the cardinality of $\cF$ would be replaced by terms related to the covering number of $\cF$. We leave such an analysis under the online setting as our future work. 
\end{remark}

\subsection{Main Result for Single-Agent MDP Setting}

\begin{theorem}[Sample Complexity]\label{thm:main} Letting
$\lambda_k= c_0 d \log(H|\cF|k/\delta)$ for a sufficiently large constant $c_0>0$ and $\gamma_k=  4H\big(12\sqrt{  |\cA|d} + \sqrt{c_0} d\big)/\CS\cdot \sqrt{\log (2Hk|\cF|/\delta)}$, with probability at least $1-3\delta$, we have
\begin{align*}
&1/K\cdot \textstyle\sum_{k=1}^K \big[ V_1^{\pi^*}(s_1)- V_1^{\pi^k}(s_1)\big] \\
&\qquad \lesssim \sqrt{C\log(H|\cF|K/\delta) \log(c_0' K)/K },
\end{align*}
where $C = H^4d^4|\cA|/(\CS)^2 +  H^4d^3|\cA|^2/(\CS)^2 + H^6d^2|\cA|/(\CS)^2 + H^6d^3$ and $c_0'$ is an absolute  constant.
\end{theorem}

Letting $\hat{\pi}$ be a policy uniformly sampled from $\{\pi^k\}_{k=1}^K$ generated by Algorithm \ref{alg:contrastive}, the above theorem indicates $\hat{\pi}$ is an $\varepsilon$-approximate optimal policy with probability at least $1-3\delta$ after executing Algorithm \ref{alg:contrastive} for $K\geq \tilde{\cO}(1/\varepsilon^2)$ episodes. Here $\tilde{\cO}$ hides logarithmic dependence on $|\cF|, H, K, 1/\delta$, and $1/\varepsilon$.

\section{Contrastive Learning for Markov Game}


\subsection{Algorithm} \label{sec:alg-mg}

 \begin{algorithm}[t]\caption{Online Contrastive RL for Markov Games} 
  \setstretch{1.1}
    \begin{small}
	\begin{algorithmic}[1]
		\State {\bfseries Initialize:} $\sigma_h^0(a,b|s) = 1/(|\cA||\cB|), \forall (s,a,b)\in \cS\times\cA\times \cB$. $\cD_h^0 = \emptyset, \forall h\in [H]$. $\delta > 0$, $\beta > 0$, and $\varepsilon > 0$.  
		\For{episode $k=1,\ldots,K$}   
		    
		   	\State Let $V_{H+1}^{k}(\cdot) = \boldsymbol 0$ and $Q_{H+1}^{k}(\cdot, \cdot, \cdot) = \boldsymbol 0$
	        \State Collect bonus data $\{\tilde{\cD}_h^k = \{(\tilde{s}_h^\tau,\tilde{a}_h^\tau,\tilde{b}_h^\tau)\}_{\tau=1}^k\}_{h=1}^H$ and contrastive training data $\{\cD_h^k\}_{h=1}^H$ by Alg. \ref{alg:sample-mg}.

		        \For{step $h=H, H-1,\ldots, 1$} 
                \State Obtain $\tilde{\phi}_h^k$ and $\tilde{\psi}_h^k$ by solving \eqref{eq:solve-contra-loss} with $\cD_h$. 
                \State Normalize $\tilde{\phi}_h^k$ and $\tilde{\psi}_h^k$ by \eqref{eq:normalization} to obtain $\hat{\phi}_h^k$ and $\hat{\psi}_h^k$.
                \State Estimate $\PP_h$ by $\hat{\PP}_h^k(\cdot|\cdot,\cdot,\cdot) = \hat{\psi}_h^k(\cdot)^\top\hat{\phi}_h^k(\cdot,\cdot,\cdot)$.
				\State $\hat{\Sigma}_h^k=\frac{1}{k}\sum_{\tau=1}^k\hat{\phi}^k_h(\tilde{s}_h^\tau,\tilde{a}_h^\tau,\tilde{b}_h^\tau)\hat{\phi}^k_h(\tilde{s}_h^\tau,\tilde{a}_h^\tau,\tilde{b}_h^\tau)^\top + \nobreak\lambda_k I$
					\State $\beta_h^k(\cdot,\cdot,\cdot) = \min\{ \gamma_k\|\hat{\phi}^k_h(\cdot,\cdot,\cdot)\|_{(\hat{\Sigma}_h^k)^{-1}}, 2H\}$.
				\State $\overline{Q}_h^k(\cdot,\cdot,\cdot) =  (r_h +\hat{\PP}_h^k\overline{V}_{h+1}^k + \beta_h^k) (\cdot, \cdot,\cdot)$.
				\State $\underline{Q}_h^k(\cdot,\cdot,\cdot) =  (r_h +\hat{\PP}_h^k\underline{V}_{h+1}^k - \beta_h^k) (\cdot, \cdot,\cdot)$.		
				\State $\overline{V}_h^k(\cdot) =  \langle\sigma_h^k(\cdot,\cdot|\cdot),  \overline{Q}_h^k(\cdot,\cdot,\cdot) \rangle$.
				\State $\underline{V}_h^k(\cdot) =  \langle\sigma_h^k(\cdot,\cdot|\cdot),  \underline{Q}_h^k(\cdot,\cdot,\cdot) \rangle$.		
				\State $\sigma_h^k(\cdot,\cdot|s) = \iota_k\text{-CCE}(\overline{Q}_h^k(s, \cdot,\cdot),\underline{Q}_h^k(s, \cdot,\cdot)), \forall s$.
				\State $\pi_h^k =  \cP_1\sigma_h^k$ and $\nu_h^k =  \cP_2\sigma_h^k$.
	            \EndFor 	
    \EndFor             
	\end{algorithmic}\label{alg:contrastive-mg}
	\end{small}
\end{algorithm}

\textbf{Algorithmic Framework.} 
We propose an online algorithm, Contrastive ULCB, for contrastive learning on Markov games in Algorithm \ref{alg:contrastive-mg}. At the $k$-th round, we execute the learned joint policy $\sigma^{k-1}$ from the last round to collect the bonus construction data $\{\tilde{\cD}_h^k\}_{h=1}^H$ and the contrastive learning data $\{\cD_h^k\}_{h=1}^H$ via the sampling algorithm in Algorithm \ref{alg:sample-mg}. At a state-action pair $(s_h, a_h, b_h)$ sampled at the $h$-th step, with probability $1/2$ respectively, we collect the positive transition data point $(s_h,a_h,b_h, s_{h+1}, 1)$ with $s_{h+1}\sim \PP(\cdot|s_h, a_h, b_h)$ and the negative transition data point $(s_h,a_h,s_{h+1}^{-}, 0)$ with $s_{h+1}^{-}\sim \cPS(\cdot)$, where $\cPS(\cdot)$ is the negative sampling distribution. Given the dataset $\{\cD_h^k\}_{h=1}^H$ for contrastive learning, we define the contrastive loss $\cL_h(\psi, \phi ;\cD_h^k)$ as in \eqref{eq:contra-loss} with setting $z=(s,a,b)$. The function class $\cF$ is then defined the same as in Definition \ref{def:func-class} by setting $z=(s,a,b)$. We solve the contrastive loss minimization problem as \eqref{eq:solve-contra-loss} at each step $h$ to learn the representation $\tilde{\phi}_h^k$ and $\tilde{\psi}_h^k$. 
Since it is not guaranteed that $\tilde{\phi}_h^k(\cdot)^\top\tilde{\psi}_h^k(s,a,b)\cPS(\cdot)$ is a distribution over $\cS$, we normalize $\tilde{\phi}_h^k$ and $\tilde{\psi}_h^k$ as 
\eqref{eq:normalization}
where $z = (s,a,b)$. Then we obtain an approximated transition kernel $\hat{\PP}_h^k(\cdot|s,a,b) := \hat{\psi}_h^k(\cdot)^\top\hat{\phi}_h^k(s,a,b)$. Furthermore, we use the bonus dataset to construct the empirical covariance matrix $\hat{\Sigma}_h^k$ and then the bonus term $\beta_h^k$.  The major differences between algorithms for single-agent MDPs and Markov games lie in the following two steps: \textbf{(1)} In Lines 11 and 12, we have two types of Q-functions with both addition and subtraction of bonus terms such that Algorithm \ref{alg:contrastive-mg} is an upper and lower confidence bound (ULCB)-type algorithm. \textbf{(2)} We update policies of two players by first finding an $\iota_k$-coarse correlated equilibrium (CCE) with the two Q-functions as a joint policy $\{\sigma_h^k\}_{h=1}^H$ in Line 15 and then applying marginalization to obtain the policies as in Line 16, where $\cP_1$ and $\cP_2$ denote getting marginal distributions over $\cA$ and $\cB$ respectively. In particular, the notion of an $\iota$-CCE \citep{moulin1978strategically,aumann1987correlated} is defined as follows:

\begin{definition}[$\iota$-CCE] \label{def:cce} For two payoff matrices $\overline{Q}, \underline{Q}\in \RR^{|\cA|\times|\cB|}$, a distribution $\mu$ over $\cA\times\cB$ is $\iota$-CCE if it satisfies
\begin{align*}
&\EE_{(a,b)\sim\mu}
[\overline{Q}(a, b)] \geq  \EE_{b\sim \cP_2\mu}[\overline{Q}(a', b)]-\iota, \forall a'\in \cA,\\
&\EE_{(a,b)\sim\mu}
[\underline{Q}(a, b)] \leq \EE_{a\sim \cP_1\mu}[\underline{Q}(a, b')]+\iota, \forall b'\in \cB.
\end{align*}
\end{definition}
An $\iota$-CCE may not have mutually independent marginals since the two players take actions in a correlated way. The $\iota$-CCE can be found \emph{efficiently} by the method developed in \citet{xie2020learning} for arbitrary $\iota>0$.

\vspace{5pt}
\noindent\textbf{Dataset for Contrastive Learning.} 
Summarized in Algorithm \ref{alg:sample-mg} in Appendix, the sampling algorithm for Markov games follows a similar sampling strategy to Algorithm \ref{alg:sample} with extending the action space from $\cA$ to $\cA\times\cB$. Letting $d^{\sigma}_h(s)$ be a state probability at step $h$ under $\PP$ and a joint policy $\sigma$, we define $\tilde{d}^{\sigma}_h(s,a,b) = d^{\sigma}_h(s) \Unif(a)\Unif(b)$ and $\breve{d}^{\sigma}_h(s,a,b) = \tilde{d}^{\sigma}_{h-1}(s',a',b')\PP_{h-1}(s|s',a',b')\Unif(a)\Unif(b)$, where we define $\Unif(a) =1/|\cA|$ and $\Unif(b) =1/|\cB|$. Analogously, at round $k$, we sample state-action pairs following $\tilde{d}^{\sigma^{k-1}}_h(\cdot,\cdot,\cdot)$ for all $h\in [H]$ and $\breve{d}^{\sigma^{k-1}}_h(\cdot,\cdot,\cdot)$ for all $h\geq2$ and then sample the next state from $\PP_h$ or negative sampling distribution $\cPS$ with probability $1/2$.  We also 
build a dataset for the construction of the bonus term in Algorithm \ref{alg:contrastive-mg} by sampling from $\tilde{d}^{\sigma^{k'}}_h(\cdot,\cdot,\cdot)$ for all $k'\in [0,k-1]$.


\subsection{Main Result for Markov Game Setting}

\begin{theorem}[Sample Complexity]\label{thm:main-mg} Letting
$\lambda_k= c_0 d \log(H|\cF|k/\delta)$ for a sufficiently large constant $c_0>0$, $\gamma_k=  4H\big(12\sqrt{  |\cA||\cB|d} + \sqrt{c_0} d\big)/\CS\cdot \sqrt{\log (2Hk|\cF|/\delta)}$, and $\iota_k \leq \cO( \sqrt{1/k})$, with probability at least $1-3\delta$, we have
\begin{align*}
&1/K\cdot\textstyle\sum_{k=1}^K \big[ V_1^{\mathrm{br}(\nu^k), \nu^k}(s_1)- V_1^{\pi^k, \mathrm{br}(\pi^k)}(s_1)\big] \lesssim \sqrt{C\log(H|\cF|K/\delta) \log(c_0' K)/K },
\end{align*}
where $C = H^4d^4|\cA||\cB|/(\CS)^2 +  H^4d^3|\cA|^2|\cB|^2/(\CS)^2+  H^6d^2|\cA||\cB|/(\CS)^2 + H^6d^3$ and $c_0'$ is an absolute constant.
\end{theorem}
This theorem further implies a PAC bound for learning an approximate NE \citep{xie2020learning}. Specifically, Theorem \ref{thm:main-mg} implies that there exists $k_0\in [K]$ such that $(\pi^{k_0}, \nu^{k_0})$ is an $\varepsilon$-approximate NE with probability at least $1-3\delta$ after executing Algorithm \ref{alg:contrastive-mg} for $K\geq \tilde{\cO}(1/\varepsilon^2)$ episodes, i.e., letting $k_0:=\min_{k\in [K]} [V_1^{\mathrm{br}(\nu^k), \nu^k}(s_1)- V_1^{\pi^k, \mathrm{br}(\pi^k)}(s_1)]$, we then have
\begin{align*}
&V_1^{\mathrm{br}(\nu^{k_0}), \nu^{k_0}}(s_1)- V_1^{\pi^{k_0}, \mathrm{br}(\pi^{k_0})}(s_1) \\
&\quad \leq 1/K \cdot  \textstyle \sum_{k=1}^K [ V_1^{\mathrm{br}(\nu^k), \nu^k}(s_1)- V_1^{\pi^k, \mathrm{br}(\pi^k)}(s_1)]\leq \varepsilon
\end{align*}
with probability at least $1-3\delta$.

%

\section{Theoretical Analysis}

This section provides the analysis of the transition kernel recovery via contrastive learning and the proofs of the main results for single-agent MDPs and zero-sum MGs. Our theoretical analysis integrates contrastive self-supervised learning for transition recovery and low-rank MDPs in a unified manner. Part of our analysis is motivated by the recent work  \citep{uehara2021representation} for learning the low-rank MDPs. In contrast to this work, our paper analyzes the representation recovery via contrastive learning under the online setting. In addition, we consider an episodic setting distinct from the infinite-horizon setting in the aforementioned work. On the other hand, the existing work on low-rank MDPs only focuses on a single-agent setting. Our analysis further considers a Markov game setting where a natural challenge of non-stationarity arises due to competitive policies of multiple players. We develop the first representation learning analysis for Markov games based on the proposed ULCB algorithm.

We first define several notations for our analysis. Recall that we have defined $d^{\pi}_h$, $\tilde{d}^{\pi}_h$, and $\breve{d}^{\pi}_h$ as in Section \ref{sec:alg-mdp}. Then, we subsequently define $\rho^k_h(s,a) := 1/k\cdot\sum_{k'=0}^{k-1} d^{\pi^{k'}}_h(s,a)$, $\tilde{\rho}^k_h(s,a) := 1/k\cdot\sum_{k'=0}^{k-1} \tilde{d}^{\pi^{k'}}_h(s,a)$, and $\breve{\rho}^k_h(s,a) := 1/k\cdot\sum_{k'=0}^{k-1} \breve{d}^{\pi^{k'}}_h(s,a)$, which are the averaged distributions across $k$ episodes for the corresponding state-action distributions. In addition, for any $\rho$ and $\phi$, we define the associated covariance matrix $\Sigma_{\rho, \phi} := k\cdot \EE_{(s,a)\sim\rho^k_h(\cdot,\cdot)}\left[\phi(s,a)\phi(s,a)^\top \right]+\lambda_k I$. On the other hand, for zero-sum MGs, in Section \ref{sec:alg-mg}, we have defined $d^{\sigma}_h$, $\tilde{d}^{\sigma}_h$, and $\breve{d}^{\sigma}_h$ for any joint policy $\sigma$. Then, we can analogously define $\rho^k_h$, $\tilde{\rho}^k_h$, $\breve{\rho}^k_h$, and $\Sigma_{\rho, \phi}$ for MGs by extending action spaces from $\cA$ to $\cA\times\cB$. We summarize these notations in a table in Section \ref{sec:tab_notation}.  Moreover, for abbreviation, letting $z=(s,a)$ for MDPs and $z=(s,a,b)$ for MGs and $\tilde{\rho}_h^k$, $\breve{\rho}_h^k$ be corresponding distributions, we define 
\begin{align}
\begin{aligned} \label{eq:def-P-diff}
&\zeta_h^k:=\EE_{z\sim \tilde{\rho}_h^k}[\|\PP_1(\cdot|z) - \hat{\PP}^k_1(\cdot|z) \|_1^2],\\
&\xi_h^k:=\EE_{z\sim\breve{\rho}^k_h}[ \|\PP_h(\cdot|z) - \hat{\PP}^k_h(\cdot|z) \|_1^2].
\end{aligned}
\end{align}

\subsection{Analysis for Single-Agent MDP}  \label{sec:mdp-theory}

Based on the above definitions and notations, we have the following lemma to show the transition recovery via contrastive learning.
\begin{lemma}[Transition Recovery]\label{lem:stat-err} After executing Algorithm \ref{alg:contrastive} for $k$ rounds, with probability at least $1-2\delta$, 
\begin{align*}
&\zeta_h^k \leq 32d/(\CS)^2\cdot\log (2kH|\cF|/\delta)/k, \quad \forall h\geq 1,\\
&\xi_h^k \leq 32d/(\CS)^2\cdot\log (2kH|\cF|/\delta)/k, \quad \forall h\geq 2,
\end{align*}
where $\zeta_h^k$ and $\xi_h^k$ are defined as \eqref{eq:def-P-diff}.
\end{lemma}

This lemma indicates that via the contrastive learning step in Algorithm \ref{alg:contrastive}, we can successfully learn a correct representation and recover the transition model. Next, we give the proof sketch of this lemma. 

\vspace{5pt}
\noindent\textbf{Proof Sketch of Lemma \ref{lem:stat-err}}.
Letting $\Pr_h^f(y|s,a,s')$ be defined as in Section \ref{sec:alg-mdp}, we have $\Pr{}_h^f(y,s'|s,a)  = \Pr{}_h^f(y|s,a,s')\Pr{}_h(s'|s,a)$ with defining $f_h(s,a,s') := \psi(s')^\top\phi(s,a)$. Furthermore,
we can  calculate that $ \Pr{}_h(s'|s,a) =  \frac{1}{2} [\PP_h(s'|s,a) + \cPS(s')] \geq \frac{1}{2}\CS > 0$ by Assumption \ref{assump:negative}. Thus, the contrastive loss minimization \eqref{eq:solve-contra-loss} is equivalent to $\max_{\phi_h, \psi_h} \EE_{\cD_h^k} \log \Pr{}_h^f(y|s,a,s')$, which further equals $ \max_{\phi_h, \psi_h} \EE_{\cD_h^k}  \log \Pr{}_h^f(y,s'|s,a)$, since $\Pr_h(s'|s,a)$ is only  determined by $\PP_h(s'|s,a)$ and $\cPS(s')$ and is independent of $f_h$. Denoting the solution as $ \hat{f}_h^k(s,a,s') = \tilde{\psi}^k_h(s')^\top \tilde{\phi}^k_h(s,a)$.
With Algorithm \ref{alg:sample}, further by the MLE guarantee in Lemma \ref{lem:recover-mle}, we can show
with high probability,
\begin{align*}
&\EE_{(s,a)\sim \tilde{\rho}_h^k(\cdot,\cdot)} \|\Pr{}_h^{\hat{f}^k}(\cdot,\cdot|s,a) - \Pr{}_h^{f^*}(\cdot,\cdot|s,a) \|_{\TV}^2 \leq \epsilon_k, \\
&\EE_{(s,a)\sim \breve{\rho}_h^k(\cdot,\cdot)} \|\Pr{}_h^{\hat{f}^k}(\cdot,\cdot|s,a) - \Pr{}_h^{f^*}(\cdot,\cdot|s,a) \|_{\TV}^2 \allowbreak \leq \epsilon_k,    
\end{align*} 
where $f^*_h$ is defined in \eqref{eq:targt} and $\epsilon_k:=2\log (2kH|\cF|/\delta)/k$.

Next, we show the recovery error bound of the transition model based on $\hat{f}_h^k$. By expanding $\Pr{}_h^{\hat{f}^k}(\cdot,\cdot|s,a)  - \allowbreak \Pr{}_h^{f^*}(\cdot,\cdot|s,a)$ and making use of $\CS > 0$, we further obtain 
\begin{align*}
&\EE_{(s,a)\sim \tilde{\rho}_h^k(\cdot,\cdot)} \|\PP_h(\cdot|s,a)-\cPS(\cdot) \tilde{\phi}_h^k(s,a)^\top \tilde{\psi}_h^k (\cdot)\|_{\TV}^2 \leq 4d\epsilon_k/(\CS)^2,\\
&\EE_{(s,a)\sim \breve{\rho}_h^k(\cdot,\cdot)} \|\PP_h(\cdot|s,a)-\cPS(\cdot) \tilde{\phi}_h^k(s,a)^\top \tilde{\psi}_h^k (\cdot)\|_{\TV}^2  \leq 4d\epsilon_k/(\CS)^2. 
\end{align*}

Now we define $\hat{g}_h^k(s,a,s') := \cPS(s') \tilde{\phi}_h^k(s,a)^\top \tilde{\psi}_h^k (s')$.
Since that $\int_{s'\in\cS}\hat{g}_h^k(s,a,s')\mathrm{d}s'$ may not be guaranteed to be $1$ though $\hat{g}_h^k(s,a,\cdot)$ is close to the true transition model $\PP_h(\cdot|s,a)$, to obtain a distribution approximator of the transition model $\PP_h$, we further normalize $\hat{g}_h^k(s,a,s')$ and define 
\begin{align*}
\hat{\PP}_h^k(s'|s,a) := \hat{g}_h^k(s,a,s')/\|\hat{g}_h^k(s,a,\cdot)\|_1 = \hat{\psi}_h^k(s')^\top \hat{\phi}_h^k(s,a),    
\end{align*}
which is equivalent to \eqref{eq:normalization}. 
By the definitions of the approximation errors $\zeta_h^k:=\mathbb{E}_{(s,a)\sim \tilde{\rho}_h^k(\cdot,\cdot)}\|\hat{\PP}_h^k(\cdot|s,a)- \PP_h(\cdot|s,a)\|_{\TV}^2$ and $\xi_h^k:=\EE_{(s,a)\sim\breve{\rho}^k_h(\cdot,\cdot)}[ \|\PP_h(\cdot|s,a) - \hat{\PP}^k_h(\cdot|s,a) \|_1^2]$,
we can further prove that 
\begin{align*}
&\zeta_h^k \leq 4\EE_{(s,a)\sim \tilde{\rho}_h^k(\cdot,\cdot)} \|\PP_h(\cdot|s,a)-\cPS(\cdot) \tilde{\phi}_h^k(s,a)^\top \tilde{\psi}_h^k (\cdot)\|_{\TV}^2 \leq 16d\epsilon_k/(\CS)^2,\\
&\xi_h^k \leq 4\EE_{(s,a)\sim \breve{\rho}_h^k(\cdot,\cdot)} \|\PP_h(\cdot|s,a)-\cPS(\cdot) \tilde{\phi}_h^k(s,a)^\top \tilde{\psi}_h^k (\cdot)\|_{\TV}^2 \leq 16d\epsilon_k/(\CS)^2.
\end{align*}
Plugging in $\epsilon_k=2\log (2kH|\cF|/\delta)/k$ gives the desired results. Please see Appendix \ref{sec:proof-stat-err} for a detailed proof.

Based on Lemma \ref{lem:stat-err}, we give the analysis of Theorem \ref{thm:main}.

\vspace{5pt}
\noindent\textbf{Proof Sketch of Theorem \ref{thm:main}}. 
We first define that $\overline{V}_{k,h}^{\pi}$ is the value function on an auxiliary MDP defined by $\hat{\PP}^k$ and $r+\beta^k$.
Then we can decompose $V^{\pi^*}_1(s_1) - V^{\pi^k}_1(s_1)$ as 
\begin{align}
\begin{aligned}\label{eq:decomp-mdp-init-sketch}
&V^{\pi^*}_1(s_1) - V^{\pi^k}_1(s_1)= V^{\pi^*}_1(s_1) - \overline{V}^{\pi^*}_{k,1}(s_1)\\
&\quad\quad+ \overline{V}^{\pi^*}_{k,1}(s_1)- V^k_1(s_1) + V^k_1(s_1) - V^{\pi^k}_1(s_1) \\
&\quad\leq  \underbrace{V^{\pi^*}_1(s_1) - \overline{V}^{\pi^*}_{k,1}(s_1)}_{(i)} +\underbrace{\overline{V}_{k,1}^{\pi^k}(s_1) - V^{\pi^k}_1(s_1)}_{(ii)}, 
\end{aligned}
\end{align}
where the first inequality is by Lemma \ref{lem:plan} that $\overline{V}^{\pi^*}_{k,1}(s_1)\leq V^k_1(s_1)$ due to the value iteration step in Algorithm \ref{alg:contrastive}. Moreover, by the definition of $\overline{V}^k_h$ above, we known $\overline{V}^k_h = \overline{V}_{k,h}^{\pi^k}$ for any $h\in [H]$. Thus, we need to bound $(i)$ and $(ii)$.

To bound term $(i)$, by Lemma \ref{lem:diff1} and Lemma \ref{lem:expand1}, we have
\begin{align*}
(i) =  V_1^{\pi^*}(s_1) - V_1^{\pi^*}(s_1)  &\leq \sqrt{|\cA|  \zeta_1^k },
\end{align*}
which indicates a near-optimism \citep{uehara2021representation} with a bias $\sqrt{|\cA|  \zeta_1^k }\leq\tilde{\cO}(\sqrt{1/k})$ according to Lemma \ref{lem:stat-err}. This is guaranteed by adding a UCB bonus to the Q-function.

Term $(ii)$ basically reflects the model difference between the defined auxiliary MDP and the true MDP under the learned policy $\pi^k$. By Lemma \ref{lem:diff2} and Lemma \ref{lem:expand2}, we have that $(ii) \leq  [\sqrt{3d|\cA|\gamma_k^2  / k } +3H^2\sqrt{|\cA|  \zeta_1^k }] + \sum_{h=1}^{H-1} [\sqrt{3d  |\cA| \gamma_k^2  + 4H^2\lambda_k  d}+3H^2\sqrt{k  |\cA|  \zeta_{h+1}^k + 4\lambda_k  d}] \allowbreak\cdot \EE_{(s,a)\sim d^{\pi^k, \PP}_h}\|\phi^*_h(s,a)\|_{\Sigma_{\rho^k_h, \phi^*_h}^{-1}}$. In fact, we can bound the term $\sum_{k=1}^K\EE_{(s,a)\sim d^{\pi^k, \PP}_h}\|\phi^*_h(s,a)\|_{\Sigma_{\rho^k_h, \phi^*_h}^{-1}}\leq \tilde{O}(\sqrt{dK})$ by Lemma \ref{lem:logdet-tele}. According to Lemma \ref{lem:stat-err}, with high probability, we can bound $\zeta_h^k$ and $\xi_h^k$. Then, $\frac{1}{K}\sum_{k=1}^K (ii)\leq \tilde{O}(1/\sqrt{K})$ with polynomial dependence on $|\cA|, H, d$ by setting parameters as in Theorem \ref{thm:main}.

By \eqref{eq:decomp-mdp-init-sketch}, we have $\frac{1}{K} [V^{\pi^*}_1(s_1) - V^{\pi^k}_1(s_1)] \leq  \frac{1}{K}\sum_{k=1}^K [(i) + (ii)]$. Then, plugging in the upper bounds for terms $(i)$ and $(ii)$, setting the parameters $\gamma_k$ and $\lambda_k$ as in Theorem \ref{thm:main}, we obtain the desired bound.
Please see Appendix \ref{sec:proof-thm-main} for a detailed proof.

\subsection{Analysis for Markov Game} \label{sec:mg-theory}
We further have a transition recovery lemma for Algorithm \ref{alg:contrastive-mg} similar to Lemma \ref{lem:stat-err}. 

\begin{lemma}[Transition Recovery]\label{lem:stat-err-mg} After executing Algorithm \ref{alg:contrastive-mg} for $k$ rounds, with probability at least $1-2\delta$, 
\begin{align*}
&\zeta_h^k \leq 32d/(\CS)^2\cdot\log (2kH|\cF|/\delta)/k, \quad \forall h\geq 1,\\
&\xi_h^k \leq 32d/(\CS)^2\cdot\log (2kH|\cF|/\delta)/k, \quad \forall h\geq 2,
\end{align*}
where $\zeta_h^k$ and $\xi_h^k$ are defined as \eqref{eq:def-P-diff}.
\end{lemma}
The proof idea for Lemma \ref{lem:stat-err-mg} is nearly identical to the one for Lemma \ref{lem:stat-err} with extending the action space from $\cA$ to $\cA\times \cB$. We defer the proof to Appendix \ref{sec:proof-stat-err-mg}. Based on Lemma \ref{lem:stat-err-mg}, we further give the analysis of Theorem \ref{thm:main-mg}.

\vspace{5pt}
\noindent\textbf{Proof Sketch of Theorem \ref{thm:main-mg}.} We define two auxiliary MGs respectively by reward function $r+\beta^k$ and transition model $\hat{\PP}^k$, and $r-\beta^k$, $\hat{\PP}^k$. 
Then, for any joint policy $\sigma$, let $\overline{V}_{k,h}^{\sigma}$ and $\underline{V}_{k,h}^{\sigma}$ be the associated value functions on the two auxiliary MGs respectively. Recall that $\overline{V}_h^k$ and $\underline{V}_h^k$ are generated by Algorithm \ref{alg:contrastive-mg}. 
We then decompose $V_1^{\mathrm{br}(\nu^k), \nu^k}(s_1) -  V_1^{\pi^k, \mathrm{br}(\pi^k)}(s_1)$ as follows
\begin{align}
&    V_1^{\mathrm{br}(\nu^k), \nu^k}(s_1) -  V_1^{\pi^k, \mathrm{br}(\pi^k)}(s_1) =  \underbrace{V_1^{\sigma_\nu^k}(s_1) - \overline{V}_{k,1}^{\sigma_\nu^k}(s_1)}_{(i)} \nonumber \\[-0.3\baselineskip]
&~~\qquad + \underbrace{\overline{V}_{k,1}^{\sigma_\nu^k}(s_1) - \overline{V}_1^k(s_1)}_{(ii)} + \underbrace{\overline{V}_1^k(s_1) - \underline{V}_1^k(s_1)}_{(iii)} \label{eq:decomp-mg-init-sketch} \\[-0.3\baselineskip]
&~~\qquad + \underbrace{\underline{V}_1^k(s_1)-  \underline{V}_{k,1}^{\sigma_\pi^k}(s_1)}_{(iv)} + \underbrace{\underline{V}_{k,1}^{\sigma_\pi^k}(s_1) - V_1^{\sigma_\pi^k}(s_1)}_{(v)}. \nonumber
\end{align}
Here we let $\sigma_\nu^k:=(\mathrm{br}(\nu^k), \nu^k)$ and $\sigma_\pi^k := (\pi^k, \mathrm{br}(\pi^k))$ for abbreviation. Terms $(ii)$ and $(iv)$ depict the planning error on the two auxiliary MGs, which is guaranteed to be small by finding $\iota_k$-CCE in Algorithm \ref{alg:contrastive-mg}. Thus, by Lemma \ref{lem:plan-mg}, we have
\begin{align*}
(ii)\leq H\iota_k, \quad  (iv)\leq H\iota_k,
\end{align*}
which can be controlled by setting a proper value to $\iota_k$ as in Theorem \ref{thm:main-mg}.

Moreover, by Lemma \ref{lem:diff1-mg} and Lemma \ref{lem:expand1-mg}, we obtain
\begin{align*}
(i) \leq \sqrt{|\cA||\cB|  \zeta_1^k }, \quad (v) \leq \sqrt{|\cA||\cB|  \zeta_1^k }, 
\end{align*}
which is guaranteed by the design of ULCB-type Q-functions with the bonus term in our algorithm. Thus we obtain the near-optimism and near-pessimism properties for terms $(i)$ and $(v)$ respectively.

Term $(iii)$ is the model difference between the two auxiliary MGs under the learned joint policy $\sigma^k$. By Lemma \ref{lem:diff2-mg} and Lemma \ref{lem:expand2-mg}, we have that $(iii) \leq  [2\sqrt{3d|\cA|\gamma_k^2  / k } +6H^2\sqrt{|\cA|  \zeta_1^k }] + \sum_{h=1}^{H-1} [2\sqrt{3d  |\cA| \gamma_k^2  + 4H^2\lambda_k  d}+6H^2\sqrt{k  |\cA|  \zeta_{h+1}^k + 4\lambda_k  d}] \allowbreak\cdot \EE_{d^{\sigma^k, \PP}_h}\|\phi^*_h\|_{\Sigma_{\rho^k_h, \phi^*_h}^{-1}}$. Furthermore, we obtain that $\sum_{k=1}^K\EE_{d^{\sigma^k, \PP}_h}\|\phi^*_h\|_{\Sigma_{\rho^k_h, \phi^*_h}^{-1}}\leq \tilde{O}(\sqrt{dK})$  by Lemma \ref{lem:logdet-tele}. According to Lemma \ref{lem:stat-err-mg} for the contrastive learning, with high probability, we can bound $\frac{1}{K}\sum_{k=1}^K (iii)\leq \tilde{O}(1/\sqrt{K})$ under the same conditions in Theorem \ref{thm:main}.  

According to \eqref{eq:decomp-mg-init-sketch}, we have $\frac{1}{K} \sum_{k=1}^K [V_1^{\mathrm{br}(\nu^k), \nu^k}(s_1) -  V_1^{\pi^k, \mathrm{br}(\pi^k)}(s_1)] \leq  \frac{1}{K}\sum_{k=1}^K [(i) + (ii) + (iii) + (iv) + (v)]$. Thus, plugging in the above upper bounds for terms $(i), (ii), (iii), (iv)$, and $(v)$, setting the parameters $\iota_k$, $\gamma_k$ and, $\lambda_k$ as in Theorem \ref{thm:main}, we get the desired result.
Please see Appendix \ref{sec:proof-thm-main-mg} for a detailed proof.

\section{Proof of Concept Experiments}


In this section, we present the experimental justification of the UCB-based exploration in practice inspired by our theory. 
\subsection{Implementation of Bonus}


\textbf{Representation Learning with SPR.} Our goal is to examine whether the proposed UCB bonus practically enhances the exploration of deep RL algorithms with contrastive learning. To this end, we adopt the SPR method \citep{spr-2021}, 
the state-of-the-art RL approach with contrastive learning on the benchmark Atari 100K \citep{kaiser2019model}. SPR utilizes the temporal information and learns the representation via maximizing the similarity between the future state representations and the corresponding predicted next state representations based on the observed state and action sequences.
The representation learning under the framework of SPR is different from the proposed representation learning from the following aspects:
\textbf{(1)} SPR considers multi-step consistency in addition to the one-step prediction of our proposed contrastive objective, namely, SPR incorporates the information of multiple steps ahead of $(s_h, a_h)$ in the representation $\hat\phi(s_h, a_h)$.
Although representation learning with one-step prediction is sufficient according to our theory, such a multi-step approach further enhances the temporal consistency of the learned representation empirically. Similar techniques also arise in various empirical studies \citep{infonce-2018, guo2018neural}. \textbf{(2)} SPR utilizes the cosine similarity to maximize the similarity of state-action representations and the embeddings of the corresponding next states. 
We remark that we adopt the architecture of SPR as an empirical simplification to our proposed contrastive objective, which does not require explicit negative sampling and the corresponding parameter tuning \citep{spr-2021}. 
This leads to better computational efficiency and avoidance of defining an improper negative sampling distribution.
In addition, we remark that the representations obtained from SPR contain sufficient temporal information of the transition dynamics required for exploration, as shown in our experiments.

\vspace{5pt}
\noindent\textbf{Architecture and UCB Bonus.} 
In our experiments, we adopt the same architecture as SPR. We further construct the UCB bonus based on SPR and propose the SPR-UCB method.
In particular, we adopt the same hyper-parameters as that of SPR \citep{spr-2021}. Meanwhile, we adopt the last layer of the Q-network as our learned representation $\widehat{\phi}$
which is linear in the estimated Q-function.
In the training stage, we update the empirical covariance matrix $\widehat{\Sigma}_h^k\in\mathbb{R}^{d\times d}$ by adding the feature covariance $\widehat{\phi}(s_h^k,a_h^k)\widehat{\phi}(s_h^k,a_h^k)^\top$ over the sampled transition tuples $\{(s_h^k,a_h^k, s_h^{k+1})\}_{h\in[H]}$ from the replay buffer, where $\widehat{\phi}\in\mathbb{R}^{d\times 1}$ is the learned representation from the Q-network of SPR. 
The transition data is sampled from the interaction history.
The bonus for the state-action pair $(s,a)$ is calculated by $\beta^k(s,a)=\gamma_k\cdot[\widehat{\phi}(s,a)^\top (\widehat{\Sigma}_h^k)^{-1} \widehat{\phi}(s,a)]^{\frac{1}{2}}$, where we set the hyperparameter $\gamma_k = 1$ for all iterations $k\in[K]$. 
Upon computing the bonus for each state-action pair of the sampled transition tuples from the replay buffer, we follow our proposed update in Algorithm \ref{alg:contrastive_RL} and add the bonus on the target of Q-functions in fitting the Q-network. 

\begin{figure}[t]
\centering
\includegraphics[width=0.73\columnwidth]{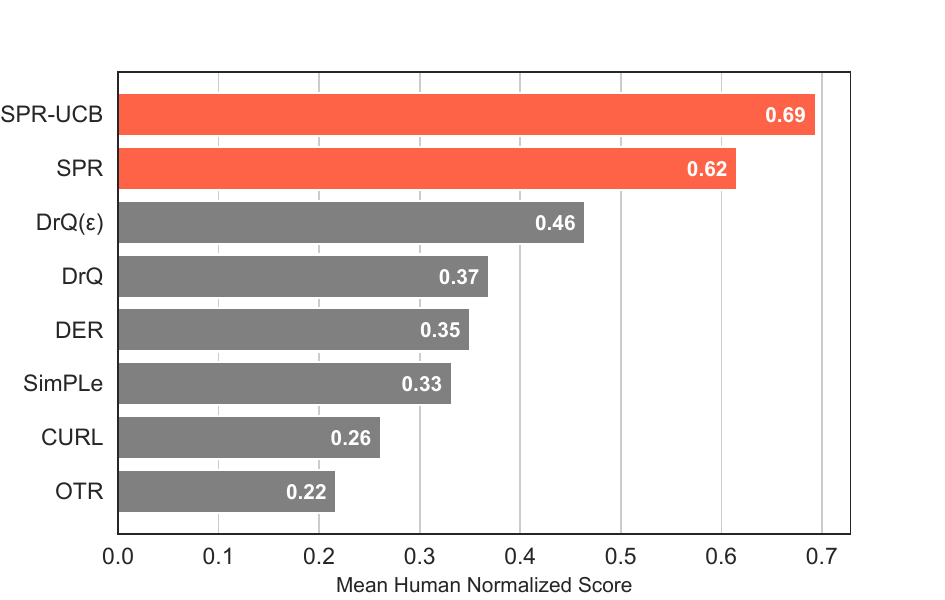}
\caption{Mean human-normalized score in Atari-100K benchmark. The results of baseline algorithms are adopted from \citet{agarwal2021deep}. We observe that SPR-UCB outperforms SPR and other baseline algorithms.}
\label{fig:barfig}
\end{figure}

\subsection{Environments and Baselines}
In our experiments, we use Atari 100K \citep{kaiser2019model} benchmark for evaluation, which contains 26 Atari games from various domains. The benchmark Atari 100K only allows the agent to interact with the environment for 100K steps. Such a setup aims to test the sample efficiency of RL algorithms.

We compare the SPR-UCB method with several baselines in Atari 100K benchmark, including \textbf{(1)} SimPLe \citep{kaiser2019model}, which learns a environment model based on the video prediction task and trains a policy under the learned model; \textbf{(2)} DER \citep{rainbow-2019} and \textbf{(3)} OTR \citep{kielak2020do}, which improve Rainbow \citep{rainbow-2019} to perform sample-efficient model-free RL; \textbf{(4)} CURL \citep{curl-2020}, which incorporates contrastive learning based on data augmentation; \textbf{(5)} DrQ \citep{drq-2021}, which directly utilizes data augmentation based on the image observations; and \textbf{(6)} SPR \citep{spr-2021}, which learns temporal consistent representation for model-free RL. 
For all methods, we calculate the human normalized score by $\frac{\rm agent\:score-random\:score}{\rm human\:score-random\:score}$. In our experiments, we run the proposed SPR-UCB over 10 different random seeds. 

\subsection{Result Comparison}

We illustrate the aggregated mean of human normalized scores among all tasks in Figure \ref{fig:barfig}. We report the score for each task in Appendix \ref{app:exp}. In our experiments, we observe that \textbf{(1)} Both SPR and SPR-UCB outperform baselines that do not learn temporal consistent representations significantly, including DER, OTR, SimPLe, CURL, and DrQ.  \textbf{(2)} By incorporating the UCB bonus, SPR-UCB outperforms SPR. In addition, we remark that SPR-UCB outperforms SPR significantly in challenging environments including \emph{Boxing}, \emph{Freeway}, \emph{Frostbite}, \emph{KungfuMaster}, \emph{PrivateEye}, and \emph{RoadRunner}. Please see Appendix \ref{app:exp} for the details.

\vspace{-0.1cm}
\section{Conclusion}
\vspace{-0.1cm}
We study contrastive-learning empowered RL for MDPs and MGs with low-rank transitions. We propose novel online RL algorithms that incorporate such a contrastive loss with temporal information for MDPs or MGs. We further theoretically prove that our algorithms recover the true representations and simultaneously achieve sample efficiency in learning the optimal policy and Nash equilibrium in MDPs and MGs respectively. We also provide empirical studies to demonstrate the efficacy of 
the UCB-based contrastive learning method for RL. 
To the best of our knowledge, we provide the first provably efficient online RL algorithm that incorporates contrastive learning for representation learning.

\vspace{-0.2cm}
\section*{Acknowledgements}

The authors would like to thank all reviewers for their valuable comments. The authors would also like to thank Sirui Zheng for helpful discussions. The contribution from Chenjia Bai was made during his time as a visiting student at the University of Toronto (Vector Institute for Artificial Intelligence), working with Animesh Garg. The theory, methods, and codes developed in this paper are shared publicly without any proprietary or other restrictions.

\bibliography{example_paper}
\bibliographystyle{icml2022}

\newpage
\begin{appendices}
	\onecolumn
\vspace{1em}
\renewcommand{\thesection}{\Alph{section}}

\vspace{0.5em}

\section{Sampling Algorithms} \label{sec:omitted}

\begin{algorithm}[h]\caption{Contrastive Data Sampling for Single-Agent MDPs} 
   \begin{footnotesize}
	\begin{algorithmic}[1]
    	        \For{step $h=1,\ldots, H-1$}
    	            \State Sample $(\tilde{s}_h^k,\tilde{a}_h^k)\sim \tilde{d}_h^{\pi^{k-1}}(\cdot,\cdot)$, $\tilde{s}_{h+1}^k \sim \PP_h(\cdot|\tilde{s}_h^k,\tilde{a}_h^k)$
    	            
    	            \State Let $\breve{s}_{h+1}^k = \tilde{s}_{h+1}^k$. Sample $\breve{a}_{h+1}^k \sim \Unif(\cA)$, $\breve{s}_{h+2}^k \sim \PP_{h+1}(\cdot|\breve{s}_{h+1}^k,\breve{a}_{h+1}^k)$, and $y_h^k\sim \Ber(1/2)$  
                    \State $\tilde{\cD}_h^k = \tilde{\cD}_h^{k-1} \cup \{(\tilde{s}_h^k,\tilde{a}_h^k)\}$.
        	        \If{$y_h^k=1$} 
        	        \State $\cD_h^k  = \cD_h^{k-1} \cup \{(\tilde{s}_h^k,\tilde{a}_h^k,\tilde{s}_{h+1}^k, 1)\}$.
        	        \State $\cD_{h+1}^k  = \cD_{h+1}^{k-1} \cup \{(\breve{s}_{h+1}^k,\breve{a}_{h+1}^k,\breve{s}_{h+2}^k, 1)\}$.
        	        \ElsIf{$y_h^k=0$}
        	        \State Sample negative transition $\tilde{s}_{h+1}^{k,-},\breve{s}_{h+2}^{k,-}\sim\cPS(\cdot)$.
        	        \State $\cD_h^k  = \cD_h^{k-1} \cup \{(\tilde{s}_h^k,\tilde{a}_h^k,\tilde{s}_{h+1}^{k,-}, 0)\}$.
        	        \State $\cD_{h+1}^k  = \cD_{h+1}^{k-1} \cup \{(\breve{s}_{h+1}^k,\breve{a}_{h+1}^k,\breve{s}_{h+2}^{k,-}, 0)\}$.
        	        \EndIf
    		            
    	        \EndFor
    	\State $(\tilde{s}_H^k,\tilde{a}_H^k)\sim \tilde{d}_H^{\pi^{k-1}}(\cdot,\cdot)$, $\tilde{s}_{H+1}^k \sim \PP_h(\cdot|\tilde{s}_H^k,\tilde{a}_H^k)$, and $y_H^k\sim \Ber(1/2)$  
                    \State $\tilde{\cD}_H^k = \tilde{\cD}_H^{k-1} \cup \{(\tilde{s}_H^k,\tilde{a}_H^k)\}$.	        
 					\If{$y_H^k=1$} 
        	        \State $\cD_H^k  = \cD_H^{k-1} \cup \{(\tilde{s}_H^k,\tilde{a}_H^k,\tilde{s}_{H+1}^k, 1)\}$.
        	        \ElsIf{$y_H^k=0$}
        	        \State Sample negative transition $\tilde{s}_{H+1}^{k,-}\sim\cPS(\cdot)$.
        	        \State $\cD_H^k  = \cD_H^{k-1} \cup \{(\tilde{s}_H^k,\tilde{a}_H^k,\tilde{s}_{H+1}^{k,-}, 0)\}$.
        	        \EndIf
        \State {\bfseries return} $\{\cD_h^k\}_{h=1}^H$ and $\{\tilde{\cD}_h^k\}_{h=1}^H$.      
	\end{algorithmic}\label{alg:sample}
	\end{footnotesize}
\end{algorithm}

 \begin{algorithm}[!h]\caption{Contrastive Data Sampling for Markov Games} 
   \begin{footnotesize}
	\begin{algorithmic}[1]
    	       	        \For{step $h=1,\ldots, H-1$}
    	            \State Sample $(\tilde{s}_h^k,\tilde{a}_h^k,\tilde{b}_h^k)\sim \tilde{d}_h^{\pi^{k-1}}(\cdot,\cdot,\cdot)$, $\tilde{s}_{h+1}^k \sim \PP_h(\cdot|\tilde{s}_h^k,\tilde{a}_h^k,\tilde{b}_h^k)$
    	            
    	            \State Let $\breve{s}_{h+1}^k = \tilde{s}_{h+1}^k$. Sample $\breve{a}_{h+1}^k \sim \Unif(\cA)$,$\breve{b}_{h+1}^k \sim \Unif(\cB)$, $\breve{s}_{h+2}^k \sim \PP_{h+1}(\cdot|\breve{s}_{h+1}^k,\breve{a}_{h+1}^k,\breve{b}_{h+1}^k)$, $y_h^k\sim \Ber(1/2)$  
                    \State $\tilde{\cD}_h^k = \tilde{\cD}_h^{k-1} \cup \{(\tilde{s}_h^k,\tilde{a}_h^k, \tilde{b}_h^k)\}$.
        	        \If{$y_h^k=1$} 
        	        \State $\cD_h^k  = \cD_h^{k-1} \cup \{(\tilde{s}_h^k,\tilde{a}_h^k,\tilde{b}_h^k,\tilde{s}_{h+1}^k, 1)\}$.
        	        \State $\cD_{h+1}^k  = \cD_{h+1}^{k-1} \cup \{(\breve{s}_{h+1}^k,\breve{a}_{h+1}^k,\breve{b}_{h+1}^k,\breve{s}_{h+2}^k, 1)\}$.
        	        \ElsIf{$y_h^k=0$}
        	        \State Sample negative transition $\tilde{s}_{h+1}^{k,-},\breve{s}_{h+2}^{k,-}\sim\cPS(\cdot)$.
        	        \State $\cD_h^k  = \cD_h^{k-1} \cup \{(\tilde{s}_h^k,\tilde{a}_h^k,\tilde{b}_h^k,\tilde{s}_{h+1}^{k,-}, 0)\}$.
        	        \State $\cD_{h+1}^k  = \cD_{h+1}^{k-1} \cup \{(\breve{s}_{h+1}^k,\breve{a}_{h+1}^k,\breve{b}_{h+1}^k,\breve{s}_{h+2}^{k,-}, 0)\}$.
        	        \EndIf
    		            
    	        \EndFor
    	\State Sample $(\tilde{s}_H^k,\tilde{a}_H^k,\tilde{b}_H^k)\sim \tilde{d}_H^{\pi^{k-1}}(\cdot,\cdot,\cdot)$, $\tilde{s}_{H+1}^k \sim \PP_h(\cdot|\tilde{s}_H^k,\tilde{a}_H^k,\tilde{b}_H^k)$, and $y_H^k\sim \Ber(1/2)$  
                    \State $\tilde{\cD}_H^k = \tilde{\cD}_H^{k-1} \cup \{(\tilde{s}_H^k,\tilde{a}_H^k,\tilde{b}_H^k)\}$.	        
 					\If{$y_H^k=1$} 
        	        \State $\cD_H^k  = \cD_H^{k-1} \cup \{(\tilde{s}_H^k,\tilde{a}_H^k,\tilde{b}_H^k,\tilde{s}_{H+1}^k, 1)\}$.
        	        \ElsIf{$y_H^k=0$}
        	        \State Sample negative transition $\tilde{s}_{H+1}^{k,-}\sim\cPS(\cdot)$.
        	        \State $\cD_H^k  = \cD_H^{k-1} \cup \{(\tilde{s}_H^k,\tilde{a}_H^k,\tilde{b}_H^k,\tilde{s}_{H+1}^{k,-}, 0)\}$.
        	        \EndIf
        \State {\bfseries return} $\{\cD_h^k\}_{h=1}^H$ and $\{\tilde{\cD}_h^k\}_{h=1}^H$.      
	\end{algorithmic}\label{alg:sample-mg}
	\end{footnotesize}
\end{algorithm}
\setlength{\textfloatsep}{0pt}
\vspace{-0.5cm}

\section{Notation} \label{sec:tab_notation}
We present the following table of notations. We denote by $\sigma$ an arbitrary joint policy. If the joint policy $\sigma$ is equivalent to a product of two separate policies for each player, i.e., $\sigma(a,b|s) = \pi(a|s)\times \nu(b|s)$, then we can replace $\sigma$ by $\pi, \nu$.

\begin{table}[!h]
\caption{Table of Notation}
\vspace{0.2cm}
\centering
\renewcommand*{\arraystretch}{1.2}
\begin{footnotesize}
\begin{tabular}{ >{\centering\arraybackslash}m{2.2cm} | >{\centering\arraybackslash}m{13.5cm} } 
\hline\hline
Notation & Meaning \\ 
\hline

$d^{\pi}_h(s)$ & state probability at step $h$ under the true transition $\PP$ and a policy $\pi$\\ 

$d^{\pi}_h(s,a)$ & state-action probability at step $h$ under the true transition $\PP$ and a policy $\pi$\\ 

$\tilde{d}^{\pi}_h(s,a)$ & $\tilde{d}^{\pi}_h(s,a) := d^{\pi}_h(s)\Unif(a)$ \\ 

$\breve{d}^{\pi}_h(s,a)$ & $\breve{d}^{\pi}_h(s,a) := \tilde{d}^{\pi}_{h-1}(s',a')\PP_{h-1}(s|s',a')\Unif(a)$ \\

$\rho^k_h(s,a)$ & $\rho^k_h(s,a) := 1/k\cdot\sum_{k'=0}^{k-1} d^{\pi^{k'}}_h(s,a)$ \\

$\tilde{\rho}^k_h(s,a)$ & $\tilde{\rho}^k_h(s,a) := 1/k\cdot\sum_{k'=0}^{k-1} \tilde{d}^{\pi^{k'}}_h(s,a)$ \\

$\breve{\rho}^k_h(s,a)$ & $\breve{\rho}^k_h(s,a) := 1/k\cdot\sum_{k'=0}^{k-1} \breve{d}^{\pi^{k'}}_h(s,a)$ \\

$\Sigma_{\rho, \phi}$ & covariance matrix defined as $k\cdot \EE_{(s,a)\sim\rho^k_h(\cdot,\cdot)}\left[\phi(s,a)\phi(s,a)^\top \right]+\lambda_k I$ for any $\rho$ and $\phi$\\

\hline
$d^{\sigma}_h(s)$ & state probability at step $h$ under the true transition $\PP$ and a joint policy $\sigma$\\ 

$d^{\sigma}_h(s,a,b)$ & state-action probability at step $h$ under the true transition $\PP$ and a joint policy $\sigma$\\ 

$\tilde{d}^{\sigma}_h(s,a,b)$ & $\tilde{d}^{\sigma}_h(s,a,b) := d^{\sigma}_h(s)\Unif(a)\Unif(b)$ \\ 

$\breve{d}^{\sigma}_h(s,a,b)$ & $\breve{d}^{\sigma}_h(s,a,b) := \tilde{d}^{\sigma}_{h-1}(s',a',b')\PP_{h-1}(s|s',a',b')\Unif(a)\Unif(b)$ \\

$\rho^k_h(s,a,b)$ & $\rho^k_h(s,a,b) := 1/k\cdot\sum_{k'=0}^{k-1} d^{\sigma^{k'}}_h(s,a,b)$ \\

$\tilde{\rho}^k_h(s,a,b)$ & $\tilde{\rho}^k_h(s,a,b) := 1/k\cdot\sum_{k'=0}^{k-1} \tilde{d}^{\sigma^{k'}}_h(s,a,b)$ \\

$\breve{\rho}^k_h(s,a,b)$ & $\breve{\rho}^k_h(s,a,b) := 1/k\cdot\sum_{k'=0}^{k-1} \breve{d}^{\sigma^{k'}}_h(s,a,b)$ \\

$\Sigma_{\rho, \phi}$ & covariance matrix defined as $k\cdot \EE_{(s,a,b)\sim\rho(\cdot,\cdot,\cdot)}\left[\phi(s,a,b)\phi(s,a,b)^\top \right]+\lambda_k I$\\
\hline
$V_h^\pi, Q_h^\pi$ & value and Q-functions at step $h$ under the policy $\pi$ and the true transition and reward  $\PP, r$\\ 

$\overline{V}_h^k, \overline{Q}_h^k$ & value and Q-functions generated in Lines 11 and 12 of Algorithm \ref{alg:contrastive}\\ 

$\overline{V}_{k,h}^\pi, \overline{Q}_{k,h}^\pi$ & value and Q-functions at step $h$ on the auxiliary MDP defined by $r+\beta^k$ and $\hat{\PP}^k$ \\ 

\hline
$V_h^\sigma, Q_h^\sigma$ & value and Q-functions at step $h$ under the joint policy $\sigma$ and the true transition and reward  $\PP, r$\\ 

$\overline{V}_h^k, \overline{Q}_h^k$ & value and Q-functions generated in Lines 11 and 13 of Algorithm \ref{alg:contrastive-mg}\\ 

$\underline{V}_h^k, \underline{Q}_h^k$ & value and Q-functions generated in Lines 12 and 14 of Algorithm \ref{alg:contrastive-mg}\\ 

$\overline{V}_{k,h}^\sigma, \overline{Q}_{k,h}^\sigma$ & value and Q-functions at step $h$ on the auxiliary MG defined by $r+\beta^k$ and $\hat{\PP}^k$ \\ 

$\underline{V}_{k,h}^\sigma, \underline{Q}_{k,h}^\sigma$ & value and Q-functions at step $h$ on the auxiliary MG defined by $r-\beta^k$ and $\hat{\PP}^k$ \\ 

\hline
$\Unif(\cA), \Unif(\cB)$ & uniform distribution over spaces $\cA$ or $\cB$\\ 

$\Unif(a), \Unif(b)$ & probabilities for the above distributions: $\Unif(a) = 1/|\cA|$ and $\Unif(b)= 1/|\cB|$\\ 

$\|\cdot\|_{\TV}$ & total variation distance \\ 

$\|\cdot\|_1$ & define $\|f\|_1 := \int_x |f(x)| \mathrm{d}x$ \\ 
\hline \hline
\end{tabular}
\end{footnotesize}
\label{tab:notation}
\end{table}

Moreover, in this appendix, we make s simplification to our notation, which is 
\begin{align*}
    \EE_{p}\left\|\phi\right\|_{\Sigma^{-1}}:=\EE_{z\sim p(\cdot,\cdot)}\left\|\phi(z)\right\|_{\Sigma^{-1}},
\end{align*}
where $z = (s,a)$ for the single-agent MDP setting and $z=(s,a,b)$ for the Markov game setting. Moreover, $p$ denotes some distribution for $z$ and $\phi(z)\in \RR^d$ is some representation of $z$. And $\Sigma$ denotes some invertible covariance matrix.

\section{Theoretical Analysis for Single-Agent MDP}

\subsection{Lemmas}

\begin{lemma}[Learning Target of Contrastive Loss]\label{lem:opt-contra-loss} For any $(s,a)\in \cS\times\cA$ that is reachable under certain sampling strategy, the learning target of the contrastive loss in \eqref{eq:contra-loss} is
\begin{align*}
    f_h^*(s,a,s') = \frac{\PP_h(s'|s,a)}{\cPS(s')}.
\end{align*}

\end{lemma}

\begin{proof}
For any $h\in [H]$, we let $\Pr{}_h$ to denote the probability for some event at the $h$-th step of an MDP. Our contrastive loss in \eqref{eq:contra-loss} implicitly assumes
\begin{align*}
    \Pr{}_h(y|s,a,s') = \left(\frac{f^*_h(s,a,s')}{1+f^*_h(s,a,s')}\right)^y\left(\frac{1}{1+f^*_h(s,a,s')}\right)^{1-y}.
\end{align*}
On the other hand, by Bayes' rule, we know $\Pr{}_h(y|s,a,s')$ can be rewritten as
\begin{align*}
\Pr{}_h(y|s,a,s') = \frac{\Pr{}_h(s,a,s'|y)\Pr{}_h(y)}{\sum_{y\in \{0,1\}}\Pr{}_h(s,a,s'|y)\Pr{}_h(y)} = \frac{\Pr{}_h(s,a,s'|y)}{\Pr{}_h(s,a)\PP_h(s'|s,a) +\Pr{}_h(s,a)\cPS(s')},
\end{align*}
where the last equation uses the fact that $\Pr{}_h(y) = 1/2$ for any $y\in\{0, 1\}$ at the $h$-th step according to our sampling algorithm. In the last equality, we also have
\begin{align*}
&\Pr{}_h(s,a,s'|y=1) = \Pr{}_h(s,a|y=1) \Pr{}_h(s'|y=1,s,a)= \Pr{}_h(s,a) \PP_h(s'|s,a),\\
&\Pr{}_h(s,a,s'|y=0) = \Pr{}_h(s,a|y=0) \Pr{}_h(s'|y=0,s,a)= \Pr{}_h(s,a) \cPS(s'),
\end{align*} 
where we use $\Pr{}_h(s,a|y=1) = \Pr{}_h(s,a|y=0) =\Pr{}_h(s,a)$ since $(s,a)$ and $y$ are independent at each step,  and also $\Pr{}_h(s'|y=1,s,a)= \PP_h(s'|s,a)$ as well as $\Pr{}_h(s'|y=0,s,a)= \cPS(s')$.  

Therefore, combining the above results, when $y=1$ at the $h$-th step,  we obtain
\begin{align*}
    \frac{f_h^*(s,a,s')}{1+f_h^*(s,a,s')} = \frac{ \Pr{}_h(s,a)\PP_h(s'|s,a)}{\Pr{}_h(s,a)\PP_h(s'|s,a) +\Pr{}_h(s,a)\cPS(s')},
\end{align*}
which further gives
\begin{align*}
    f_h^*(s,a,s') = \frac{\PP_h(s'|s,a)}{\cPS(s')},
\end{align*}
since$(s,a)$ is reachable under the sampling algorithm, namely $\Pr{}_h(s,a) > 0$. Equivalently, when $y=0$, we get the same result. This completes the proof.
\end{proof}


\begin{lemma}
\label{lem:diff1} Let $\pi^*:=\argmax_\pi V_1^\pi(s_1)$ be the optimal policy and $\overline{V}_{k,1}^\pi$ be the value function under any policy $\pi$ associated with an MDP defined by the reward function $r + \beta^k$ and the estimated transition $\hat{\PP}^k$ with $\beta^k$ and $\hat{\PP}^k$ obtained at episode $k$ of Algorithm \ref{alg:contrastive}. We have the decomposition of the difference between the following two value functions as
\begin{align*}
&V_1^{\pi^*}(s_1) - \overline{V}_{k,1}^{\pi^*}(s_1)  = \EE \left[ \sum_{h=1}^H \left(-\beta^k_h(s_h,a_h) + (\PP_h - \hat{\PP}^k_h )V^{\pi^*}_{h+1}(s_h,a_h)\right) \Bigggiven \pi^*, \hat{\PP}^k \right].    
\end{align*}
\end{lemma}

\begin{proof} 
We consider two MDPs defined by $(\cS,\cA, H, r, \PP)$ and $(\cS,\cA, H, r+\beta, \PP')$ where  $\PP$ and $\PP'$ are any transition models and $r$ and $\beta$ are arbitrary reward function and bonus term. Then, for any deterministic policy $\pi$, we let $Q^\pi_h$ and $V^\pi_h$ be the associated Q-function and value function at the $h$-th step for the MDP defined by $(\cS,\cA, H, r, \PP)$, and  $\tilde{Q}^\pi_h$ and $\tilde{V}^\pi_h$ be the associated Q-function and value function at the $h$-th step for the MDP defined by $(\cS,\cA, H, r+\beta, \PP')$. Then, we have for any $(s_h, a_h)\in \cS\times\cA$,
\begin{align*}
    & Q_h^\pi(s_h,a_h) - \tilde{Q}_h^\pi(s_h,a_h)\\
    &\qquad =  -\beta_h(s_h,a_h) + \PP_h V^{\pi}_{h+1}(s_h,a_h)- \PP'_h \tilde{V}^{\pi}_{h+1}(s_h,a_h) \\
    &\qquad = -\beta_h(s_h,a_h) + \PP_h V^{\pi}_{h+1}(s_h,a_h)- \PP'_h V^{\pi}_{h+1}(s_h,a_h) + \PP'_h V^{\pi}_{h+1}(s_h,a_h) - \PP'_h\tilde{V}^{\pi}_{h+1}(s_h,a_h)  \\
    &\qquad = -\beta_h(s_h,a_h) + (\PP_h - \PP'_h )V^{\pi}_{h+1}(s_h,a_h) + \PP'_h[V^{\pi}_{h+1}(s_h,a_h) - \tilde{V}^{\pi}_{h+1}(s_h,a_h)],
\end{align*}
where we use the Bellman equation for the above equalities. Thus, further by the Bellman equation and the above result, we have
\begin{align*}
&V_h^\pi(s_h) - \tilde{V}_h^\pi(s_h) \\
&\qquad = Q_h^\pi(s_h,\pi_h(s_h)) - \tilde{Q}_h^\pi(s_h,\pi_h(s_h))\\
&\qquad =  -b_h(s_h,\pi_h(s_h)) + (\PP_h - \PP'_h )V^{\pi}_{h+1}(s_h,\pi_h(s_h)) + \PP'_h[V^{\pi}_{h+1}(s_h,\pi_h(s_h)) - \tilde{V}^{\pi}_{h+1}(s_h,\pi_h(s_h))]. 
\end{align*}
By the fact that $V_{H+1}^\pi(s) = \tilde{V}_{H+1}^\pi(s) = 0$ for any $s\in\cS$ and $\pi$, recursively applying the above relation, we have
\begin{align*}
&V_1^\pi(s_1) - \tilde{V}_1^\pi(s_1)  = \EE \left[ \sum_{h=1}^H \left(-\beta_h(s_h,a_h) + (\PP_h - \PP'_h )V^{\pi}_{h+1}(s_h,a_h)\right) \Bigggiven \pi, \PP' \right].     
\end{align*}
Note that the above results can be straightforwardly extended to any randomized policy $\pi = \{\pi_h\}_{h=1}^H$ with $\pi_h: \cS\times\cA \mapsto [0,1]$. 

For any episode $k$, setting $\PP', \pi, \beta$ to be $\hat{\PP}^k,\pi^*,\beta^k$ defined in Algorithm \ref{alg:contrastive} and $\PP, r$ to be the true transition model and reward function, by the above equation and the definition of $V_h^\pi$ and $\overline{V}_h^\pi$, we obtain
\begin{align*}
&V_1^{\pi^*}(s_1) - \overline{V}_{k,1}^{\pi^*}(s_1)  = \EE \left[ \sum_{h=1}^H \left(-\beta^k_h(s_h,a_h) + (\PP_h - \hat{\PP}^k_h )V^{\pi^*}_{h+1}(s_h,a_h)\right) \Bigggiven \pi^*, \hat{\PP}^k \right].    
\end{align*}
This completes the proof.
\end{proof}

\begin{lemma}
\label{lem:diff2}  Let $\pi^k$ be the learned policy at episode $k$ of Algorithm \ref{alg:contrastive} and $\overline{V}_{k,1}^\pi$ be the value function under any policy $\pi$ associated with an MDP defined by the reward function $r + \beta^k$ and the estimated transition $\hat{\PP}^k$ with $\beta^k$ and $\hat{\PP}^k$ obtained at episode $k$ of Algorithm \ref{alg:contrastive}. We have the decomposition of the difference between the following two value functions as
\begin{align*}
&V_1^{\pi^k}(s_1) - \overline{V}_{k,1}^{\pi^k}(s_1)  = \EE \left[ \sum_{h=1}^H \left(-\beta^k_h(s_h,a_h) + (\PP_h - \hat{\PP}^k_h )\overline{V}^{\pi^k}_{h+1}(s_h,a_h)\right) \Bigggiven \pi^k, \PP \right].    
\end{align*}
\end{lemma}

\begin{proof} Similar to 
\hyperref[lem:diff1]{Proof of Lemma \ref{lem:diff1}}, 
we consider two arbitrary MDPs defined by $(\cS,\cA, H, r, \PP)$ and $(\cS,\cA, H, r+\beta, \PP')$. For any deterministic policy $\pi$, let $Q^\pi_h$ and $V^\pi_h$ be the associated Q-function and value function at the $h$-th step for the MDP defined by $(\cS,\cA, H, r, \PP)$, and  $\tilde{Q}^\pi_h$ and $\tilde{V}^\pi_h$ be the associated Q-function and value function at the $h$-th step for the MDP defined by $(\cS,\cA, H, r+\beta, \PP')$. For any $(s_h, a_h)\in \cS\times\cA$, by Bellman equation, we have
\begin{align*}
    & Q_h^\pi(s_h,a_h) - \tilde{Q}_h^\pi(s_h,a_h)\\
    &\qquad =  -\beta_h(s_h,a_h) + \PP_h V^{\pi}_{h+1}(s_h,a_h)- \PP'_h\tilde{V}^{\pi}_{h+1}(s_h,a_h) \\
    &\qquad = -\beta_h(s_h,a_h) + \PP_h V^{\pi}_{h+1}(s_h,a_h)- \PP_h\tilde{V}^{\pi}_{h+1}(s_h,a_h)+ \PP_h \tilde{V}^{\pi}_{h+1}(s_h,a_h)- \PP'_h\tilde{V}^{\pi}_{h+1}(s_h,a_h) \\
    &\qquad = -\beta_h(s_h,a_h) + \PP_h[V^{\pi}_{h+1}(s_h,a_h) - \tilde{V}^{\pi}_{h+1}(s_h,a_h)] + (\PP_h - \PP'_h )\tilde{V}^{\pi}_{h+1}(s_h,a_h). 
\end{align*}
Then, further by the Bellman equation and the above result, we have
\begin{align*}
&V_h^\pi(s_h) - \tilde{V}_h^\pi(s_h) \\
&\qquad = Q_h^\pi(s_h,\pi_h(s_h)) - \tilde{Q}_h^\pi(s_h,\pi_h(s_h))\\
&\qquad = -\beta_h(s_h,\pi_h(s_h)) + \PP_h[V^{\pi}_{h+1}(s_h,\pi_h(s_h)) - \tilde{V}^{\pi}_{h+1}(s_h,\pi_h(s_h))] + (\PP_h - \PP'_h )\tilde{V}^{\pi}_{h+1}(s_h,\pi_h(s_h))]. 
\end{align*}
By the fact that $V_{H+1}^\pi(s) = \tilde{V}_{H+1}^\pi(s) = 0$ for any $s\in\cS$ and $\pi$, recursively applying the above relation, we have
\begin{align*}
&V_1^\pi(s_1) - \tilde{V}_1^\pi(s_1)  = \EE \left[ \sum_{h=1}^H \left(-\beta_h(s_h,a_h) + (\PP_h - \PP'_h )\tilde{V}^{\pi}_{h+1}(s_h,a_h)\right) \Bigggiven \pi, \PP \right].
\end{align*}
The above results can be straightforwardly extended to any randomized policy $\pi = \{\pi_h\}_{h=1}^H$ with $\pi_h: \cS\times\cA \mapsto [0,1]$. 

For any episode $k$, setting $\PP', \pi, \beta$ to be $\hat{\PP}^k,\pi^k,\beta^k$ defined in Algorithm \ref{alg:contrastive} and $\PP, r$ to be the true transition model and reward function, by the above equation and the definition of $V_h^\pi$ and $\overline{V}_h^\pi$, we obtain
\begin{align*}
&V_1^{\pi^k}(s_1) - \overline{V}_{k,1}^{\pi^k}(s_1)  = \EE \left[ \sum_{h=1}^H \left(-\beta^k_h(s_h,a_h) + (\PP_h - \hat{\PP}^k_h )\overline{V}^{\pi^k}_{h+1}(s_h,a_h)\right) \Bigggiven \pi^k, \PP \right].    
\end{align*}
This completes the proof.
\end{proof}

\begin{lemma}\label{lem:expand1}  Let $\hat{\PP}^k$ be the estimated transition obtained at episode $k$ of Algorithm \ref{alg:contrastive}. Define $\zeta^k_{h-1}:=\EE_{(s'',a'')\sim\tilde{\rho}^k_{h-1}(\cdot,\cdot)}\allowbreak \|\hat{\PP}^k_{h-1}(\cdot|s'',a'') -\PP_{h-1}(\cdot|s'',a'')\|_1^2$ for all $h\geq 2$, $\tilde{\rho}^k_h(\cdot,\cdot) := \frac{1}{k}\sum_{k'=0}^{k-1} \tilde{d}^{\pi^{k'}}_h(\cdot,\cdot)$ for all $h\geq 1$ with $\tilde{\rho}_1^k(s_1,a) =\Unif(a)$, and $\breve{\rho}^k_h(\cdot,\cdot) := \frac{1}{k}\sum_{k'=0}^{k-1} \breve{d}^{\pi^{k'}}_h(\cdot,\cdot)$ for all $h\geq 2$.  Then for any function $g:\cS\times\cA\mapsto [0,B]$ and policy $\pi$, we have for any $h \geq 2$, the following inequality holds
\begin{align*}
&\left|\EE_{(s,a)\sim d^{\pi, \hat{\PP}^k}_h(\cdot,\cdot)}[g(s,a)]\right| \\
&\qquad \leq \sqrt{2k B^2 \zeta^k_{h-1} + 2k  |\cA|\cdot  \EE_{(s,a)\sim\breve{\rho}^k_h(\cdot,\cdot)}[ g(s,a)^2] + \lambda_k B^2 d/(\CS)^2} \cdot \EE_{ d^{\pi, \hat{\PP}^k}_{h-1}}\left\|\hat{\phi}^k_{h-1}\right\|_{\Sigma_{\tilde{\rho}^k_{h-1}, \hat{\phi}^k_{h-1}}^{-1}}.
\end{align*}
Moreover, for $h=1$, we have
\begin{align*}
\left|\EE_{(s,a)\sim d^{\pi, \hat{\PP}^k}_1(\cdot,\cdot)}[g(s,a)]\right| = \sqrt{g(s_1,\pi_1(s_1))^2} \leq \sqrt{|\cA|  \EE_{a\sim \tilde{\rho}_1^k(s_1,\cdot)}[g(s_1,a)^2] },
\end{align*} 
where $\tilde{\rho}_1^k(s_1,a) =\Unif(a)$.
\end{lemma}

\begin{proof}
For any function $g:\cS\times\cA\mapsto [0,B]$ and any deterministic policy $\pi$, under the estimated transition model $\hat{\PP}^k$ at the episode $k$, for any $h\geq 2$, we have
\begin{align}
\begin{aligned} \label{eq:step-back-mdp1}
&\left|\EE_{(s,a)\sim d^{\pi, \hat{\PP}^k}_h(\cdot,\cdot)}[g(s,a)]\right| \\
&\qquad= \left|\EE_{(s',a')\sim d^{\pi, \hat{\PP}^k}_{h-1}(\cdot,\cdot), s\sim \hat{\PP}^k_{h-1}(\cdot|s',a')}[g(s,\pi_h(s))]\right| \\
&\qquad= \left|\EE_{(s',a')\sim d^{\pi, \hat{\PP}^k}_{h-1}(\cdot,\cdot)} \left[\hat{\phi}^k_{h-1}(s',a')^\top \int_{\cS} \hat{\psi}^k_{h-1}(s)  g(s,\pi_h(s))\mathrm{d} s\right]\right|\\
&\qquad\leq  \EE_{d^{\pi, \hat{\PP}^k}_{h-1}}\left\|\hat{\phi}^k_{h-1}\right\|_{\Sigma_{\tilde{\rho}^k_{h-1}, \hat{\phi}^k_{h-1}}^{-1}} \cdot \left\| \int_{\cS} \hat{\psi}^k_{h-1}(s) g(s,\pi_h(s))\mathrm{d} s\right\|_{\Sigma_{\tilde{\rho}^k_{h-1}, \hat{\phi}^k_{h-1}}},
\end{aligned}
\end{align}
where the inequality is due to the Cauchy-Schwarz inequality. Hereafter, we define the covariance matrix $\Sigma_{\tilde{\rho}^k_{h-1}, \hat{\phi}^k_{h-1}} :=k \EE_{(s,a)\sim\tilde{\rho}^k_{h-1}}[\hat{\phi}^k_{h-1}(s,a)\hat{\phi}^k_{h-1}(s,a)^\top] + \lambda_k I$ with $\tilde{\rho}^k_{h-1}(s,a) = \frac{1}{k}\sum_{k'=0}^{k-1} \tilde{d}^{\pi^{k'}}_{h-1}(s,a)$.

Next, we can bound
\begin{align}
\begin{aligned}\label{eq:step-back-mdp2}
&\left\| \int_{\cS} \hat{\psi}^k_{h-1}(s) g(s,\pi_h(s))\mathrm{d} s\right\|_{\Sigma_{\tilde{\rho}^k_{h-1},\hat{\phi}^k_{h-1}}}^2\\
&\qquad = k\left(\int_{\cS} \hat{\psi}^k_{h-1}(s) g(s,\pi_h(s))\mathrm{d}s\right)^\top  \EE_{\tilde{\rho}^k_{h-1}}\left[\hat{\phi}^k_{h-1}(\hat{\phi}^k_{h-1})^\top \right]  \left(\int_{\cS} \hat{\psi}^k_{h-1}(s) g(s,\pi_h(s))\mathrm{d}s\right)\\
&\qquad \quad + \lambda_k\left(\int_{\cS} \hat{\psi}^k_{h-1}(s) g(s,\pi_h(s))\mathrm{d}s\right)^\top \left(\int_{\cS} \hat{\psi}^k_{h-1}(s) g(s,\pi_h(s))\mathrm{d}s\right)\\
&\qquad = k \EE_{(s'',a'')\sim\tilde{\rho}^k_{h-1}(\cdot,\cdot)}\left[\int_{\cS} \hat{\phi}^k_{h-1}(s'',a'')^\top \hat{\psi}^k_{h-1}(s) g(s,\pi_h(s))\mathrm{d}s  \right] \\
&\qquad \quad + \lambda_k\left(\int_{\cS} \hat{\psi}^k_{h-1}(s) g(s,\pi_h(s))\mathrm{d}s\right)^\top \left(\int_{\cS} \hat{\psi}^k_{h-1}(s) g(s,\pi_h(s))\mathrm{d}s\right)\\
&\qquad \leq k \EE_{(s'',a'')\sim\tilde{\rho}^k_{h-1}(\cdot,\cdot)}\left[\int_{\cS} \hat{\phi}^k_{h-1}(s'',a'')^\top \hat{\psi}^k_{h-1}(s) g(s,\pi_h(s))\mathrm{d}s  \right]^2  + \lambda_k B^2 d/(\CS)^2,
\end{aligned}
\end{align}
where the last inequality is due to 
\begin{align*}
\left(\int_{\cS} \hat{\psi}^k_{h-1}(s) g(s,\pi_h(s))\mathrm{d}s\right)^\top \left(\int_{\cS} \hat{\psi}^k_{h-1}(s) g(s,\pi_h(s))\mathrm{d}s\right) \leq B^2 \left|\int_{\cS} \hat{\psi}^k_{h-1}(s) \mathrm{d}s\right|_2^2\leq B^2 d/(\CS)^2,
\end{align*}
since $\|\int_{\cS} \hat{\psi}_{h-1}^k(s) \mathrm{d}s\|_2^2 :=\|\int_{\cS}  \cPS(s)  \tilde{\psi}_{h-1}^k (s)\mathrm{d}s\|_2^2 \leq \|\int_{\cS}  \tilde{\psi}_{h-1}^k (s)\mathrm{d}s\|_2^2 \leq (\int_{\cS} \|  \tilde{\psi}_{h-1}^k (s)\|_2\mathrm{d}s)^2 \leq d/(\CS)^2$ according to the definition of the function class in Definition \ref{def:func-class} and the assumption that all states are normalized such that $\mathrm{Vol}(\cS)\leq 1$.  Moreover, we have
\begin{align}
\begin{aligned}\label{eq:step-back-mdp3}
&k \EE_{(s'',a'')\sim\tilde{\rho}^k_{h-1}(\cdot,\cdot)}\left[\int_{\cS} \hat{\phi}^k_{h-1}(s'',a'')^\top \hat{\psi}^k_{h-1}(s) g(s,\pi_h(s))\mathrm{d}s  \right]^2 \\
&\qquad \leq  2k \EE_{(s'',a'')\sim\tilde{\rho}^k_{h-1}(\cdot,\cdot)}\left[\int_{\cS} \left(\hat{\PP}^k_{h-1}(s|s'',a'') -\PP_{h-1}(s|s'',a'') \right) g(s,\pi_h(s))\mathrm{d}s  \right]^2 \\
&\qquad \quad + 2k \EE_{(s'',a'')\sim\tilde{\rho}^k_{h-1}(\cdot,\cdot)}\left[\int_{\cS} \PP_{h-1}(s|s'',a'') g(s,\pi_h(s))\mathrm{d}s  \right]^2\\
&\qquad \leq  2k B^2 \zeta^k_{h-1}  + 2k \EE_{(s'',a'')\sim\tilde{\rho}^k_{h-1}(\cdot,\cdot)}\left[\int_{\cS} \PP_{h-1}(s|s'',a'') g(s,\pi_h(s))\mathrm{d}s  \right]^2\\
&\qquad \leq  2k B^2 \zeta^k_{h-1}  + 2k \EE_{(s'',a'')\sim\tilde{\rho}^k_{h-1}(\cdot,\cdot), s\sim \PP_{h-1}(\cdot|s'',a'')  }[g(s,\pi_h(s))^2]\\
&\qquad \leq  2k B^2 \zeta^k_{h-1}  + 2k \frac{1}{\Unif(a)} \EE_{(s,a)\sim\breve{\rho}^k_h(\cdot,\cdot)}[ g(s,a)^2]\\
&\qquad =  2k B^2 \zeta^k_{h-1} + 2k  |\cA|\cdot  \EE_{(s,a)\sim\breve{\rho}^k_h(\cdot,\cdot)}[ g(s,a)^2],
\end{aligned}
\end{align}
where the first inequality is due to $(x+y)^2\leq 2x^2 + 2y^2$, the second inequality is due to $\EE_{(s'',a'')\sim\tilde{\rho}^k_{h-1}(\cdot,\cdot)}\allowbreak [\int_{\cS} (\hat{\PP}^k_{h-1}(s|s'',a'') -\PP_{h-1}(s|s'',a'') ) g(s,\pi_h(s))\mathrm{d}s  ]^2 \leq B^2\EE_{(s'',a'')\sim\tilde{\rho}^k_{h-1}(\cdot,\cdot)}\allowbreak \|\hat{\PP}^k_{h-1}(\cdot|s'',a'') -\PP_{h-1}(\cdot|s'',a'')\|_1^2 \allowbreak= B^2 \zeta^k_{h-1}$ with $\zeta^k_{h-1}:=\EE_{(s'',a'')\sim\tilde{\rho}^k_{h-1}(\cdot,\cdot)}\allowbreak \|\hat{\PP}^k_{h-1}(\cdot|s'',a'') -\PP_{h-1}(\cdot|s'',a'')\|_1^2$, the third inequality is by Jensen's inequality, and the fourth inequality is due to $g(s,\pi_h(s))^2 \leq \sum_{a\in \cA } g(s,a)^2 = 1/\Unif(a) \cdot \EE_{a\sim\Unif(\cA)}[g(s,a)^2]$ and $\breve{\rho}^k_h(s,a):=\tilde{\rho}^k_{h-1}(s',a')\PP_{h-1}(s|s',a')\Unif(a)$ for all $h\geq 2$.

Combining \eqref{eq:step-back-mdp1},\eqref{eq:step-back-mdp2}, and \eqref{eq:step-back-mdp3}, we have for any $h \geq 2$,
\begin{align*}
   &\left|\EE_{(s,a)\sim d^{\pi, \hat{\PP}^k}_h(\cdot,\cdot)}[g(s,a)]\right| \\
&\qquad \leq \sqrt{2k B^2 \zeta^k_{h-1} + 2k  |\cA|\cdot  \EE_{(s,a)\sim\breve{\rho}^k_h(\cdot,\cdot)}[ g(s,a)^2] + \lambda_k B^2 d/(\CS)^2} \cdot \EE_{ d^{\pi, \hat{\PP}^k}_{h-1}}\left\|\hat{\phi}^k_{h-1}\right\|_{\Sigma_{\tilde{\rho}^k_{h-1}, \hat{\phi}^k_{h-1}}^{-1}}.
\end{align*}
For $h=1$, we have
\begin{align*}
\left|\EE_{(s,a)\sim d^{\pi, \hat{\PP}^k}_1(\cdot,\cdot)}[g(s,a)]\right| = \sqrt{g(s_1,\pi_1(s_1))^2} \leq \sqrt{|\cA|  \EE_{a\sim \tilde{\rho}_1^k(s_1,\cdot)}[g(s_1,a)^2] },
\end{align*} 
where we let $\tilde{\rho}_1^k(s_1,a) =\Unif(a)$.
Note that the above derivations also hold for any randomized policy $\pi$.  The proof is completed.
\end{proof}

\begin{lemma}\label{lem:expand2} Define $\tilde{\rho}^k_h(\cdot,\cdot) := \frac{1}{k}\sum_{k'=0}^{k-1} \tilde{d}^{\pi^{k'}}_h(\cdot,\cdot)$ for all $h\geq 1$ with $\tilde{\rho}_1^k(s_1,a) =\Unif(a)$ and $\rho^k_h(\cdot,\cdot) := \frac{1}{k}\sum_{k'=0}^{k-1} d^{\pi^{k'}}_h(\cdot,\cdot)$ for all $h\geq 2$.  Then for any function $g:\cS\times\cA\mapsto [0,B]$ and policy $\pi$, we have for any $h \geq 2$, the following inequality holds
\begin{align*}
   &\left|\EE_{(s,a)\sim d^{\pi, \PP}_h(\cdot,\cdot)}[g(s,a)]\right| \\
&\qquad\leq \sqrt{k  |\cA|\cdot  \EE_{(s,a)\sim\tilde{\rho}^k_h(\cdot,\cdot)}[ g(s,a)^2] + \lambda_k B^2 d} \cdot \EE_{d^{\pi, \PP}_{h-1}}\left\|\phi^*_{h-1}\right\|_{\Sigma_{\rho^k_{h-1}, \phi^*_{h-1}}^{-1}}.
\end{align*}
Moreover, for $h=1$, we have
\begin{align*}
\left|\EE_{(s,a)\sim d^{\pi, \PP}_1(\cdot,\cdot)}[g(s,a)]\right| &\leq \sqrt{g(s_1,\pi_1(s_1))^2} \leq \sqrt{|\cA|  \EE_{a\sim \tilde{\rho}_1^k(s_1,\cdot)}[g(_1s,a)^2] }.
\end{align*} 
\end{lemma}

\begin{proof}
For any function $g:\cS\times\cA\mapsto [0,B]$ and any deterministic policy $\pi$, under the true transition model $\PP$, for any $h\geq 2$, we have
\begin{align}
\begin{aligned}\label{eq:step-back2-mdp1}
&\left|\EE_{(s,a)\sim d^{\pi, \PP}_h(\cdot,\cdot)}[g(s,a)]\right| \\
&\qquad= \left|\EE_{(s',a')\sim d^{\pi, \PP}_{h-1}(\cdot,\cdot), s\sim \PP_{h-1}(\cdot|s',a')}[g(s, \pi_h(s))]\right| \\
&\qquad= \left|\EE_{(s',a')\sim d^{\pi, \PP}_{h-1}(\cdot,\cdot)} \left[\phi^*_{h-1}(s',a')^\top \int_{\cS} \psi^*_{h-1}(s) g(s,\pi_h(s))\mathrm{d} s\right]\right|\\
&\qquad\leq  \EE_{(s',a')\sim d^{\pi, \PP}_{h-1}(\cdot,\cdot)}\left\|\phi^*_{h-1}(s',a')\right\|_{\Sigma_{\rho^k_{h-1}, \phi^*_{h-1}}^{-1}} \left\| \int_{\cS} \psi^*_{h-1}(s) g(s,\pi_h(s))\mathrm{d} s\right\|_{\Sigma_{\rho^k_{h-1}, \phi^*_{h-1}}},
\end{aligned}
\end{align}
where the inequality is due to the Cauchy-Schwarz inequality. Here, we define the covariance matrix $\Sigma_{\rho^k_{h-1}, \phi^*_{h-1}} :=k \EE_{(s,a)\sim\rho^k_{h-1}}[\phi^*_{h-1}(s,a)\phi^*_{h-1}(s,a)^\top] + \lambda_k I$ with $\rho^k_{h-1}(s,a) = \frac{1}{k}\sum_{k'=0}^{k-1} d^{\pi^{k'}}_{h-1}(s,a)$.

Next, we have
\begin{align}
\begin{aligned}\label{eq:step-back2-mdp2}
&\left\| \int_{\cS} \psi^*_{h-1}(s) g(s,\pi_h(s))\mathrm{d} s\right\|_{\Sigma_{\rho^k_{h-1},\phi^*_{h-1}}}^2\\
&\qquad = k\left(\int_{\cS} \psi^*_{h-1}(s) g(s,\pi_h(s))\mathrm{d}s\right)^\top  \EE_{\rho^k_{h-1}}\left[\phi^*_{h-1}(\phi^*_{h-1})^\top \right]  \left(\int_{\cS} \psi^*_{h-1}(s) g(s,\pi_h(s))\mathrm{d}s\right)\\
&\qquad \quad + \lambda_k\left(\int_{\cS} \psi^*_{h-1}(s) g(s,\pi_h(s))\mathrm{d}s\right)^\top \left(\int_{\cS} \psi^*_{h-1}(s) g(s,\pi_h(s))\mathrm{d}s\right)\\
&\qquad = k \EE_{(s'',a'')\sim\rho^k_{h-1}(\cdot,\cdot)}\left[\int_{\cS} \phi^*_{h-1}(s'',a'')^\top \psi^*_{h-1}(s) g(s,\pi_h(s))\mathrm{d}s  \right] \\
&\qquad \quad + \lambda_k\left(\int_{\cS} \psi^*_{h-1}(s) g(s,\pi_h(s))\mathrm{d}s\right)^\top \left(\int_{\cS} \psi^*_{h-1}(s) g(s,\pi_h(s))\mathrm{d}s\right)\\
&\qquad \leq k \EE_{(s'',a'')\sim\rho^k_{h-1}(\cdot,\cdot)}\left[\int_{\cS} \phi^*_{h-1}(s'',a'')^\top \psi^*_{h-1}(s) g(s,\pi_h(s))\mathrm{d}s  \right]^2  + \lambda_k B^2 d,
\end{aligned}
\end{align}
where, by Assumption \ref{assump:low-rank}, the last inequality is due to 
\begin{align*}
\left(\int_{\cS} \psi^*_{h-1}(s) g(s,\pi_h(s))\mathrm{d}s\right)^\top \left(\int_{\cS} \psi^*_{h-1}(s) g(s,\pi_h(s))\mathrm{d}s\right) \leq B^2 \left|\int_{\cS} \psi^*_{h-1}(s) \mathrm{d}s\right|_2^2\leq B^2 d.
\end{align*}
Furthermore, we have
\begin{align}
\begin{aligned}\label{eq:step-back2-mdp3}
&k \EE_{(s'',a'')\sim\rho^k_{h-1}(\cdot,\cdot)}\left[\int_{\cS} \phi^*_{h-1}(s'',a'')^\top \psi^*_{h-1}(s) g(s,\pi_h(s))\mathrm{d}s  \right]^2 \\
&\qquad = k \EE_{(s'',a'')\sim\rho^k_{h-1}(\cdot,\cdot)}\left[\int_{\cS} \PP_{h-1}(s|s'',a'') g(s,\pi_h(s))\mathrm{d}s  \right]^2\\
&\qquad \leq  k \EE_{(s'',a'')\sim\rho^k_{h-1}(\cdot,\cdot), s\sim \PP_{h-1}(\cdot|s'',a'') }[g(s,\pi_h(s))^2]\\
&\qquad \leq k \frac{1}{\Unif(a)} \EE_{(s,a)\sim\tilde{\rho}^k_h(\cdot,\cdot)}[ g(s,a)^2] =  k  |\cA|\cdot  \EE_{(s,a)\sim\tilde{\rho}^k_h(\cdot,\cdot)}[ g(s,a)^2],
\end{aligned}
\end{align}
where the first inequality is due to Jensen's inequality and the second inequality is by $g(s,\pi_h(s))^2 \leq \sum_{a\in \cA } g(s,a)^2 = 1/\Unif(a) \cdot \EE_{a\sim\Unif(\cA)}[g(s,a)^2]$ and $\tilde{\rho}^k_h(s,a):=\rho^k_{h-1}(s',a')\PP_{h-1}(s|s',a')\Unif(a)$ for all $h\geq 2$.

Combining \eqref{eq:step-back2-mdp1},  \eqref{eq:step-back2-mdp2}, and \eqref{eq:step-back2-mdp3}, we have for any $h \geq 2$,
\begin{align*}
   &\left|\EE_{(s,a)\sim d^{\pi, \PP}_h(\cdot,\cdot)}[g(s,a)]\right| \\
&\qquad\leq \sqrt{k  |\cA|\cdot  \EE_{(s,a)\sim\tilde{\rho}^k_h(\cdot,\cdot)}[ g(s,a)^2] + \lambda_k B^2 d} \cdot \EE_{d^{\pi, \PP}_{h-1}}\left\|\phi^*_{h-1}\right\|_{\Sigma_{\rho^k_{h-1}, \phi^*_{h-1}}^{-1}}.
\end{align*}
For $h=1$, we have
\begin{align*}
\left|\EE_{(s,a)\sim d^{\pi, \PP}_1(\cdot,\cdot)}[g(s,a)]\right| &\leq \sqrt{g(s_1,\pi_1(s_1))^2} \leq \sqrt{|\cA|  \EE_{a\sim \tilde{\rho}_1^k(s_1,\cdot)}[g(_1s,a)^2] },
\end{align*} 
where we define $\tilde{\rho}_1^k(s_1,a) =\Unif(a)$. The above derivations also hold for any randomized policy $\pi$. The proof is completed.
\end{proof}

\begin{lemma}\label{lem:plan} Let $\pi^*:=\argmax_{\pi} V_1^\pi(s_1)$, $\overline{V}_1^k(s_1)$ be the value function updated in Algorithm \ref{alg:contrastive}, and $\overline{V}_{k,1}^\pi$ be the value function under any policy $\pi$ associated with an MDP defined by the reward function $r + \beta^k$ and the estimated transition $\hat{\PP}^k$ with $\beta^k$ and $\hat{\PP}^k$ obtained at episode $k$ of Algorithm \ref{alg:contrastive}. Then we have
\begin{align*}
\overline{V}_1^k(s_1) \geq \overline{V}_{k,1}^{\pi^*}(s_1).
\end{align*}
\end{lemma}

\begin{proof}
We prove this lemma by induction. First, we have $\overline{V}_{H+1}^k(s) = \overline{V}_{k,H+1}^\pi(s) = 0$ for any $s\in \cS$ and any (randomized) policy $\pi$ such that the Bellman equation is written as 
$Q_h^\pi(s,a) = r_h(s,a) + \PP_hV_{h+1}^\pi(s,a)$ and $V_h^\pi(s) = \EE_{a\sim\pi(\cdot|s)}[Q_h^\pi(s,a)]$. Here, we aim to prove this lemma holds for any policy $\pi$, we slightly abuse the notation $\pi$ and let $\pi_h(a|s)$ be the probability of taking action $a$ under the state $s$. Next, we assume the following inequality holds 
\begin{align*}
\overline{V}_{h+1}^k(s) \geq \overline{V}_{k, h+1}^\pi(s) .
\end{align*}
Then, with the above inequality, by the Bellman equation, we have
\begin{align}
\begin{aligned}
\label{eq:opt-mdp-1}
&\overline{Q}_h^k(s,a) - \overline{Q}_{k,h}^\pi(s,a)\\
&\qquad =  r_h(s,a) +  \beta_h^k(s,a) + \hat{\PP}_h^k\overline{V}_{h+1}^k(s,a) - r_h(s,a) -  \beta_h^k(s,a) - \hat{\PP}_h^k\overline{V}_{k,h+1}^\pi(s,a) \\
&\qquad =  \hat{\PP}_h^k\overline{V}_{h+1}^k(s) - \hat{\PP}_h^k\overline{V}_{k,h+1}^\pi(s)\geq 0.
\end{aligned}
\end{align}
Then, we have
\begin{align*}
\overline{V}_h^k(s) &=\max_{a\in \cA} \overline{Q}_h^k(s,a)\\
&\geq \max_{a\in \cA} \overline{Q}_{k,h}^\pi(s,a)\\
&\geq \EE_{a\sim\pi_h(\cdot|s)}[\overline{Q}_{k,h}^\pi(s,a)]  = \overline{V}_{k,h}^\pi(s) ,
\end{align*}
where the first inequality is by \eqref{eq:opt-mdp-1} and the second inequality is due to the fact that $\max_i \vb_i \geq \langle \vb, \db\rangle$ when $\vb$ is any vector and $\db$ is a vector in a probability simplex satisfying $\sum_i \db_i = 1$ and $\db_i\geq 0$. Thus, we obtain for any policy $\pi$,
\begin{align*}
\overline{V}_1^k(s_1) \geq \overline{V}_{k,1}^\pi(s_1),
\end{align*}
which further implies
\begin{align*}
\overline{V}_1^k(s_1) \geq \overline{V}_{k,1}^{\pi^*}(s_1).
\end{align*}
This completes the proof.
\end{proof}

\subsection{Proof of Lemma \ref{lem:stat-err}} \label{sec:proof-stat-err}

\begin{proof}
For any function $f_h\in \cF$, we let $\Pr_h^f(y|s,a,s')$ denote the conditional probability characterized by the function $f_h$ at the step $h$, which is
\begin{align*}
\Pr{}_h^f(y|s,a,s') =  \left(\frac{f_h(s,a,s')}{1+f_h(s,a,s')}\right)^y\left(\frac{1}{1+f_h(s,a,s')}\right)^{1-y}.   
\end{align*}
Furthermore, we have
\begin{align*}
\Pr{}_h^f(y,s'|s,a)  = \Pr{}_h^f(y|s,a,s')\Pr{}_h(s'|s,a) = \left(\frac{f_h(s,a,s')\Pr{}_h(s'|s,a)}{1+f_h(s,a,s')}\right)^y\left(\frac{\Pr{}_h(s'|s,a)}{1+f_h(s,a,s')}\right)^{1-y},
\end{align*}
where we have
\begin{align}
\begin{aligned}\label{eq:tran-lower}
    \Pr{}_h(s'|s,a) &= \Pr{}_h(y=1| s,a) \Pr{}_h(s'|y=1, s,a) + \Pr{}_h(y = 0| s,a) \Pr{}_h(s'|y=0, s,a) \\
    &= \Pr{}_h(y=1) \Pr{}_h(s'|y=1, s,a) + \Pr{}_h(y = 0) \Pr{}_h(s'|y=0, s,a) \\
    &= \frac{1}{2} [\PP_h(s'|s,a) + \cPS(s')] \geq \frac{1}{2}\CS > 0,
\end{aligned}
\end{align}
since we assume $P_\cS^-(s')\geq \CS$.

Thus, we have the equivalency of solving the following two problems with $f_h(s,a,s')=\phi_h(s,a)^\top \psi_h(s')$, which is
\begin{align}\label{eq:equi-solution}
\max_{\phi_h\in \Phi, \psi_h\in \Psi} \sum_{(s,a,s',y)\in \cD_h^k} \log \Pr{}_h^f(y|s,a,s') = \max_{\phi_h, \psi_h} \sum_{(s,a,s',y)\in \cD_h^k} \log \Pr{}_h^f(y,s'|s,a),
\end{align}
since the conditional probability $\Pr_h(s'|s,a)$ is only  determined by $\PP_h(s'|s,a)$ and $\cPS(s')$ and is independent of $f_h$ as shown in \eqref{eq:tran-lower}. We denote the solution of \eqref{eq:equi-solution} as $\tilde{\phi}_h^k$ and $\tilde{\psi}_h^k$ such that 
\begin{align*}
\hat{f}_h^k(s,a,s') = \tilde{\psi}^k_h(s')^\top \tilde{\phi}^k_h(s,a).
\end{align*} 
According to Algorithm \ref{alg:sample}, we know that for each $h\geq 2$, at each episode $k'\in [k]$, the data $(s,a)$ is sampled from both $\tilde{d}_h^{\pi^{k'}}(\cdot,\cdot)$ and $\breve{d}_h^{\pi^{k'}}(\cdot,\cdot)$. Therefore, further with Lemma \ref{lem:recover-mle}, by solving the contrastive loss in \eqref{eq:contra-loss} or equivalently as in \eqref{eq:equi-solution}, with probability at least $1-\delta$, for all $h\geq 2$, we have
\begin{align*}
\sum_{k'=1}^k  \Bigg[ &\EE_{(s,a)\sim \tilde{d}_h^{\pi^{k'}}(\cdot,\cdot)} \left\|\Pr{}_h^{\hat{f}^k}(\cdot,\cdot|s,a) - \Pr{}_h^{f^*}(\cdot,\cdot|s,a) \right\|_{\TV}^2 \\
&+ \EE_{(s,a)\sim \breve{d}_h^{\pi^{k'}}(\cdot,\cdot)} \left\|\Pr{}_h^{\hat{f}^k}(\cdot,\cdot|s,a) - \Pr{}_h^{f^*}(\cdot,\cdot|s,a) \right\|_{\TV}^2 \Bigg] \leq 2\log (2kH|\cF|/\delta),
\end{align*}
where the factor $2H$ inside $\log$ is due to the data being sampled from two distributions and applying union bound for all $h\geq 2$. The above inequality is equivalent to 
\begin{align}
\begin{aligned}\label{eq:ave-mle-bound1} 
&\EE_{(s,a)\sim \tilde{\rho}_h^k(\cdot,\cdot)} \left\|\Pr{}_h^{\hat{f}^k}(\cdot,\cdot|s,a) - \Pr{}_h^{f^*}(\cdot,\cdot|s,a) \right\|_{\TV}^2 \\
&\qquad + \EE_{(s,a)\sim \breve{\rho}_h^k(\cdot,\cdot)} \left\|\Pr{}_h^{\hat{f}^k}(\cdot,\cdot|s,a) - \Pr{}_h^{f^*}(\cdot,\cdot|s,a) \right\|_{\TV}^2  \leq 2\log (2kH|\cF|/\delta)/k, \quad \forall h\geq 2,
\end{aligned}
\end{align}
where we use the fact that $\tilde{\rho}^k_h(s,a) = \frac{1}{k}\sum_{k'=0}^{k-1} \tilde{d}^{\pi^{k'}}_h(s,a)$ and $\breve{\rho}^k_h(s,a) = \frac{1}{k}\sum_{k'=0}^{k-1} \breve{d}^{\pi^{k'}}_h(s,a)$. On the other hand, for $h=1$, the data is only sampled from $\tilde{d}^{\pi^{k'}}_1(\cdot,\cdot)$ for any $k'\in [k]$. Therefore,  we have
\begin{align*}
\sum_{k'=1}^k  \Bigg[ &\EE_{(s,a)\sim \tilde{d}_1^{\pi^{k'}}(\cdot,\cdot)} \left\|\Pr{}_1^{\hat{f}^k}(\cdot,\cdot|s,a) - \Pr{}_1^{f^*}(\cdot,\cdot|s,a) \right\|_{\TV}^2  \Bigg] \leq 2\log (2k|\cF|/\delta),
\end{align*}
which, analogously, gives
\begin{align} \label{eq:ave-mle-bound2} 
&\EE_{(s,a)\sim \tilde{\rho}_1^k(\cdot,\cdot)} \left\|\Pr{}_1^{\hat{f}^k}(\cdot,\cdot|s,a) - \Pr{}_1^{f^*}(\cdot,\cdot|s,a) \right\|_{\TV}^2  \leq 2\log (2k|\cF|/\delta)/k.
\end{align}
Thus, by \eqref{eq:ave-mle-bound1} and \eqref{eq:ave-mle-bound2},   with probability at least $1-2\delta$, we have
\begin{align}
\begin{aligned}\label{eq:ave-mle-bound} 
&\EE_{(s,a)\sim \tilde{\rho}_h^k(\cdot,\cdot)} \left\|\Pr{}_h^{\hat{f}^k}(\cdot,\cdot|s,a) - \Pr{}_h^{f^*}(\cdot,\cdot|s,a) \right\|_{\TV}^2 \leq 2\log (2kH|\cF|/\delta)/k, \quad \forall h\geq 1,\\
&\EE_{(s,a)\sim \breve{\rho}_h^k(\cdot,\cdot)} \left\|\Pr{}_h^{\hat{f}^k}(\cdot,\cdot|s,a) - \Pr{}_h^{f^*}(\cdot,\cdot|s,a) \right\|_{\TV}^2  \leq 2\log (2kH|\cF|/\delta)/k, \quad \forall h\geq 2,
\end{aligned}
\end{align}
Next, we show the recovery error bound of the transition model based on \eqref{eq:ave-mle-bound}. We have
\begin{align*}
 &\left\|\Pr{}_h^{\hat{f}^k}(\cdot,\cdot|s,a) - \Pr{}_h^{f^*}(\cdot,\cdot|s,a) \right \|_{\TV}^2\\
 &\quad = \left( \left\|\Pr{}_h^{\hat{f}^k}(y=0,\cdot|s,a) - \Pr{}_h^{f^*}(y=0,\cdot|s,a) \right \|_{\TV} + \left\|\Pr{}_h^{\hat{f}^k}(y=1,\cdot|s,a) - \Pr{}_h^{f^*}(y=1,\cdot|s,a) \right \|_{\TV}   \right)^2 \\
&\quad = 4\left\|\frac{\Pr{}_h(\cdot|s,a)}{1+\hat{f}_h^k(s,a,\cdot)} - \frac{\Pr{}_h(\cdot|s,a)}{1+f_h^*(s,a,\cdot)} \right \|_{\TV}^2\\
&\quad = 2 \left[\int_{s'\in \cS}\frac{\Pr{}_h(s'|s,a)\cdot |f_h^*(s,a,s')-\hat{f}_h^k(s,a,s')|}{[1+\hat{f}_h^k(s,a,s')]\cdot[1+f_h^*(s,a,s')]}\mathrm{d}s' \right]^2,
\end{align*}
where $f^*(s,a,s') = \frac{\PP(s'|s,a)}{\cPS(s')}  ~~\text{with}~~ \cPS(s') \geq \CS,~~\forall s'\in\cS$ and the second equation is due to $\|\Pr{}_h^{\hat{f}^k}(y=0,\cdot|s,a) - \Pr{}_h^{f^*}(y=0,\cdot|s,a)  \|_{\TV} = \|\Pr{}_h^{\hat{f}^k}(y=1,\cdot|s,a) - \Pr{}_h^{f^*}(y=1,\cdot|s,a)  \|_{\TV} = \Big\|\frac{\Pr{}_h(\cdot|s,a)}{1+\hat{f}_h^k(s,a,\cdot)} - \frac{\Pr{}_h(\cdot|s,a)}{1+f_h^*(s,a,\cdot)} \Big \|_{\TV}$. 
Moreover, according to Lemma \ref{lem:opt-contra-loss} and \eqref{eq:tran-lower}, we have 
\begin{align*}
&\frac{\Pr{}_h(s'|s,a)\cdot |f_h^*(s,a,s')-\hat{f}_h^k(s,a,s')|}{[1+\hat{f}_h^k(s,a,s')]\cdot[1+f_h^*(s,a,s')]}\\
&\qquad =\frac{1/2 \cdot [\PP_h(s'|s,a) + \cPS(s')]\cdot |\PP_h(s'|s,a)/\cPS(s')-\hat{f}_h^k(s,a,s')|}{[1+\hat{f}_h^k(s,a,s')]\cdot[1+\PP_h(s'|s,a)/\cPS(s')]}\\
&\qquad =\frac{1/2 \cdot |\PP_h(s'|s,a)-\cPS(s') \hat{f}_h^k(s,a,s')|}{1+\hat{f}_h^k(s,a,s')} \geq  \frac{|\PP_h(s'|s,a)-\cPS(s') \hat{f}_h^k(s,a,s')|}{4\sqrt{d}/\CS},
\end{align*}
where the inequality is due to $[1+\hat{f}_h^k(s,a,s')]\leq (1+\sqrt{d}/\CS)\leq 2\sqrt{d}/\CS$ since $\hat{f}_h^k(s,a,s')\leq \sqrt{d}/\CS$ with $d\geq 1$ and $0<\CS\leq 1$. Thus, the above results further give
\begin{align*}
 \frac{(\CS)^2}{8d}\left[\int_{s'\in\cS}\left|\PP_h(s'|s,a)-\cPS(s') \hat{f}_h^k(s,a,s')\right| \mathrm{d}s' \right]^2 \leq \left\|\Pr{}_h^{\hat{f}^k}(\cdot,\cdot|s,a) - \Pr{}_h^{f^*}(\cdot,\cdot|s,a) \right \|_{\TV}^2.
\end{align*}
Therefore, combining this inequality with \eqref{eq:ave-mle-bound}, we obtain that for all $h\geq 2$,
\begin{align}
    &\mathbb{E}_{(s,a)\sim \tilde{\rho}_h^k(\cdot,\cdot)} \left\|\PP_h(\cdot|s,a)-\cPS(\cdot) \tilde{\phi}_h^k(s,a)^\top \tilde{\psi}_h^k (\cdot) \right\|_{\TV}^2 \nonumber  \\ 
    &\qquad = 1/2\cdot\mathbb{E}_{(s,a)\sim \tilde{\rho}_h^k(\cdot,\cdot)} \left[\int_{s'\in\cS}\left|\PP_h(s'|s,a)-\cPS(s') \tilde{\phi}_h^k(s,a)^\top \tilde{\psi}_h^k (s')\right| \mathrm{d}s' \right]^2 \nonumber \\
    &\qquad\leq 8d/(\CS)^2\cdot\log (2kH|\cF|/\delta)/k. \label{eq:init-P-diff-1}
\end{align}
Similarly, we can obtain
\begin{align}
\begin{aligned}\label{eq:init-P-diff-2}
&\EE_{(s,a)\sim \tilde{\rho}_1^k(\cdot,\cdot)} \left\|\PP_h(\cdot|s,a)-\cPS(\cdot) \tilde{\phi}_h^k(s,a)^\top \tilde{\psi}_h^k (\cdot)\right\|_{\TV}^2 \leq 8d/(\CS)^2\cdot\log (2k|\cF|/\delta)/k, \\
&\EE_{(s,a)\sim \breve{\rho}_h^k(\cdot,\cdot)} \left\|\PP_h(\cdot|s,a)-\cPS(\cdot) \tilde{\phi}_h^k(s,a)^\top \tilde{\psi}_h^k (\cdot)\right\|_{\TV}^2  \leq 8d/(\CS)^2\cdot\log (2kH|\cF|/\delta)/k, \quad \forall h\geq 2.
\end{aligned}
\end{align}
Now we define 
\begin{align*}
\hat{g}_h^k(s,a,s') := \cPS(s') \tilde{\phi}_h^k(s,a)^\top \tilde{\psi}_h^k (s').
\end{align*} 
Note that $\int_{s'\in\cS}\hat{g}_h^k(s,a,s')\mathrm{d}s'$ may not be guaranteed to be $1$ though $\hat{g}_h^k(s,a,\cdot)$ is close to the true transition model $\PP_h(\cdot|s,a)$ according to \eqref{eq:init-P-diff-1} and \eqref{eq:init-P-diff-2}. Therefore, to obtain an approximator of the transition model $\PP_h$ lying on a probability simplex, we should further normalize $\hat{g}_h^k(s,a,s')$. Thus, we define for all $(s,a,s')\in \cS\times\cA\times\cS$,
\begin{align*}
\hat{\PP}_h^k(s'|s,a) := \frac{\hat{g}_h^k(s,a,s')}{\|\hat{g}_h^k(s,a,\cdot)\|_1}=  \frac{\hat{g}_h^k(s,a,s')}{\int_{s'\in\cS}\hat{g}_h^k(s,a,s')\mathrm{d}s'} = \frac{\cPS(s') \tilde{\phi}_h^k(s,a)^\top \tilde{\psi}_h^k (s')}{\int_{s'\in\cS}\cPS(s') \tilde{\phi}_h^k(s,a)^\top \tilde{\psi}_h^k (s')\mathrm{d}s'}.
\end{align*}
We further let 
\begin{align*}
    \hat{\phi}_h^k(s,a) :=  \tilde{\phi}_h^k(s,a)\big/\textstyle \int_{s'\in\cS}\cPS(s') \tilde{\phi}_h^k(s,a)^\top \tilde{\psi}_h^k (s')\mathrm{d}s'   , \qquad \hat{\psi}_h^k(s') :=\cPS(s')  \tilde{\psi}_h^k (s'),
\end{align*}
such that
\begin{align*}
    \hat{\PP}_h^k(s'|s,a) = \hat{\psi}_h^k(s')^\top \hat{\phi}_h^k(s,a).
\end{align*}
Next, based on the above definitions and results, we will give the upper bound of the approximation error $\mathbb{E}_{(s,a)\sim \tilde{\rho}_h^k(\cdot,\cdot)}\|\hat{\PP}_h^k(\cdot|s,a)- \PP_h(\cdot|s,a)\|_{\TV}^2$. We have
\begin{align}
\begin{aligned}\label{eq:P-diff0}
&\mathbb{E}_{(s,a)\sim \tilde{\rho}_h^k(\cdot,\cdot)}\|\hat{\PP}_h^k(\cdot|s,a)- \PP_h(\cdot|s,a)\|_{\TV}^2 \\
&\qquad \leq 2\mathbb{E}_{(s,a)\sim \tilde{\rho}_h^k(\cdot,\cdot)}\|\hat{\PP}_h^k(\cdot|s,a)- \hat{g}_h^k(s,a,\cdot) \|_{\TV}^2 + 2\mathbb{E}_{(s,a)\sim \tilde{\rho}_h^k(\cdot,\cdot)}\|\hat{g}_h^k(s,a,\cdot) - \PP_h(\cdot|s,a)\|_{\TV}^2  \\
&\qquad \leq 2\mathbb{E}_{(s,a)\sim \tilde{\rho}_h^k(\cdot,\cdot)}\|\hat{\PP}_h^k(\cdot|s,a)- \hat{g}_h^k(s,a,\cdot) \|_{\TV}^2 + 16d/(\CS)^2\cdot\log (2kH|\cF|/\delta)/k,  
\end{aligned}
\end{align}
where the first inequality is by $(x+y)^2\leq 2x^2 + 2y^2$ and the last inequality is by \eqref{eq:init-P-diff-1}. Moreover, we have
\begin{align*}
    &\mathbb{E}_{(s,a)\sim \tilde{\rho}_h^k(\cdot,\cdot)}\|\hat{\PP}_h^k(\cdot|s,a)- \hat{g}_h^k(s,a,\cdot) \|_{\TV}^2\\
    &\qquad = \mathbb{E}_{(s,a)\sim \tilde{\rho}_h^k(\cdot,\cdot)}\left\|\frac{\hat{g}_h^k(s,a,s')}{\|\hat{g}_h^k(s,a,\cdot)\|_1}- \hat{g}_h^k(s,a,\cdot)\right\|_{\TV}^2\\
    &\qquad = \frac{1}{4}\mathbb{E}_{(s,a)\sim \tilde{\rho}_h^k(\cdot,\cdot)}\left(\|\hat{g}_h^k(s,a,\cdot)\|_1- 1\right)^2\\
    &\qquad \leq \frac{1}{4}\mathbb{E}_{(s,a)\sim \tilde{\rho}_h^k(\cdot,\cdot)}\left(\|\hat{g}_h^k(s,a,\cdot)-\PP_h(\cdot|s,a)\|_1 + \|\PP_h(\cdot|s,a)\|_1- 1\right)^2\\
    &\qquad \leq \frac{1}{4}\mathbb{E}_{(s,a)\sim\tilde{\rho}_h^k(\cdot,\cdot)}\|\hat{g}_h^k(s,a,\cdot)-\PP_h(\cdot|s,a)\|_1^2 \\ &\qquad = \mathbb{E}_{(s,a)\sim \tilde{\rho}_h^k(\cdot,\cdot)}\|\hat{g}_h^k(s,a,\cdot)-\PP_h(\cdot|s,a)\|_{\TV}^2\leq 8d/(\CS)^2\cdot\log (2kH|\cF|/\delta)/k.
\end{align*}
Combining the above inequality with \eqref{eq:P-diff0}, we eventually obtain
\begin{align*}
\mathbb{E}_{(s,a)\sim \tilde{\rho}_h^k(\cdot,\cdot)}\|\hat{\PP}_h^k(\cdot|s,a)- \PP_h(\cdot|s,a)\|_{\TV}^2 \leq 32d/(\CS)^2\cdot\log (2kH|\cF|/\delta)/k, \quad \forall h\geq 1.
\end{align*}
Thus, we similarly have
\begin{align*}
&\mathbb{E}_{(s,a)\sim\breve{\rho}_h^k(\cdot,\cdot)}\|\hat{\PP}_h^k(\cdot|s,a)- \PP_h(\cdot|s,a)\|_{\TV}^2 \leq 32d/(\CS)^2\cdot\log (2kH|\cF|/\delta)/k, \quad \forall h\geq 2.
\end{align*}
The above three inequalities hold with probability at least $1-2\delta$. This completes the proof.
\end{proof}

\subsection{Proof of Theorem \ref{thm:main}}\label{sec:proof-thm-main}
\begin{proof} We first decompose the term $V^{\pi^*}_1(s_1) - V^{\pi^k}_1(s_1)$ as follows
\begin{align}
\begin{aligned}\label{eq:decomp-mdp-init}
V^{\pi^*}_1(s_1) - V^{\pi^k}_1(s_1)& = V^{\pi^*}_1(s_1) - \overline{V}^{\pi^*}_{k,1}(s_1)+ \overline{V}^{\pi^*}_{k,1}(s_1)- V^k_1(s_1) + V^k_1(s_1) - V^{\pi^k}_1(s_1) \\
& \leq  V^{\pi^*}_1(s_1) - \overline{V}^{\pi^*}_{k,1}(s_1) +\overline{V}^k_1(s_1) - V^{\pi^k}_1(s_1)\\
& = \underbrace{V^{\pi^*}_1(s_1) - \overline{V}^{\pi^*}_{k,1}(s_1)}_{(i)} +\underbrace{\overline{V}_{k,1}^{\pi^k}(s_1) - V^{\pi^k}_1(s_1)}_{(ii)},
\end{aligned}
\end{align}
where the first inequality is by the result of Lemma \ref{lem:plan} that $\overline{V}^{\pi^*}_{k,1}(s_1)\leq V^k_1(s_1)$ and the second equation is by the definition of $\overline{V}^k_h$ as in Algorithm \ref{alg:contrastive} such that $\overline{V}^k_h = \overline{V}_{k,h}^{\pi^k}$ for any $h\in [H]$. Thus, to bound the term $V^{\pi^*}_1(s_1) - V^{\pi^k}_1(s_1)$, we only need to bound the terms $(i)$ and $(ii)$ as in \eqref{eq:decomp-mdp-init}.

To bound term $(i)$, by Lemma \ref{lem:diff1}, we have
\begin{align}
\begin{aligned}\label{eq:term-i-decomp}
   (i)= V_1^{\pi^*}(s_1) - \overline{V}^{\pi^*}_{k,1}(s_1)&= \EE \left[ \sum_{h=1}^H \left(-\beta^k_h(s_h,a_h) + (\PP_h - \hat{\PP}^k_h )V^{\pi^*}_{h+1}(s_h,a_h)\right) \Bigggiven \pi^*, \hat{\PP}^k \right]\\
    &\leq  \EE \left[ \sum_{h=1}^H \left(-\beta^k_h(s_h,a_h) + H \|\PP_h(\cdot|s_h,a_h) - \hat{\PP}^k_h(\cdot|s_h,a_h) \|_1\right) \Bigggiven \pi^*, \hat{\PP}^k \right], 
\end{aligned}
\end{align}
where the first inequality is by the fact $\sup_{s\in\cS} V^\pi_{h+1}(s) \leq H$. Next, we bound the term $\EE [ \sum_{h=1}^H  H\cdot \allowbreak \|\PP_h(\cdot|s_h,a_h) - \hat{\PP}^k_h(\cdot|s_h,a_h) \|_1 \given \pi^*, \hat{\PP}^k ]$. Note that for the term $\|\PP_h(\cdot|s_h,a_h) - \hat{\PP}^k_h(\cdot|s_h,a_h) \|_1$, we first have a trivial bound that $\|\PP_h(\cdot|s_h,a_h) - \hat{\PP}^k_h(\cdot|s_h,a_h) \|_1 \leq \|\PP_h(\cdot|s_h,a_h)\|_1 + \|\hat{\PP}^k_h(\cdot|s_h,a_h) \|_1 = 2$. Moreover, according to Lemma \ref{lem:expand1}, we have
\begin{align*}
&\EE \left[ \sum_{h=1}^H \|\PP_h(\cdot|s_h,a_h) - \hat{\PP}^k_h(\cdot|s_h,a_h) \|_1 \Bigggiven \pi^*, \hat{\PP}^k \right] \\
&\qquad = \sum_{h=1}^H \EE_{(s_h, a_h)\sim d_h^{\pi^*, \hat{\PP}^k}(\cdot,\cdot)} [  \|\PP_h(\cdot|s_h,a_h) - \hat{\PP}^k_h(\cdot|s_h,a_h) \|_1] \\
&\qquad= \sum_{h=2}^H \sqrt{8k  \zeta^k_{h-1} + 2k  |\cA|\xi_h^k + 4\lambda_k  d/(\CS)^2} \cdot \EE_{ d^{\pi^*, \hat{\PP}^k}_{h-1}}\left\|\hat{\phi}^k_{h-1}\right\|_{\Sigma_{\tilde{\rho}^k_{h-1}, \hat{\phi}^k_{h-1}}^{-1}}  + \sqrt{|\cA|  \zeta_1^k },
\end{align*}
where the last equation is by the below definitions for all $(h,k)\in [H]\times[K]$,
\begin{align}
\begin{aligned} \label{eq:tran-err-def}
&\zeta_h^k:=\EE_{(s,a)\sim \tilde{\rho}_h^k(\cdot,\cdot)}[\|\PP_h(\cdot|s
,a) - \hat{\PP}^k_1(\cdot|s,a) \|_h^2],\\
&\xi_h^k:=\EE_{(s,a)\sim\breve{\rho}^k_h(\cdot,\cdot)}[ \|\PP_h(\cdot|s,a) - \hat{\PP}^k_h(\cdot|s,a) \|_1^2],
\end{aligned}
\end{align}
whose upper bound will be characterized later. Therefore, the above results imply that
\begin{align*}
&\EE \left[ \sum_{h=1}^H H \|\PP_h(\cdot|s_h,a_h) - \hat{\PP}^k_h(\cdot|s_h,a_h) \|_1 \Bigggiven \pi^*, \hat{\PP}^k \right] \\
& \leq \min\bigg\{ H\sqrt{|\cA|  \zeta_1^k } + \sum_{h=2}^H H\sqrt{8k  \zeta^k_{h-1} + 2k  |\cA|\xi_h^k + 4\lambda_k  d/(\CS)^2}\cdot \EE_{d^{\pi^*, \hat{\PP}^k}_{h-1}}\left\|\hat{\phi}^k_{h-1}\right\|_{\Sigma_{\tilde{\rho}^k_{h-1}, \hat{\phi}^k_{h-1}}^{-1}}  , ~~ 2H^2\bigg\}.
\end{align*}
On the other hand, we further bound the term $\EE [ \sum_{h=1}^H -\beta^k_h(s_h,a_h)  \given \pi^*, \hat{\PP}^k ]$ in \eqref{eq:term-i-decomp}. We have 
\begin{align*}
\EE \left[ \sum_{h=1}^H -\beta^k_h(s_h,a_h)  \bigggiven \pi^*, \hat{\PP}^k \right]&  =\EE \left[ \sum_{h=1}^H - \min\{\gamma_k \|\hat{\phi}^k_h(s_h,a_h)\|_{(\hat{\Sigma}_h^k)^{-1}}, 2H\} \bigggiven \pi^*, \hat{\PP}^k \right]\\
&  \leq\EE \left[ \sum_{h=1}^H - \min\left\{\frac{3}{5}\gamma_k \|\hat{\phi}^k_h(s_h,a_h)\|_{\Sigma_{\tilde{\rho}^k_h, \hat{\phi}^k_h}^{-1}}, 2H\right\} \bigggiven \pi^*, \hat{\PP}^k \right]\\
&  = - \min\left\{\frac{3}{5}\gamma_k \sum_{h=1}^H \EE_{d^{\pi^*, \hat{\PP}^k}_h}\|\hat{\phi}^k_h\|_{\Sigma_{\tilde{\rho}^k_h, \hat{\phi}^k_h}^{-1}}, 2H^2\right\} \\
&  \leq - \min\left\{\frac{3}{5}\gamma_k \sum_{h=1}^{H-1} \EE_{d^{\pi^*, \hat{\PP}^k}_h}\|\hat{\phi}^k_h\|_{\Sigma_{\tilde{\rho}^k_h, \hat{\phi}^k_h}^{-1}}, 2H^2\right\} ,
\end{align*} 
when  $\lambda_k\geq c_0 d \log(H|\Phi|k/\delta)$ with probability at least $1-\delta$. The first inequality is obtained by applying  Lemma \ref{lem:con-inverse} for all $h\in [H]$. Thus, plugging in the above results into \eqref{eq:term-i-decomp}, for a sufficient large $c_0$, setting 
\begin{align}
\lambda_k= c_0 d \log(H|\Phi|k/\delta), \qquad
\gamma_k= \frac{5}{3}H\sqrt{8k  \zeta^k_{h-1} + 2k  |\cA|\xi_h^k + 4\lambda_k  d/(\CS)^2}, \label{eq:extra-condition}
\end{align}
we have that
\begin{align}
(i) =  V_1^{\pi^*}(s_1) - V_1^{\pi^*}(s_1)  &\leq \sqrt{|\cA|  \zeta_1^k }, \label{eq:near-optim}
\end{align}
where the inequality is due to  $\min\{x+y,2H^2\}-\min\{y, 2H^2\} \leq x, \forall x,y\geq 0$. The above inequality \eqref{eq:near-optim} looks similar to the optimism in linear MDPs \citep{jin2020provably} but has an additional positive bias $\sqrt{|\cA|  \zeta_1^k }$ which depends on $\sqrt{1/k}$. \citet{uehara2021representation} refers to such a biased optimism as \emph{near-optimism}. In our work, we prove that our algorithm for the low-rank MDP in an episodic setting also leads to near-optimism. 

Next, we show the upper bound of the term $(ii)$ in \eqref{eq:decomp-mdp-init}. By Lemma \ref{lem:diff2}, we have
\begin{align}
\begin{aligned}\label{eq:term-ii-decomp}
(ii) = \overline{V}_{k,1}^{\pi^k}(s_1) -V_1^{\pi^k}(s_1)  &= \EE \left[ \sum_{h=1}^H \left(\beta^k_h(s_h,a_h) - (\PP_h - \hat{\PP}^k_h )\overline{V}^{\pi^k}_{h+1}(s_h,a_h)\right) \Bigggiven \pi^k, \PP \right]\\
&\leq \EE \left[ \sum_{h=1}^H \left(\beta^k_h(s_h,a_h) + 3H^2\|\PP_h(\cdot|s_h,a_h) - \hat{\PP}^k_h(\cdot|s_h,a_h) \|_1\right) \Bigggiven \pi^k, \PP \right]\\
&= \sum_{h=1}^H \EE_{(s,a)\sim d_h^{\pi^k, \PP}(\cdot,\cdot)}  \left(\beta^k_h(s,a) + 3H^2\|\PP_h(\cdot|s,a) - \hat{\PP}^k_h(\cdot|s,a) \|_1\right) .    
\end{aligned}
\end{align}
where the first inequality is due to $\sup_{s\in\cS,a\in\cA}(r_h+\beta_h^k)(s,a)\leq 1 + 2H\leq 3H$ such that $\sup_{s\in\cS}\overline{V}^\pi_h(s) \leq 3H^2, \forall h\in [H]$ and the last equation is by the definition of $d_h^{\pi^k, \PP}$. Then, we need to separately bound the two terms in the last equation above. By Lemma \ref{lem:expand2}, since $\sup_{s\in\cS, a\in\cA}\beta_h^k(s,a)\leq 2H$ according to the definition of $\beta_h^k$ in Algorithm \ref{alg:contrastive}, we have
\begin{align*}
&\sum_{h=1}^H \EE_{(s,a)\sim d_h^{\pi^k, \PP}(\cdot,\cdot)} [ \beta^k_h(s,a) ]\\
&\leq   \sqrt{|\cA|  \EE_{a\sim \tilde{\rho}_1^k(s_1,\cdot)}[\beta^k_1(s_1,a)^2] } + \sum_{h=2}^H \sqrt{k  |\cA|  \EE_{(s,a)\sim\tilde{\rho}^k_h(\cdot,\cdot)}[ \beta^k_h(s,a)^2] + 4H^2\lambda_k  d} ~ \EE_{ d^{\pi^k, \PP}_{h-1}}\left\|\phi^*_{h-1}\right\|_{\Sigma_{\rho^k_{h-1}, \phi^*_{h-1}}^{-1}}\\
&\leq   \sqrt{|\cA|\gamma_k^2  \EE_{a\sim \tilde{\rho}_1^k(s_1,\cdot)} \|\hat{\phi}^k_1(s_1,a)\|_{(\hat{\Sigma}_1^k)^{-1}}^2 } \\
&\quad + \sum_{h=2}^H \sqrt{k  |\cA| \gamma_k^2 \EE_{\tilde{\rho}^k_h} \|\hat{\phi}^k_h\|_{(\hat{\Sigma}_h^k)^{-1}}^2 + 4H^2\lambda_k  d} ~ \EE_{ d^{\pi^k, \PP}_{h-1}}\left\|\phi^*_{h-1}\right\|_{\Sigma_{\rho^k_{h-1}, \phi^*_{h-1}}^{-1}}
,
\end{align*}
where the second inequality is due to $\beta_h^k(s,a) \leq \|\hat{\phi}^k_h(s,a)\|_{(\hat{\Sigma}_h^k)^{-1}}$. Furthermore, we have that with $\lambda_k \geq c_0 d\log(H|\Phi|k/\delta)$, with probability at least $1-\delta$, for all $h\in [H]$, 
\begin{align*}
\EE_{(s,a)\sim\tilde{\rho}^k_h(\cdot,\cdot)} \|\hat{\phi}^k_h(s,a)\|_{(\hat{\Sigma}_h^k)^{-1}}^2 &\leq 3\EE_{(s,a)\sim\tilde{\rho}^k_h(\cdot,\cdot)} \|\hat{\phi}^k_h(s,a)\|_{\Sigma_{\tilde{\rho}^k_h, \hat{\phi}^k_h}^{-1}}^2\\
&=3\EE_{\tilde{\rho}^k_h}\left[ \hat{\phi}^k_h{}^\top \left(k\EE_{\tilde{\rho}^k_h}[ \hat{\phi}^k_h (\hat{\phi}^k_h)^\top] + \lambda_k I \right)^{-1}\hat{\phi}^k_h \right]\\
&=\frac{3}{k}\tr\left\{ k\EE_{\tilde{\rho}^k_h}[ \hat{\phi}^k_h\hat{\phi}^k_h{}^\top] \left(k\EE_{\tilde{\rho}^k_h}[ \hat{\phi}^k_h (\hat{\phi}^k_h)^\top] + \lambda_k I \right)^{-1}  \right\}\\
&\leq \frac{3}{k}\tr(I) = \frac{3d}{k},
\end{align*}
where the first inequality is by Lemma \ref{lem:con-inverse} and $\EE_{\tilde{\rho}^k_h}$ is short for $\EE_{(s,a)\sim\tilde{\rho}^k_h(\cdot,\cdot)}$. Thus, combining the above results, we have the following inequality holds with probability at least $1-\delta$,
\begin{align}
\begin{aligned}\label{eq:term-ii-decomp-1}
&\sum_{h=1}^H \EE_{(s,a)\sim d_h^{\pi^k, \PP}(\cdot,\cdot)} [ \beta^k_h(s,a) ]\\
&\qquad \leq   \sqrt{3d|\cA|\gamma_k^2  / k } + \sum_{h=2}^H \sqrt{3d  |\cA| \gamma_k^2  + 4H^2\lambda_k  d} ~ \EE_{d^{\pi^k, \PP}_{h-1}}\left\|\phi^*_{h-1}\right\|_{\Sigma_{\rho^k_{h-1}, \phi^*_{h-1}}^{-1}}.
\end{aligned}
\end{align}
In addition, further by Lemma \ref{lem:expand2}, due to $\|\PP_h(\cdot|s,a) - \hat{\PP}^k_h(\cdot|s,a) \|_1 \leq \|\PP_h(\cdot|s,a)\|_1 + \|\hat{\PP}^k_h(\cdot|s,a) \|_1\leq 2$, we have
\begin{align}
\begin{aligned}\label{eq:term-ii-decomp-2}
&\sum_{h=1}^H \EE_{(s,a)\sim d_h^{\pi^k, \PP}(\cdot,\cdot)} [ \|\PP_h(\cdot|s,a) - \hat{\PP}^k_h(\cdot|s,a) \|_1 ]\\
&\qquad =   \sqrt{|\cA|  \zeta_1^k }  + \sum_{h=2}^H \sqrt{k  |\cA|  \zeta_h^k + 4\lambda_k  d} ~ \EE_{d^{\pi^k, \PP}_{h-1}}\left\|\phi^*_{h-1}\right\|_{\Sigma_{\rho^k_{h-1}, \phi^*_{h-1}}^{-1}}.
\end{aligned}
\end{align}
Therefore, combining \eqref{eq:term-ii-decomp}, \eqref{eq:term-ii-decomp-1}, and \eqref{eq:term-ii-decomp-2}, we obtain 
\begin{align}
(ii) &\leq  \left[\sqrt{3d|\cA|\gamma_k^2  / k } +3H^2\sqrt{|\cA|  \zeta_1^k }\right] \label{eq:term-ii-bound} \\
&\quad + \sum_{h=2}^H \left[\sqrt{3d  |\cA| \gamma_k^2  + 4H^2\lambda_k  d}+3H^2\sqrt{k  |\cA|  \zeta_h^k + 4\lambda_k  d}\right] \EE_{d^{\pi^k, \PP}_{h-1}}\left\|\phi^*_{h-1}\right\|_{\Sigma_{\rho^k_{h-1}, \phi^*_{h-1}}^{-1}}. \nonumber
\end{align}
Now we characterize the upper bound of $\zeta_h^k$ and $\xi_h^k$ as defined in \eqref{eq:tran-err-def}.  According to Lemma \ref{lem:stat-err}, we have with probability at least $1-2\delta$,
\begin{align}
\begin{aligned}\label{eq:stat-err-bound}
&\zeta_h^k \leq  32d\log (2kH|\cF|/\delta)/k, \quad \forall h\geq 1,\\
& \xi_h^k  \leq 32d\log (2kH|\cF|/\delta)/k, \quad \forall h\geq 2,
\end{aligned}
\end{align}
Plugging \eqref{eq:stat-err-bound} and \eqref{eq:extra-condition} into \eqref{eq:near-optim} and \eqref{eq:term-ii-bound}, we obtain
\begin{align*}
&(i) =  V_1^{\pi^*}(s_1) - V_1^{\pi^*}(s_1) \lesssim \sqrt{d|\cA| \log (KH|\cF|/\delta)/k }, \\
&(ii) =  \overline{V}_{k,1}^{\pi^k}(s_1) -V_1^{\pi^k}(s_1)\lesssim \sqrt{C_1 \log(H|\cF|K/\delta)/k}  +  \sqrt{(C_1 +C_2)\log(H|\cF|K/\delta)}\sum_{h=1}^{H-1} \EE_{ d^{\pi^k, \PP}_h}\left\|\phi^*_h\right\|_{\Sigma_{\rho^k_h, \phi^*_h}^{-1}}.
\end{align*}
where we let $C_1 = H^2d^3|\cA|/(\CS)^2 +  H^2d^2|\cA|^2/(\CS)^2 + H^4d|\cA|/(\CS)^2 $ and $C_2 = H^4d^2$. Further by \eqref{eq:decomp-mdp-init}, we have
\begin{align*}
\frac{1}{K}\sum_{k=1}^K \left[ V_1^{\pi^*}(s_1)- V_1^{\pi^k}(s_1)\right]&\lesssim    \sqrt{(C_1 +C_2)\log(H|\cF|K/\delta)}/K\cdot \sum_{h=1}^{H-1} \sum_{k=1}^K\EE_{d^{\pi^k, \PP}_h}\left\|\phi^*_h\right\|_{\Sigma_{\rho^k_h, \phi^*_h}^{-1}}\\
&\quad +\sqrt{C_1 \log(H|\cF|K/\delta)/K}.
\end{align*}
Moreover, we have
\begin{align*}
&\frac{1}{K}\sum_{k=1}^K\EE_{(s,a)\sim d^{\pi^k, \PP}_h(\cdot,\cdot)}\left\|\phi^*_h(s,a)\right\|_{\Sigma_{\rho^k_h, \phi^*_h}^{-1}} \\
&\qquad\leq \sqrt{\frac{1}{K}\sum_{k=1}^K\EE_{(s,a)\sim d^{\pi^k, \PP}_h(\cdot,\cdot)}\left\|\phi^*_h(s,a)\right\|^2_{\Sigma_{\rho^k_h, \phi^*_h}^{-1}}}\\
&\qquad= \sqrt{\frac{1}{K}\sum_{k=1}^K\tr\left(\EE_{(s,a)\sim d^{\pi^k, \PP}_h(\cdot,\cdot)}\left(\phi^*_h(s,a)\phi^*_h(s,a)^\top\right)\Sigma_{\rho^k_h, \phi^*_h}^{-1}\right) }\\
&\qquad\leq \sqrt{d \log(1+kd/\lambda_k)/K} \leq \sqrt{d \log(1+c_1 K)/K}.
\end{align*}
where the first inequality is by Jensen's inequality and the second inequality is by Lemma \ref{lem:logdet-tele} with $c_1$ being some absolute constant. Thus, we have
\begin{align*}
&\frac{1}{K}\sum_{k=1}^K \left[ V_1^{\pi^*}(s_1)- V_1^{\pi^k}(s_1)\right]\\
&\qquad \lesssim    \sqrt{(C_1 +C_2)\log(H|\cF|K/\delta) H^2d \log(1+c_1 K)/K }+\sqrt{C_1 \log(H|\cF|K/\delta)/K} \\
&\qquad \lesssim \sqrt{H^2d(C_1 +C_2)\log(H|\cF|K/\delta)  \log(1+c_1 K)/K }.
\end{align*}
Taking union bound for all events in this proof, due to $|\cF|\geq |\Phi|$, setting 
\begin{align*}
\lambda_k= c_0 d \log(H|\cF|k/\delta), \quad
\gamma_k=  4H \big( 12\sqrt{  |\cA|d} + \sqrt{c_0} d\big)/\CS\cdot \sqrt{\log (2Hk|\cF|/\delta)},
\end{align*}
we obtain with probability at least $1-3\delta$,
\begin{align*}
\frac{1}{K}\sum_{k=1}^K \left[ V_1^{\pi^*}(s_1)- V_1^{\pi^k}(s_1)\right]\lesssim \sqrt{C\log(H|\cF|K/\delta) \log(c_0' K)/K },
\end{align*}
where $C = H^4d^4|\cA|/(\CS)^2 +  H^4d^3|\cA|^2/(\CS)^2 + H^6d^2|\cA|/(\CS)^2 + H^6d^3$ and $c_0,c_0'$ are absolute constants.
This completes the proof.
\end{proof}

\section{Theoretical Analysis for Markov Game}

\subsection{Lemmas}

\begin{lemma}[Learning Target of Contrastive Loss]\label{lem:opt-contra-loss-mg} For any $(s,a,b)\in \cS\times\cA\times\cB$ that is reachable under certain sampling strategy, the learning target of the contrastive loss in \eqref{eq:contra-loss} with setting $z=(s,a,b)$ is
\begin{align*}
    f_h^*(s,a,b,s') = \frac{\PP_h(s'|s,a,b)}{\cPS(s')}.
\end{align*}

\end{lemma}

\begin{proof}
For any $h\in [H]$, we let $\Pr{}_h$ to denote the probability for some event at the $h$-th step of a Markov game. The contrastive loss in \eqref{eq:contra-loss} with setting $z=(s,a,b)$ implicitly assumes
\begin{align*}
    \Pr{}_h(y|s,a,b,s') = \left(\frac{f^*_h(s,a,b,s')}{1+f^*_h(s,a,b,s')}\right)^y\left(\frac{1}{1+f^*_h(s,a,b,s')}\right)^{1-y}.
\end{align*}
In addition, by Bayes' rule, we also have
\begin{align*}
\Pr{}_h(y|s,a,b,s') = \frac{\Pr{}_h(s,a,b,s'|y)\Pr{}_h(y)}{\sum_{y\in \{0,1\}}\Pr{}_h(s,a,b,s'|y)\Pr{}_h(y)} = \frac{\Pr{}_h(s,a,b,s'|y)}{\Pr{}_h(s,a,b)\PP_h(s'|s,a,b) +\Pr{}_h(s,a,b)\cPS(s')},
\end{align*}
where we use $\Pr{}_h(y) = 1/2$ for any $y\in\{0, 1\}$ according to Algorithm \ref{alg:sample-mg}. In the last equality, we also have
\begin{align*}
&\Pr{}_h(s,a,b,s'|y=1) = \Pr{}_h(s,a,b|y=1) \Pr{}_h(s'|y=1,s,a,b)= \Pr{}_h(s,a,b) \PP_h(s'|s,a,b),\\
&\Pr{}_h(s,a,b,s'|y=0) = \Pr{}_h(s,a,b|y=0) \Pr{}_h(s'|y=0,s,a,b)= \Pr{}_h(s,a,b) \cPS(s'),
\end{align*} 
where we use $\Pr{}_h(s,a,b|y=1) = \Pr{}_h(s,a,b|y=0) =\Pr{}_h(s,a,b)$ and also $\Pr{}_h(s'|y=1,s,a,b)= \PP_h(s'|s,a,b)$, $\Pr{}_h(s'|y=0,s,a,b)= \cPS(s')$.  

Combining the above results, when $y=1$ at the $h$-th step,  we obtain
\begin{align*}
    \frac{f_h^*(s,a,b,s')}{1+f_h^*(s,a,b,s')} = \frac{ \Pr{}_h(s,a,b)\PP_h(s'|s,a,b)}{\Pr{}_h(s,a,b)\PP_h(s'|s,a,b) +\Pr{}_h(s,a,b)\cPS(s')},
\end{align*}
which gives
\begin{align*}
    f_h^*(s,a,b,s') = \frac{\PP_h(s'|s,a,b)}{\cPS(s')}.
\end{align*}
Equivalently, when $y=0$, we get the same result. This completes the proof.
\end{proof}

\begin{lemma}
\label{lem:diff1-mg} Suppose the policies $\pi^k$, $\nu^k$, the estimated transition $\hat{\PP}^k$, and the bonus $\beta^k$ are obtained at episode $k$ of Algorithm \ref{alg:contrastive-mg}. Let $\mathrm{br}(\cdot)$ denote the best response policy given the opponent's policy. Moreover, $\underline{V}_{k,1}^\sigma(s_1)$ denotes the value function under any joint policy $\sigma$ for the zero-sum Markov game defined by the reward function $r-\beta^k$ and $\hat{\PP}^k$ while $\overline{V}_{k,1}^\sigma(s_1)$ denotes the value function for the zero-sum Markov game defined by $r+\beta^k$ and $\hat{\PP}^k$. Then, we have the following value function differences decomposed as
\begin{small}
\begin{align*}
&V_1^{\mathrm{br}(\nu^k), \nu^k}(s_1) - \overline{V}_{k,1}^{\mathrm{br}(\nu^k), \nu^k}(s_1)   = \EE \left[ \sum_{h=1}^H \left(-\beta^k_h(s_h,a_h,b_h) + (\PP_h - \hat{\PP}^k_h )V^{\mathrm{br}(\nu^k), \nu^k}_{h+1}(s_h,a_h,b_h)\right) \Bigggiven\mathrm{br}(\nu^k), \nu^k, \hat{\PP}^k \right],\\    
&\underline{V}_{k,1}^{\pi^k, \mathrm{br}(\pi^k)}(s_1)-V_1^{\pi^k, \mathrm{br}(\pi^k)}(s_1)    = \EE \left[ \sum_{h=1}^H \left(-\beta^k_h(s_h,a_h,b_h) - (\PP_h - \hat{\PP}^k_h )V^{\pi^k, \mathrm{br}(\pi^k)}_{h+1}(s_h,a_h,b_h)\right) \Bigggiven \pi^k, \mathrm{br}(\pi^k), \hat{\PP}^k \right].
\end{align*}
\end{small}
\end{lemma}

\begin{proof}
Consider two zero-sum Markov games defined by $(\cS,\cA, \cB, H, r, \PP)$ and $(\cS,\cA, \cB, H, r+\beta, \PP')$ where  $\PP$ and $\PP'$ are any transition models and $r$ and $\beta$ are arbitrary reward function and bonus term. Then, for any joint policy $\sigma$, we let $Q^\sigma_h$ and $V^\sigma_h$ be the associated Q-function and value function at the $h$-th step for the Markov game defined by $(\cS,\cA, \cB, H, r, \PP)$, and  $\tilde{Q}^\sigma_h$ and $\tilde{V}^\sigma_h$ be the associated Q-function and value function at the $h$-th step for the Markov game defined by $(\cS,\cA, \cB, H, r+\beta, \PP')$. Then, by Bellman equation, we have for any $(s_h, a_h, b_h)\in \cS\times\cA\times \cB$,
\begin{align*}
    & Q_h^\sigma(s_h,a_h,b_h) - \tilde{Q}_h^\sigma(s_h,a_h,b_h)\\
    &\qquad =  -\beta_h(s_h,a_h,b_h) + \PP_h V^{\sigma}_{h+1}(s_h,a_h,b_h)- \PP'_h\tilde{V}^{\sigma}_{h+1}(s_h,a_h,b_h) \\
    &\qquad = -\beta_h(s_h,a_h,b_h) + \PP_h V^{\sigma}_{h+1}(s_h,a_h,b_h)- \PP'_h\tilde{V}^{\sigma}_{h+1}(s_h,a_h,b_h) \\
    &\qquad = -\beta_h(s_h,a_h,b_h) + (\PP_h - \PP'_h )V^{\sigma}_{h+1}(s_h,a_h,b_h) + \PP'_h[V^{\sigma}_{h+1}(s_h,a_h,b_h) - \tilde{V}^{\sigma}_{h+1}(s_h,a_h,b_h)].
\end{align*}
Further by the Bellman equation and the above result, we have
\begin{align*}
&V_h^\sigma(s_h) - \tilde{V}_h^\sigma(s_h) \\
&\qquad = \langle \sigma_h(\cdot,\cdot|s_h), Q_h^\sigma(s_h,\cdot,\cdot) - \tilde{Q}_h^\sigma(s_h,\cdot,\cdot)\rangle    \\
&\qquad = \langle \sigma_h(\cdot,\cdot|s_h), -\beta_h(s_h,\cdot,\cdot) + (\PP_h - \PP'_h )V^{\sigma}_{h+1}(s_h,\cdot,\cdot) + \PP'_h[V^{\sigma}_{h+1}(s_h,\cdot,\cdot) - \tilde{V}^{\sigma}_{h+1}(s_h,\cdot,\cdot)]\rangle. 
\end{align*}
Since $V_{H+1}^\sigma(s) = \tilde{V}_{H+1}^\sigma(s) = 0$ for any $s\in\cS$ and $\sigma$, recursively applying the above relation, we have
\begin{align*}
&V_1^\sigma(s_1) - \tilde{V}_1^\sigma(s_1)  = \EE \left[ \sum_{h=1}^H \left(-\beta_h(s_h,a_h,b_h) + (\PP_h - \PP'_h )V^{\sigma}_{h+1}(s_h,a_h,b_h)\right) \Bigggiven \sigma, \PP' \right].     
\end{align*}
For any episode $k$, setting $\PP', \sigma, \beta$ to be $\hat{\PP}^k,(\mathrm{br}(\nu^k), \nu^k),\beta^k$ defined in Algorithm \ref{alg:contrastive-mg} and $\PP, r$ to be the true transition model and reward function, by the above equality, according to the definition of $V_h^\sigma$ and $\overline{V}_{k,h}^\sigma$, we obtain
\begin{align*}
&V_1^{\mathrm{br}(\nu^k), \nu^k}(s_1) - \overline{V}_{k,1}^{\mathrm{br}(\nu^k), \nu^k}(s_1)   \\
&\qquad = \EE \left[ \sum_{h=1}^H \left(-\beta^k_h(s_h,a_h,b_h) + (\PP_h - \hat{\PP}^k_h )V^{\mathrm{br}(\nu^k), \nu^k}_{h+1}(s_h,a_h,b_h)\right) \Bigggiven\mathrm{br}(\nu^k), \nu^k, \hat{\PP}^k \right].    
\end{align*}
Moreover, setting $\PP', \sigma, \beta$ to be $\hat{\PP}^k,(\pi^k, \mathrm{br}(\pi^k)),-\beta^k$ defined in Algorithm \ref{alg:contrastive-mg} and $\PP, r$ to be the true transition model and reward function, by the definition of $V_h^\sigma$ and $\underline{V}_h^\sigma$, we obtain
\begin{align*}
& V_1^{\pi^k, \mathrm{br}(\pi^k)}(s_1)- \underline{V}_{k,1}^{\pi^k, \mathrm{br}(\pi^k)}(s_1)    \\
&\qquad = \EE \left[ \sum_{h=1}^H \left(\beta^k_h(s_h,a_h,b_h) + (\PP_h - \hat{\PP}^k_h )V^{\pi^k, \mathrm{br}(\pi^k)}_{h+1}(s_h,a_h,b_h)\right) \Bigggiven \pi^k, \mathrm{br}(\pi^k), \hat{\PP}^k \right],
\end{align*}
which leads to
\begin{align*}
&\underline{V}_{k,1}^{\pi^k, \mathrm{br}(\pi^k)}(s_1)-V_1^{\pi^k, \mathrm{br}(\pi^k)}(s_1) \\
&\qquad = \EE \left[ \sum_{h=1}^H \left(-\beta^k_h(s_h,a_h,b_h) - (\PP_h - \hat{\PP}^k_h )V^{\pi^k, \mathrm{br}(\pi^k)}_{h+1}(s_h,a_h,b_h)\right) \Bigggiven \pi^k, \mathrm{br}(\pi^k), \hat{\PP}^k \right].
\end{align*}
This completes the proof.
\end{proof}

\begin{lemma}
\label{lem:diff2-mg} Suppose the joint policy $\sigma^k$, the estimated transition $\hat{\PP}^k$, and the bonus $\beta^k$ are obtained at episode $k$ of Algorithm \ref{alg:contrastive-mg}. Moreover, $\overline{V}_1^k(s_1)$ and $\underline{V}_1^k(s_1)$ are the estimated value functions based on UCB and LCB obtained at episode $k$ of Algorithm \ref{alg:contrastive-mg}. Then, their difference can be decomposed as
\begin{align*}
&\overline{V}_1^k(s_1) - \underline{V}_1^k(s_1)  =   \EE \left[ \sum_{h=1}^H 2 \beta^k_h(s_h,a_h,b_h) + (\hat{\PP}_h^k - \PP_h)  \big(\overline{V}^k_{h+1} - \underline{V}^k_{h+1} \big) (s_h,a_h,b_h)\Bigggiven \sigma^k, \PP \right].
\end{align*}
\end{lemma}

\begin{proof} For the episode $k$, we consider two Markov games defined by $(\cS,\cA, \cB, H, r+\beta^k, \hat{\PP}^k)$ and $(\cS,\cA, \cB, H, r-\beta^k, \hat{\PP}^k)$. Then, for the joint policy $\sigma^k$, by Algorithm \ref{alg:contrastive-mg}, we have for any $(s_h, a_h, b_h)\in \cS\times\cA\times \cB$,
\begin{align*}
    & \overline{Q}_h^k(s_h,a_h,b_h) - \underline{Q}_h^k(s_h,a_h,b_h) \\
    &\qquad=   2\beta^k_h(s_h,a_h,b_h) + \hat{\PP}_h^k \big(\overline{V}^k_{h+1}- \underline{V}^k_{h+1} \big)(s_h,a_h,b_h) \\
    &\qquad=   2\beta^k_h(s_h,a_h,b_h) + (\hat{\PP}_h^k - \PP_h)  \big(\overline{V}^k_{h+1} - \underline{V}^k_{h+1} \big) (s_h,a_h,b_h) + \PP_h \big(\overline{V}^k_{h+1}- \underline{V}^k_{h+1} \big)(s_h,a_h,b_h). 
\end{align*}
Then, we have
\begin{align*}
\overline{V}_h^k(s_h) - \underline{V}_h^k(s_h)&= \langle \sigma_h^k(\cdot,\cdot|s_h), \overline{Q}_h^k(s_h,\cdot,\cdot) - \underline{Q}_h^k(s_h,\cdot,\cdot)\rangle    \\
&= 2\left\langle \sigma_h^k(\cdot,\cdot|s_h), \beta^k_h(s_h,a_h,b_h)  \right\rangle + \Big\langle \sigma_h^k(\cdot,\cdot|s_h), (\hat{\PP}_h^k - \PP_h)  \big(\overline{V}^k_{h+1} - \underline{V}^k_{h+1} \big) (s_h,\cdot,\cdot) \Big\rangle \\
& \quad + \Big\langle \sigma_h^k(\cdot,\cdot|s_h), \PP_h \big(\overline{V}^k_{h+1}- \underline{V}^k_{h+1} \big)(s_h,\cdot,\cdot)	\Big\rangle.
\end{align*}
By the fact that $\overline{V}_{H+1}^k(s) = \underline{V}_{H+1}^k(s) = 0$ for any $s\in\cS$, recursively applying the above relation, we have
\begin{align*}
&\overline{V}_1^k(s_1) - \underline{V}_1^k(s_1)  =   \EE \left[ \sum_{h=1}^H 2 \beta^k_h(s_h,a_h,b_h) + (\hat{\PP}_h^k - \PP_h)  \big(\overline{V}^k_{h+1} - \underline{V}^k_{h+1} \big) (s_h,a_h,b_h)\Bigggiven \sigma^k, \PP \right].
\end{align*}
This completes the proof.
\end{proof}


\begin{lemma}\label{lem:expand1-mg}
Suppose that $\hat{\PP}^k$ is the estimated transition obtained at episode $k$ of Algorithm \ref{alg:contrastive-mg}. We define $\zeta^k_{h-1}:=\EE_{(s'',a'',b'')\sim\tilde{\rho}^k_{h-1}(\cdot,\cdot,\cdot)} \|\hat{\PP}^k_{h-1}(\cdot|s'',a'',b'') -\PP_{h-1}(\cdot|s'',a'',b'')\|_1^2$ for all $h\geq 2$, $\tilde{\rho}^k_h(\cdot,\cdot,\cdot) := \frac{1}{k}\sum_{k'=0}^{k-1} \tilde{d}^{\sigma^{k'}}_h(\cdot,\cdot,\cdot)$ for all $h\geq 1$ with $\tilde{\rho}_1^k(s_1,a,b) =\Unif(a)\Unif(b)$, and $\breve{\rho}^k_h(\cdot,\cdot,\cdot) := \frac{1}{k}\sum_{k'=0}^{k-1} \breve{d}^{\sigma^{k'}}_h(\cdot,\cdot,\cdot)$ for all $h\geq 2$.  Then for any function $g:\cS\times\cA\times\cB\mapsto [0,B]$ and joint policy $\sigma$, we have for any $h \geq 2$, the following inequality holds 
\begin{align*}
&\left|\EE_{(s,a,b)\sim d^{\sigma, \hat{\PP}^k}_h(\cdot,\cdot,\cdot)}[g(s,a,b)]\right| \\
&\qquad \leq \sqrt{2k B^2 \zeta^k_{h-1} + 2k  |\cA||\cB|\cdot  \EE_{(s,a,b)\sim\breve{\rho}^k_h(\cdot,\cdot,\cdot)}[ g(s,a,b)^2] + \lambda_k B^2 d/(\CS)^2} \cdot \EE_{d^{\sigma, \hat{\PP}^k}_{h-1}}\left\|\hat{\phi}^k_{h-1}\right\|_{\Sigma_{\tilde{\rho}^k_{h-1}, \hat{\phi}^k_{h-1}}^{-1}}.
\end{align*}
In addition, for $h=1$, we have
\begin{align*}
\left|\EE_{(s,a,b)\sim d^{\sigma, \PP}_1(\cdot,\cdot,\cdot)}[g(s,a,b)]\right| &= \sqrt{\EE_{(a,b)\sim \sigma_1(\cdot,\cdot|s_1)}[g(s_1,a,b)^2]} \leq \sqrt{|\cA||\cB|  \EE_{(a,b)\sim \tilde{\rho}_1^k(s_1,\cdot,\cdot)}[g(s_1,a,b)^2] }.
\end{align*} 
\end{lemma}

\begin{proof}
For any function $g:\cS\times\cA\times\cB\mapsto [0, B]$ and any joint policy $\sigma$, under the estimated transition model $\hat{\PP}^k$ at the $k$-th episode, for any $h\geq 2$, we have
\begin{align}
&\left|\EE_{(s,a,b)\sim d^{\sigma, \hat{\PP}^k}_h(\cdot,\cdot,\cdot)}[g(s,a,b)]\right|\nonumber \\
&\quad= \left|\EE_{(s',a',b')\sim d^{\sigma, \hat{\PP}^k}_{h-1}(\cdot,\cdot,\cdot), s\sim \hat{\PP}^k_{h-1}(\cdot|s',a',b'), (a,b)\sim \sigma_h(\cdot,\cdot|s)}[g(s,a,b)]\right| \nonumber\\
&\quad= \left|\EE_{(s',a',b')\sim d^{\sigma, \hat{\PP}^k}_{h-1}(\cdot,\cdot,\cdot)} \left[\hat{\phi}^k_{h-1}(s',a',b')^\top \int_{\cS} \hat{\psi}^k_{h-1}(s) \sum_{a\in \cA, b\in \cB}\sigma_h(a,b|s) g(s,a,b)\mathrm{d} s\right]\right| \nonumber\\
&\quad\leq  \EE_{d^{\sigma, \hat{\PP}^k}_{h-1}}\left\|\hat{\phi}^k_{h-1}\right\|_{\Sigma_{\tilde{\rho}^k_{h-1}, \hat{\phi}^k_{h-1}}^{-1}} \cdot  \left\| \int_{\cS} \hat{\psi}^k_{h-1}(s) \sum_{a\in \cA, b\in \cB}\sigma_h(a,b|s) g(s,a,b)\mathrm{d} s\right\|_{\Sigma_{\tilde{\rho}^k_{h-1}, \hat{\phi}^k_{h-1}}}, \label{eq:step-back-mg1}
\end{align}
where the inequality is due to the Cauchy-Schwarz inequality. We define the covariance matrix as $\Sigma_{\tilde{\rho}^k_{h-1}, \hat{\phi}^k_{h-1}} :=k \EE_{(s,a,b)\sim\tilde{\rho}^k_{h-1}(\cdot,\cdot,\cdot)}[\hat{\phi}^k_{h-1}(s,a,b)\hat{\phi}^k_{h-1}(s,a,b)^\top] + \lambda_k I$ with $\tilde{\rho}^k_{h-1}(s,a,b) = \frac{1}{k}\sum_{k'=0}^{k-1} \tilde{d}^{\sigma^{k'}}_{h-1}(s,a,b)$.
Moreover, we have
\begin{footnotesize}
\begin{align}
&\left\| \int_{\cS} \hat{\psi}^k_{h-1}(s) \sum_{a\in \cA,b\in \cB}\sigma_h(a,b|s) g(s,a,b)\mathrm{d} s\right\|_{\Sigma_{\tilde{\rho}^k_{h-1},\hat{\phi}^k_{h-1}}}^2 \nonumber\\
& = k\left(\int_{\cS} \hat{\psi}^k_{h-1}(s) \sum_{a\in \cA,b\in \cB}\sigma_h(a,b|s) g(s,a,b)\mathrm{d}s\right)^\top  \EE_{\tilde{\rho}^k_{h-1}}\left[\hat{\phi}^k_{h-1}(\hat{\phi}^k_{h-1})^\top \right]  \left(\int_{\cS} \hat{\psi}^k_{h-1}(s) \sum_{a\in \cA,b\in \cB}\sigma_h(a,b|s) g(s,a,b)\mathrm{d}s\right)\nonumber\\
& \quad + \lambda_k\left(\int_{\cS} \hat{\psi}^k_{h-1}(s) \sum_{a\in \cA,b\in \cB}\sigma_h(a,b|s) g(s,a,b)\mathrm{d}s\right)^\top \left(\int_{\cS} \hat{\psi}^k_{h-1}(s) \sum_{a\in \cA,b\in \cB}\sigma_h(a,b|s) g(s,a,b)\mathrm{d}s\right)\nonumber\\
&= k \EE_{(s'',a'',b'')\sim\tilde{\rho}^k_{h-1}(\cdot,\cdot,\cdot)}\left[\int_{\cS} \hat{\phi}^k_{h-1}(s'',a'',b'')^\top \hat{\psi}^k_{h-1}(s) \sum_{a\in \cA,b\in \cB}\sigma_h(a,b|s) g(s,a,b)\mathrm{d}s  \right] \nonumber\\
&\quad + \lambda_k\left(\int_{\cS} \hat{\psi}^k_{h-1}(s) \sum_{a\in \cA,b\in \cB}\sigma_h(a,b|s) g(s,a,b)\mathrm{d}s\right)^\top \left(\int_{\cS} \hat{\psi}^k_{h-1}(s) \sum_{a\in \cA,b\in \cB}\sigma_h(a,b|s) g(s,a,b)\mathrm{d}s\right)\nonumber\\
&\leq k \EE_{(s'',a'',b'')\sim\tilde{\rho}^k_{h-1}(\cdot,\cdot,\cdot)}\left[\int_{\cS} \hat{\phi}^k_{h-1}(s'',a'',b'')^\top \hat{\psi}^k_{h-1}(s) \sum_{a\in \cA,b\in \cB}\sigma_h(a,b|s) g(s,a,b)\mathrm{d}s  \right]^2  + \lambda_k B^2 d/(\CS)^2, \label{eq:step-back-mg2}
\end{align}
\end{footnotesize}
where the last inequality is by
\begin{align*}
\bigg(\int_{\cS} \hat{\psi}^k_{h-1}(s) \sum_{a\in \cA,b\in \cB}\sigma_h(a,b|s) g(s,a,b)\mathrm{d}s\bigg)^\top \bigg(\int_{\cS} \hat{\psi}^k_{h-1}(s) \sum_{a\in \cA,b\in \cB}\sigma_h(a,b|s) g(s,a,b)\mathrm{d}s\bigg)\leq B^2 d/(\CS)^2,
\end{align*}
since $0\leq g(s,a,b)\leq B$ and $\|\int_{\cS} \hat{\psi}_{h-1}^k(s) \mathrm{d}s\|_2^2 :=\|\int_{\cS}  \cPS(s)  \tilde{\psi}_{h-1}^k (s)\mathrm{d}s\|_2^2 \leq \|\int_{\cS}  \tilde{\psi}_{h-1}^k (s)\mathrm{d}s\|_2^2 \leq (\int_{\cS} \|  \tilde{\psi}_{h-1}^k (s)\|_2\mathrm{d}s)^2 \leq d/(\CS)^2$ according to the definition of the function class in Definition \ref{def:func-class} and the assumption that all states are normalized such that $\mathrm{Vol}(\cS)\leq 1$. In addition, we have
\begin{small}
\begin{align}
&k \EE_{(s'',a'',b'')\sim\tilde{\rho}^k_{h-1}(\cdot,\cdot,\cdot)}\bigg[\int_{\cS} \hat{\phi}^k_{h-1}(s'',a'',b'')^\top \hat{\psi}^k_{h-1}(s) \sum_{a\in \cA,b\in \cB}\sigma_h(a,b|s) g(s,a,b)\mathrm{d}s  \bigg]^2 \nonumber\\
&\quad \leq  2k \EE_{(s'',a'',b'')\sim\tilde{\rho}^k_{h-1}(\cdot,\cdot,\cdot)}\bigg[\int_{\cS} \left(\hat{\PP}^k_{h-1}(s|s'',a'',b'') -\PP_{h-1}(s|s'',a'',b'') \right) \sum_{a\in \cA,b\in \cB}\sigma_h(a,b|s) g(s,a,b)\mathrm{d}s  \bigg]^2 \nonumber\\
&\quad \quad + 2k \EE_{(s'',a'',b'')\sim\tilde{\rho}^k_{h-1}(\cdot,\cdot,\cdot)}\bigg[\int_{\cS} \PP_{h-1}(s|s'',a'',b'') \sum_{a\in \cA,b\in \cB}\sigma_h(a,b|s) g(s,a,b)\mathrm{d}s  \bigg]^2\nonumber\\
&\quad \leq  2k B^2 \zeta^k_{h-1}  + 2k \EE_{(s'',a'',b'')\sim\tilde{\rho}^k_{h-1}(\cdot,\cdot,\cdot)}\bigg[\int_{\cS} \PP_{h-1}(s|s'',a'',b'') \sum_{a\in \cA,b\in \cB}\sigma_h(a,b|s) g(s,a,b)\mathrm{d}s  \bigg]^2\nonumber\\
&\quad \leq  2k B^2 \zeta^k_{h-1}  + 2k \EE_{(s'',a'',b'')\sim\tilde{\rho}^k_{h-1}(\cdot,\cdot,\cdot), s\sim \PP_{h-1}(\cdot|s'',a'',b''), (a,b)\sim \sigma_h(\cdot,\cdot|s) }[g(s,a,b)^2]\nonumber\\
&\quad \leq  2k B^2 \zeta^k_{h-1}  + 2k \sup_{a\in\cA,b\in\cB,s\in \cS} \frac{\sigma_h(a,b|s)}{\Unif(a)\Unif(b)} \EE_{(s,a,b)\sim\breve{\rho}^k_h(\cdot,\cdot,\cdot)}[ g(s,a,b)^2] \nonumber\\
&\quad =  2k B^2 \zeta^k_{h-1} + 2k  |\cA||\cB|\cdot  \EE_{(s,a,b)\sim\breve{\rho}^k_h(\cdot,\cdot,\cdot)}[ g(s,a,b)^2], \label{eq:step-back-mg3}
\end{align}
\end{small}
where the first inequality is by $(a+b)^2\leq 2a^2 + 2b^2$, the second inequality is by $\EE_{(s'',a'',b'')\sim\tilde{\rho}^k_{h-1}(\cdot,\cdot,\cdot)} \allowbreak [\int_{\cS}  (\hat{\PP}^k_{h-1}(s|s'',a'',b'') -\PP_{h-1}(s|s'',a'',b'') ) \sum_{a\in \cA,b\in \cB}\sigma_h(a,b|s) g(s,a,b)\mathrm{d}s  ]^2 \leq B^2\EE_{(s'',a'',b'')\sim\tilde{\rho}^k_{h-1}(\cdot,\cdot,\cdot)} \allowbreak \|\hat{\PP}^k_{h-1}(\cdot|s'',a'',b'') -\PP_{h-1}(\cdot|s'',a'',b'')\|_1^2\leq B^2 \zeta^k_{h-1}$, the third inequality is by Jensen's inequality, and the fourth inequality is by substituting the joint policy $\sigma$ with the uniform distribution.

Combining \eqref{eq:step-back-mg1},\eqref{eq:step-back-mg2}, and \eqref{eq:step-back-mg3}, we have for any $h \geq 2$,
\begin{align*}
   &\left|\EE_{(s,a,b)\sim d^{\sigma, \hat{\PP}^k}_h(\cdot,\cdot,\cdot)}[g(s,a,b)]\right| \\
&\leq \sqrt{2k B^2 \zeta^k_{h-1} + 2k  |\cA||\cB|\cdot  \EE_{(s,a,b)\sim\breve{\rho}^k_h(\cdot,\cdot,\cdot)}[ g(s,a,b)^2] + \lambda_k B^2 d/(\CS)^2} \cdot \EE_{d^{\sigma, \hat{\PP}^k}_{h-1}}\left\|\hat{\phi}^k_{h-1}\right\|_{\Sigma_{\tilde{\rho}^k_{h-1}, \hat{\phi}^k_{h-1}}^{-1}}.
\end{align*}
On the other hand, for $h=1$, we have
\begin{align*}
\left|\EE_{(s,a,b)\sim d^{\sigma, \PP}_1(\cdot,\cdot,\cdot)}[g(s,a,b)]\right| &= \sqrt{\EE_{(a,b)\sim \sigma_1(\cdot,\cdot|s_1)}[g(s_1,a,b)^2]} \leq \sqrt{|\cA||\cB|  \EE_{(a,b)\sim \tilde{\rho}_1^k(s_1,\cdot,\cdot)}[g(s_1,a,b)^2] },
\end{align*} 
where we let $\tilde{\rho}_1^k(s_1,a,b) =\Unif(a)\Unif(b)$ and the last inequality is by $\EE_{(a,b)\sim \sigma_1(\cdot,\cdot|s_1)}[g(s_1,a,b)^2]\leq \max_{a,b} \frac{\sigma_1(a,b|s_1)}{\Unif(a)\Unif(b)}  \EE_{(a,b)\sim \tilde{\rho}_1^k(s_1,\cdot,\cdot)}[g(s_1,a,b)^2]$. The proof is completed.
\end{proof}


\begin{lemma}\label{lem:expand2-mg} Suppose that $\hat{\PP}^k$ is the estimated transition obtained at episode $k$ of Algorithm \ref{alg:contrastive-mg}. We define $\tilde{\rho}^k_h(\cdot,\cdot,\cdot) := \frac{1}{k}\sum_{k'=0}^{k-1} \tilde{d}^{\sigma^{k'}}_h(\cdot,\cdot,\cdot)$ for all $h\geq 1$ with $\tilde{\rho}_1^k(s_1,a,b) =\Unif(a)\Unif(b)$ and $\rho^k_h(\cdot,\cdot,\cdot) := \frac{1}{k}\sum_{k'=0}^{k-1} d^{\sigma^{k'}}_h(\cdot,\cdot,\cdot)$ for all $h\geq 2$.  Then for any function $g:\cS\times\cA\times\cB\mapsto [0,B]$ and joint policy $\sigma$, we have for any $h \geq 2$, the following inequality holds 
\begin{align*}
   &\left|\EE_{(s,a,b)\sim d^{\sigma, \PP}_h(\cdot,\cdot,\cdot)}[g(s,a,b)]\right| \\
&\qquad\leq \sqrt{k  |\cA||\cB|\cdot  \EE_{(s,a,b)\sim\tilde{\rho}^k_h(\cdot,\cdot,\cdot)}[ g(s,a,b)^2] + \lambda_k B^2 d} \cdot \EE_{(s',a',b')\sim d^{\sigma, \PP}_{h-1}(\cdot,\cdot,\cdot)}\left\|\phi^*_{h-1}(s',a',b')\right\|_{\Sigma_{\rho^k_{h-1}, \phi^*_{h-1}}^{-1}}.
\end{align*}
In addition, for $h=1$, we have
\begin{align*}
\left|\EE_{(s,a,b)\sim d^{\sigma, \PP}_1(\cdot,\cdot,\cdot)}[g(s,a,b)]\right| &\leq \sqrt{\EE_{(a,b)\sim \sigma_1(\cdot,\cdot|s_1)}[g(s_1,a,b)^2]} \leq \sqrt{|\cA||\cB|  \EE_{(a,b)\sim \tilde{\rho}_1^k(s_1,\cdot,\cdot)}[g(s_1,a,b)^2] }.
\end{align*} 
\end{lemma}

\begin{proof}
For any function $g:\cS\times\cA\times\cB\mapsto \RR$ and any joint policy $\sigma$, under the true transition model $\PP$, for any $h\geq 2$, we have
\begin{align}
\begin{aligned}\label{eq:step-back2-mg1}
&\left|\EE_{(s,a,b)\sim d^{\sigma, \PP}_h(\cdot,\cdot,\cdot)}[g(s,a,b)]\right| \\
&\quad= \left|\EE_{(s',a',b')\sim d^{\sigma, \PP}_{h-1}(\cdot,\cdot,\cdot), s\sim \PP_{h-1}(\cdot|s',a'), (a,b)\sim \sigma_h(\cdot,\cdot|s)}[g(s,a,b)]\right| \\
&\quad= \bigg|\EE_{(s',a',b')\sim d^{\sigma, \PP}_{h-1}(\cdot,\cdot,\cdot)} \bigg[\phi^*_{h-1}(s',a',b')^\top \int_{\cS} \psi^*_{h-1}(s) \sum_{a\in \cA,b\in \cB}\sigma_h(a,b|s) g(s,a,b)\mathrm{d} s\bigg]\bigg|\\
&\quad\leq  \EE_{ d^{\sigma, \PP}_{h-1}}\left\|\phi^*_{h-1}\right\|_{\Sigma_{\rho^k_{h-1}, \phi^*_{h-1}}^{-1}} \cdot \bigg\| \int_{\cS} \psi^*_{h-1}(s) \sum_{a\in \cA,b\in \cB}\sigma_h(a,b|s) g(s,a,b)\mathrm{d} s\bigg\|_{\Sigma_{\rho^k_{h-1}, \phi^*_{h-1}}},
\end{aligned}
\end{align}
where the inequality is by Cauchy-Schwarz inequality. We define the covariance matrix $\Sigma_{\rho^k_{h-1}, \phi^*_{h-1}} :=k \EE_{(s,a,b)\sim\rho^k_{h-1}}[\phi^*_{h-1}(s,a,b)\phi^*_{h-1}(s,a,b)^\top] + \lambda_k I$ with $\rho^k_{h-1}(s,a,b) = \frac{1}{k}\sum_{k'=0}^{k-1} d^{\pi^{k'}}_{h-1}(s,a,b)$.

Next, we have
\begin{footnotesize}
\begin{align}
&\left\| \int_{\cS} \psi^*_{h-1}(s) \sum_{a\in \cA,b\in \cB}\sigma_h(a,b|s) g(s,a,b)\mathrm{d} s\right\|_{\Sigma_{\rho^k_{h-1},\phi^*_{h-1}}}^2\nonumber\\
& = k\left(\int_{\cS} \psi^*_{h-1}(s) \sum_{a\in \cA,b\in \cB}\sigma_h(a,b|s) g(s,a,b)\mathrm{d}s\right)^\top  \EE_{\rho^k_{h-1}}\left[\phi^*_{h-1}(\phi^*_{h-1})^\top \right]  \left(\int_{\cS} \psi^*_{h-1}(s) \sum_{a\in \cA,b\in \cB}\sigma_h(a,b|s) g(s,a,b)\mathrm{d}s\right)\nonumber\\
& \quad + \lambda_k\left(\int_{\cS} \psi^*_{h-1}(s) \sum_{a\in \cA,b\in \cB}\sigma_h(a,b|s) g(s,a,b)\mathrm{d}s\right)^\top \left(\int_{\cS} \psi^*_{h-1}(s) \sum_{a\in \cA,b\in \cB}\sigma_h(a,b|s) g(s,a,b)\mathrm{d}s\right)\nonumber\\
& = k \EE_{(s'',a'',b'')\sim\rho^k_{h-1}(\cdot,\cdot,\cdot)}\left[\int_{\cS} \phi^*_{h-1}(s'',a'',b'')^\top \psi^*_{h-1}(s) \sum_{a\in \cA,b\in \cB}\sigma_h(a,b|s) g(s,a,b)\mathrm{d}s  \right] \nonumber \\
& \quad + \lambda_k\left(\int_{\cS} \psi^*_{h-1}(s) \sum_{a\in \cA,b\in \cB}\sigma_h(a,b|s) g(s,a,b)\mathrm{d}s\right)^\top \left(\int_{\cS} \psi^*_{h-1}(s) \sum_{a\in \cA,b\in \cB}\sigma_h(a,b|s) g(s,a,b)\mathrm{d}s\right)\nonumber\\
& \leq k \EE_{(s'',a'',b'')\sim\rho^k_{h-1}(\cdot,\cdot,\cdot)}\left[\int_{\cS} \phi^*_{h-1}(s'',a'',b'')^\top \psi^*_{h-1}(s) \sum_{a\in \cA,b\in \cB}\sigma_h(a,b|s) g(s,a,b)\mathrm{d}s  \right]^2  + \lambda_k B^2 d, \label{eq:step-back2-mg2}
\end{align}
\end{footnotesize}
where, by Assumption \ref{assump:low-rank}, the last inequality is due to 
\begin{align*}
&\bigg(\int_{\cS} \psi^*_{h-1}(s) \sum_{a\in \cA,b\in \cB}\sigma_h(a,b|s) g(s,a,b)\mathrm{d}s\bigg)^\top \bigg(\int_{\cS} \psi^*_{h-1}(s) \sum_{a\in \cA,b\in \cB}\sigma_h(a,b|s) g(s,a,b)\mathrm{d}s\bigg) \\
&\qquad \leq B^2 \left|\int_{\cS} \psi^*_{h-1}(s) \mathrm{d}s\right|_2^2\leq B^2 d.
\end{align*}
Moreover, we have
\begin{align}
\begin{aligned}\label{eq:step-back2-mg3}
&k \EE_{(s'',a'',b'')\sim\rho^k_{h-1}(\cdot,\cdot,\cdot)}\bigg[\int_{\cS} \phi^*_{h-1}(s'',a'',b'')^\top \psi^*_{h-1}(s) \sum_{a\in \cA,b\in \cB}\sigma_h(a,b|s) g(s,a,b)\mathrm{d}s  \bigg]^2 \\
&\qquad = k \EE_{(s'',a'',b'')\sim\rho^k_{h-1}(\cdot,\cdot,\cdot)}\bigg[\int_{\cS} \PP_{h-1}(s|s'',a'',b'') \sum_{a\in \cA,b\in \cB}\sigma_h(a,b|s) g(s,a,b)\mathrm{d}s  \bigg]^2\\
&\qquad \leq  k \EE_{(s'',a'',b'')\sim\rho^k_{h-1}(\cdot,\cdot,\cdot), s\sim \PP_{h-1}(\cdot|s'',a'',b''), (a,b)\sim \sigma_h(\cdot,\cdot|s) }[g(s,a,b)^2]\\
&\qquad \leq k \sup_{a\in\cA,b\in\cA,s\in \cS} \frac{\sigma_h(a,b|s)}{\Unif(a)\Unif(b)} \EE_{(s,a,b)\sim\tilde{\rho}^k_h(\cdot,\cdot,\cdot)}[ g(s,a,b)^2]\\
&\qquad =  k  |\cA||\cB|\cdot  \EE_{(s,a,b)\sim\tilde{\rho}^k_h(\cdot,\cdot,\cdot)}[ g(s,a,b)^2],
\end{aligned}
\end{align}
where the first inequality is due to Jensen's inequality and the second inequality is by substituting the joint policy $\sigma$ with the uniform distribution and $\tilde{\rho}^k_h(s,a,b):=\rho^k_{h-1}(s',a',b')\PP_{h-1}(s|s',a',b')\Unif(a)\Unif(b)$ for all $h\geq 2$.
Combining \eqref{eq:step-back2-mg1},\eqref{eq:step-back2-mg2}, and \eqref{eq:step-back2-mg3}, we have for any $h \geq 2$,
\begin{align*}
   &\left|\EE_{(s,a,b)\sim d^{\sigma, \PP}_h(\cdot,\cdot,\cdot)}[g(s,a,b)]\right| \\
&\qquad\leq \sqrt{k  |\cA||\cB|\cdot  \EE_{(s,a,b)\sim\tilde{\rho}^k_h(\cdot,\cdot,\cdot)}[ g(s,a,b)^2] + \lambda_k B^2 d} \cdot \EE_{d^{\sigma, \PP}_{h-1}}\left\|\phi^*_{h-1}\right\|_{\Sigma_{\rho^k_{h-1}, \phi^*_{h-1}}^{-1}}.
\end{align*}

For $h=1$, we have
\begin{align*}
\left|\EE_{(s,a,b)\sim d^{\sigma, \PP}_1(\cdot,\cdot,\cdot)}[g(s,a,b)]\right| &\leq \sqrt{\EE_{(a,b)\sim \sigma_1(\cdot,\cdot|s_1)}[g(s_1,a,b)^2]} \leq \sqrt{|\cA||\cB|  \EE_{(a,b)\sim \tilde{\rho}_1^k(s_1,\cdot,\cdot)}[g(s_1,a,b)^2] },
\end{align*} 
where we let $\tilde{\rho}_1^k(s_1,a,b) =\Unif(a)\Unif(b)$ and the last inequality is by $\EE_{(a,b)\sim \sigma_1(\cdot,\cdot|s_1)}[g(s_1,a,b)^2]\leq \max_{a,b} \frac{\sigma_1(a,b|s_1)}{\Unif(a)\Unif(b)}  \EE_{(a,b)\sim \tilde{\rho}_1^k(s_1,\cdot,\cdot)}[g(s_1,a,b)^2]$. The proof is completed.
\end{proof}


\begin{lemma}\label{lem:plan-mg} Suppose at the $k$-th episode of Algorithm \ref{alg:contrastive-mg}, 
$\pi^k,\nu^k$ are learned policies , $\iota_k$ is the CCE learning accuracy, and $\overline{V}_1^k(s_1)$ and $\underline{V}_1^k(s_1)$ are the value functions updated as in the algorithm. Moreover, for any joint policy $\sigma$, $\overline{V}_{k,1}^{\sigma}(s_1)$ is the value function associated with the Markov game defined by the reward function $r + \beta^k$ and the estimated transition $\hat{\PP}^k$ while $\underline{V}_{k,1}^{\sigma}(s_1)$ is the value function associated with the Markov game defined by the reward function $r - \beta^k$ and the estimated transition $\hat{\PP}^k$. Then we have
\begin{align*}
&\overline{V}_{k,1}^{\mathrm{br}(\nu^k), \nu^k}(s_1) \leq \overline{V}_1^k(s_1)+H\iota_k, \qquad \underline{V}_{k,1}^{\pi^k, \mathrm{br}(\pi^k)}(s_1) \geq \underline{V}_1^k(s_1)-H\iota_k.
\end{align*}
\end{lemma}

\begin{proof}
We prove this lemma by induction. For the first inequality in this lemma, we have $\overline{V}_{k,H+1}^{\mathrm{br}(\nu^k), \nu^k}(s) = \overline{V}_{H+1}^k(s) = 0$ for any $s\in \cS$. Next, we assume the following inequality holds 
\begin{align*}
\overline{V}_{k,h+1}^{\mathrm{br}(\nu^k), \nu^k}(s) \leq  \overline{V}_{h+1}^k(s) + (H-h)\iota_k.
\end{align*}
Then, with the above inequality, by the Bellman equation, we have
\begin{align}
&\overline{Q}_{k,h}^{\mathrm{br}(\nu^k), \nu^k}(s,a,b) - \overline{Q}_h^k(s,a,b)\nonumber\\
&\qquad =  r_h(s,a,b) +  \beta_h^k(s,a,b) + \PP_h\overline{V}_{k,h+1}^{\mathrm{br}(\nu^k), \nu^k}(s,a,b) - r_h(s,a,b) -  \beta_h^k(s,a,b) - \PP_h\overline{V}_{h+1}^k(s,a,b) \nonumber\\
&\qquad =  \PP_h\overline{V}_{k,h+1}^{\mathrm{br}(\nu^k), \nu^k}(s,a,b) - \PP_h\overline{V}_{h+1}^k(s,a,b) \leq (H-h)\iota_k.\label{eq:opt-mg-1}
\end{align}
Then, we have
\begin{align*}
\overline{V}_{k,h}^{\mathrm{br}(\nu^k), \nu^k}(s) &=\EE_{a\sim\mathrm{br}(\nu^k)_h, b\sim\nu_h^k}\left[\overline{Q}_{k,h}^{\mathrm{br}(\nu^k), \nu^k}(s,a,b) \right]\\
&\leq \EE_{a\sim\mathrm{br}(\nu^k)_h, b\sim\nu_h^k}\left[\overline{Q}_h^k(s,a,b) \right] + (H-h)\iota_k\\
&\leq\EE_{(a,b)\sim\sigma_h^k}\left[\overline{Q}_h^k(s,a,b) \right]+(H+1-h)\iota_k  \\
&= \overline{V}_h^k(s)+(H+1-h)\iota_k,
\end{align*}
where the first inequality is by \eqref{eq:opt-mg-1} and the second inequality is by the definition of $\iota_k$-CCE as in Definition \ref{def:cce}. Thus, we obtain
\begin{align*}
\overline{V}_{k,1}^{\mathrm{br}(\nu^k), \nu^k}(s_1) \leq \overline{V}_1^k(s_1)+H\iota_k.
\end{align*}
For the second inequality in this lemma, we have $\underline{V}_{k,H+1}^{\pi^k, \mathrm{br}(\pi^k)}(s) = \underline{V}_{H+1}^k(s) = 0$. Then, we assume that
\begin{align*}
\underline{V}_{k,h+1}^{\pi^k, \mathrm{br}(\pi^k)}(s) \geq  \underline{V}_{h+1}^k(s)-(H-h)\iota_k.
\end{align*}
Then, by the Bellman equation, we have
\begin{align}
\begin{aligned}
\label{eq:opt-mg-2}
&\underline{Q}_{k,h}^{\pi^k, \mathrm{br}(\pi^k)}(s,a,b) - \underline{Q}_h^k(s,a,b)\\
&\qquad =  r_h(s,a,b) +  \beta_h^k(s,a,b) + \PP_h\underline{V}_{k,h+1}^{\pi^k, \mathrm{br}(\pi^k)}(s,a,b) - r_h(s,a,b) -  \beta_h^k(s,a,b) - \PP_h\underline{V}_{h+1}^k(s,a,b) \\
&\qquad =  \PP_h\underline{V}_{k,h+1}^{\pi^k, \mathrm{br}(\pi^k)}(s,a,b) - \PP_h\underline{V}_{h+1}^k(s,a,b) \geq -(H-h)\iota_k.
\end{aligned}
\end{align}
Then, we have
\begin{align*}
\underline{V}_{k,h}^{\pi^k, \mathrm{br}(\pi^k)}(s) &=\EE_{a\sim\pi^k_h, b\sim\mathrm{br}(\pi^k)_h}\left[\underline{Q}_{k,h}^{\pi^k, \mathrm{br}(\pi^k)}(s,a,b) \right]\\
&\geq \EE_{a\sim\pi^k_h, b\sim\mathrm{br}(\pi^k)_h}\left[\underline{Q}_h^k(s,a,b) \right] -(H-h)\iota_k\\
&\geq\EE_{(a,b)\sim\sigma_h^k}\left[\underline{Q}_h^k(s,a,b) \right] -(H+1-h)\iota_k \\
&= \underline{V}_h^k(s)-(H+1-h)\iota_k,
\end{align*}
where the first inequality is by \eqref{eq:opt-mg-2} and the second inequality is by the definition of $\iota_k$-CCE as in Definition \ref{def:cce}. Thus, we obtain
\begin{align*}
\underline{V}_{k,1}^{\pi^k, \mathrm{br}(\pi^k)}(s_1) \geq \underline{V}_1^k(s_1)-H\iota_k.
\end{align*}
This completes the proof.
\end{proof}

\subsection{Proof of Lemma \ref{lem:stat-err-mg}} \label{sec:proof-stat-err-mg}
The proof of this lemma follows from \hyperref[sec:proof-stat-err]{Proof of Lemma \ref{lem:stat-err}} by expanding the action space from $\cA$ to $\cA\times \cB$. In this subsection, we briefly present the major steps of the proof.

\begin{proof}
For any function $f_h\in \cF$, we let $\Pr_h^f(y|s,a,b,s')$ denote the conditional probability characterized by the function $f_h$ at the step $h$, which is
\begin{align*}
\Pr{}_h^f(y|s,a,b,s') =  \left(\frac{f_h(s,a,b,s')}{1+f_h(s,a,b,s')}\right)^y\left(\frac{1}{1+f_h(s,a,b,s')}\right)^{1-y}.   
\end{align*}
Moreover, there is
\begin{align*}
\Pr{}_h^f(y,s'|s,a,b)  &= \Pr{}_h^f(y|s,a,b,s')\Pr{}_h(s'|s,a,b) \\
&= \left(\frac{f_h(s,a,b,s')\Pr{}_h(s'|s,a,b)}{1+f_h(s,a,b,s')}\right)^y\left(\frac{\Pr{}_h(s'|s,a,b)}{1+f_h(s,a,b,s')}\right)^{1-y},
\end{align*}
where we have
\begin{align}
    \Pr{}_h(s'|s,a,b) &= \Pr{}_h(y=1| s,a,b) \Pr{}_h(s'|y=1, s,a,b) + \Pr{}_h(y = 0| s,a,b) \Pr{}_h(s'|y=0, s,a,b) \nonumber\\
    &= \Pr{}_h(y=1) \Pr{}_h(s'|y=1, s,a,b) + \Pr{}_h(y = 0) \Pr{}_h(s'|y=0, s,a,b) \nonumber\\
    &= \frac{1}{2} [\PP_h(s'|s,a,b) + \cPS(s')] \geq \frac{1}{2}\CS > 0.\label{eq:tran-lower-mg}
\end{align}
Thus, we have the equivalency of solving the following two problems with $f_h(s,a,b,s')=\phi_h(s,a,b)^\top \psi_h(s')$, which is
\begin{align}\label{eq:equi-solution-mg}
\max_{\phi_h\in \Phi, \psi_h\in \Psi} \sum_{(s,a,s',y)\in \cD_h^k} \log \Pr{}_h^f(y|s,a,b,s') = \max_{\phi_h, \psi_h} \sum_{(s,a,s',y)\in \cD_h^k} \log \Pr{}_h^f(y,s'|s,a,b).
\end{align}
We denote the solution of \eqref{eq:equi-solution-mg} as $\tilde{\phi}_h^k$ and $\tilde{\psi}_h^k$ such that 
\begin{align*}
 \hat{f}_h^k(s,a,b,s') = \tilde{\psi}^k_h(s')^\top \tilde{\phi}^k_h(s,a,b).   
\end{align*}

According to Algorithm \ref{alg:sample-mg}, for any $h\geq 2$ and $k'\in [k]$, the data $(s,a,b)$ is sampled from both $\tilde{d}_h^{\sigma^{k'}}(\cdot,\cdot,\cdot)$ and $\breve{d}_h^{\sigma^{k'}}(\cdot,\cdot,\cdot)$. Then, by Lemma \ref{lem:recover-mle}, solving the contrastive loss in \eqref{eq:contra-loss} with letting $z = (s,a,b)$ gives, with probability at least $1-\delta$, for all $h\geq 2$,
\begin{align*}
\sum_{k'=1}^k  \Bigg[ &\EE_{(s,a,b)\sim \tilde{d}_h^{\sigma^{k'}}(\cdot,\cdot,\cdot)} \left\|\Pr{}_h^{\hat{f}^k}(\cdot,\cdot|s,a,b) - \Pr{}_h^{f^*}(\cdot,\cdot|s,a,b) \right\|_{\TV}^2 \\
&+ \EE_{(s,a,b)\sim \breve{d}_h^{\sigma^{k'}}(\cdot,\cdot,\cdot)} \left\|\Pr{}_h^{\hat{f}^k}(\cdot,\cdot|s,a,b) - \Pr{}_h^{f^*}(\cdot,\cdot|s,a,b) \right\|_{\TV}^2 \Bigg] \leq 2\log (2kH|\cF|/\delta),
\end{align*}
which is equivalent to 
\begin{align}
\begin{aligned}\label{eq:ave-mle-bound1-mg} 
&\EE_{(s,a,b)\sim \tilde{\rho}_h^k(\cdot,\cdot,\cdot)} \left\|\Pr{}_h^{\hat{f}^k}(\cdot,\cdot|s,a,b) - \Pr{}_h^{f^*}(\cdot,\cdot|s,a,b) \right\|_{\TV}^2 \\
&\qquad + \EE_{(s,a,b)\sim \breve{\rho}_h^k(\cdot,\cdot,\cdot)} \left\|\Pr{}_h^{\hat{f}^k}(\cdot,\cdot|s,a,b) - \Pr{}_h^{f^*}(\cdot,\cdot|s,a,b) \right\|_{\TV}^2  \leq 2\log (2kH|\cF|/\delta)/k, \quad \forall h\geq 2,
\end{aligned}
\end{align}
where $\tilde{\rho}^k_h(s,a,b) = \frac{1}{k}\sum_{k'=0}^{k-1} \tilde{d}^{\pi^{k'}}_h(s,a,b)$ and $\breve{\rho}^k_h(s,a,b) = \frac{1}{k}\sum_{k'=0}^{k-1} \breve{d}^{\pi^{k'}}_h(s,a,b)$. Moreover, since for $h=1$, the data is only sampled from $\tilde{d}^{\pi^{k'}}_1(\cdot,\cdot,\cdot)$ for any $k'\in [k]$, then we analogously have
\begin{align} \label{eq:ave-mle-bound2-mg} 
&\EE_{(s,a,b)\sim \tilde{\rho}_1^k(\cdot,\cdot,\cdot)} \left\|\Pr{}_1^{\hat{f}^k}(\cdot,\cdot|s,a,b) - \Pr{}_1^{f^*}(\cdot,\cdot|s,a,b) \right\|_{\TV}^2  \leq 2\log (2k|\cF|/\delta)/k.
\end{align}
Thus, combining \eqref{eq:ave-mle-bound1-mg} and \eqref{eq:ave-mle-bound2-mg}, with probability at least $1-2\delta$, we have
\begin{align}
\begin{aligned}\label{eq:ave-mle-bound-mg} 
&\EE_{(s,a,b)\sim \tilde{\rho}_h^k(\cdot,\cdot,\cdot)} \left\|\Pr{}_h^{\hat{f}^k}(\cdot,\cdot|s,a,b) - \Pr{}_h^{f^*}(\cdot,\cdot|s,a,b) \right\|_{\TV}^2 \leq 2\log (2kH|\cF|/\delta)/k, \quad \forall h\geq 1,\\
&\EE_{(s,a,b)\sim \breve{\rho}_h^k(\cdot,\cdot,\cdot)} \left\|\Pr{}_h^{\hat{f}^k}(\cdot,\cdot|s,a,b) - \Pr{}_h^{f^*}(\cdot,\cdot|s,a,b) \right\|_{\TV}^2  \leq 2\log (2kH|\cF|/\delta)/k, \quad \forall h\geq 2,
\end{aligned}
\end{align}
Next, we show the recovery error bound of the transition model based on \eqref{eq:ave-mle-bound-mg}. We have
\begin{align*}
 &\left\|\Pr{}_h^{\hat{f}^k}(\cdot,\cdot|s,a,b) - \Pr{}_h^{f^*}(\cdot,\cdot|s,a,b) \right \|_{\TV}^2\\
 &\qquad = \bigg( \left\|\Pr{}_h^{\hat{f}^k}(y=0,\cdot|s,a,b) - \Pr{}_h^{f^*}(y=0,\cdot|s,a,b) \right \|_{\TV} \\
 &\qquad \quad + \left\|\Pr{}_h^{\hat{f}^k}(y=1,\cdot|s,a,b) - \Pr{}_h^{f^*}(y=1,\cdot|s,a,b) \right \|_{\TV}   \bigg)^2 \\
&\qquad = 4\left\|\frac{\Pr{}_h(\cdot|s,a,b)}{1+\hat{f}_h^k(s,a,b,\cdot)} - \frac{\Pr{}_h(\cdot|s,a,b)}{1+f_h^*(s,a,b,\cdot)} \right \|_{\TV}^2\\
&\qquad = 2 \left[\int_{s'\in \cS}\frac{\Pr{}_h(s'|s,a,b)\cdot |f_h^*(s,a,b,s')-\hat{f}_h^k(s,a,b,s')|}{[1+\hat{f}_h^k(s,a,b,s')]\cdot[1+f_h^*(s,a,b,s')]}\mathrm{d}s' \right]^2,
\end{align*}
where $f^*(s,a,b,s') = \frac{\PP(s'|s,a,b)}{\cPS(s')}  ~~\text{with}~~ \cPS(s') \geq \CS,~~\forall s'\in\cS$ and the second equation is due to $\big\|\Pr{}_h^{\hat{f}^k}(y=0,\cdot|s,a,b) - \Pr{}_h^{f^*}(y=0,\cdot|s,a,b) \big \|_{\TV} = \big\|\Pr{}_h^{\hat{f}^k}(y=1,\cdot|s,a,b) - \Pr{}_h^{f^*}(y=1,\cdot|s,a,b) \big \|_{\TV} = \Big\|\frac{\Pr{}_h(\cdot|s,a,b)}{1+\hat{f}_h^k(s,a,b,\cdot)} - \frac{\Pr{}_h(\cdot|s,a,b)}{1+f_h^*(s,a,b,\cdot)} \Big \|_{\TV}$ . 
Moreover, according to Lemma \ref{lem:opt-contra-loss} and \eqref{eq:tran-lower}, we have 
\begin{align*}
&\frac{\Pr{}_h(s'|s,a,b)\cdot |f_h^*(s,a,b,s')-\hat{f}_h^k(s,a,b,s')|}{[1+\hat{f}_h^k(s,a,b,s')]\cdot[1+f_h^*(s,a,b,s')]}\\
&\qquad =\frac{1/2 \cdot [\PP_h(s'|s,a,b) + \cPS(s')]\cdot |\PP_h(s'|s,a,b)/\cPS(s')-\hat{f}_h^k(s,a,b,s')|}{[1+\hat{f}_h^k(s,a,b,s')]\cdot[1+\PP_h(s'|s,a,b)/\cPS(s')]}\\
&\qquad =\frac{1/2 \cdot |\PP_h(s'|s,a,b)-\cPS(s') \hat{f}_h^k(s,a,b,s')|}{1+\hat{f}_h^k(s,a,b,s')} \geq  \frac{|\PP_h(s'|s,a,b)-\cPS(s') \hat{f}_h^k(s,a,b,s')|}{4\sqrt{d}/\CS},
\end{align*}
where the inequality is due to $[1+\hat{f}_h^k(s,a,b,s')]\leq (1+\sqrt{d})\leq 2\sqrt{d}/\CS$ since $\hat{f}_h^k(s,a,b,s')\leq \sqrt{d}/\CS$ with $d\geq 1$ and $0<\CS\leq 1$. Thus, combining this inequality with \eqref{eq:ave-mle-bound-mg}, we further have, $\forall h\geq 2$,
\begin{align}
\begin{aligned}\label{eq:init-P-diff-1-mg}
    &\mathbb{E}_{(s,a,b)\sim \tilde{\rho}_h^k(\cdot,\cdot,\cdot)} \left\|\PP_h(\cdot|s,a,b)-\cPS(\cdot) \tilde{\phi}_h^k(s,a,b)^\top \tilde{\psi}_h^k (\cdot) \right\|_{\TV}^2  \leq 8d/(\CS)^2\cdot\log (2kH|\cF|/\delta)/k . 
\end{aligned}
\end{align}
Similarly, we can obtain
\begin{align}
\begin{aligned}\label{eq:init-P-diff-2-mg}
&\EE_{(s,a,b)\sim \tilde{\rho}_1^k(\cdot,\cdot,\cdot)} \left\|\PP_h(\cdot|s,a,b)-\cPS(\cdot) \tilde{\phi}_h^k(s,a,b)^\top \tilde{\psi}_h^k (\cdot)\right\|_{\TV}^2 \leq 8d/(\CS)^2\cdot\log (2k|\cF|/\delta)/k, \\
&\EE_{(s,a,b)\sim \breve{\rho}_h^k(\cdot,\cdot,\cdot)} \left\|\PP_h(\cdot|s,a,b)-\cPS(\cdot) \tilde{\phi}_h^k(s,a,b)^\top \tilde{\psi}_h^k (\cdot)\right\|_{\TV}^2  \leq 8d/(\CS)^2\cdot\log (2kH|\cF|/\delta)/k, \quad \forall h\geq 2.
\end{aligned}
\end{align}
Now we define 
\begin{align*}
\hat{g}_h^k(s,a,b,s') := \cPS(s') \tilde{\phi}_h^k(s,a,b)^\top \tilde{\psi}_h^k (s').
\end{align*} 
Since $\int_{s'\in\cS}\hat{g}_h^k(s,a,b,s')\mathrm{d}s'$ may not be guaranteed to be $1$, to obtain an approximator of the transition model $\PP_h$ lying on a probability simplex, we should further normalize $\hat{g}_h^k(s,a,b,s')$. Thus, we define for all $(s,a,b,s')\in \cS\times\cA\times\cB\times\cS$,
\begin{align*}
\hat{\PP}_h^k(s'|s,a,b) := \frac{\hat{g}_h^k(s,a,b,s')}{\|\hat{g}_h^k(s,a,b,\cdot)\|_1}=  \frac{\hat{g}_h^k(s,a,b,s')}{\int_{s'\in\cS}\hat{g}_h^k(s,a,b,s')\mathrm{d}s'} = \frac{\cPS(s') \tilde{\phi}_h^k(s,a,b)^\top \tilde{\psi}_h^k (s')}{\int_{s'\in\cS}\cPS(s') \tilde{\phi}_h^k(s,a,b)^\top \tilde{\psi}_h^k (s')\mathrm{d}s'}.
\end{align*}
We further let 
\begin{align*}
    &\hat{\phi}_h^k(s,a,b) :=  \frac{ \tilde{\phi}_h^k(s,a,b)}{ \int_{s'\in\cS}\cPS(s') \tilde{\phi}_h^k(s,a,b)^\top \tilde{\psi}_h^k (s')\mathrm{d}s'},\\
    & \hat{\psi}_h^k(s') :=\cPS(s')  \tilde{\psi}_h^k (s'),
\end{align*}
such that
\begin{align*}
    \hat{\PP}_h^k(s'|s,a,b) = \hat{\psi}_h^k(s')^\top \hat{\phi}_h^k(s,a,b).
\end{align*}
Next, we give the upper bound of the approximation error $\mathbb{E}_{(s,a,b)\sim \tilde{\rho}_h^k(\cdot,\cdot,\cdot)}\|\hat{\PP}_h^k(\cdot|s,a,b)- \PP_h(\cdot|s,a,b)\|_{\TV}^2$. We have
\begin{align}
\begin{aligned}\label{eq:P-diff0-mg}
&\mathbb{E}_{(s,a,b)\sim \tilde{\rho}_h^k(\cdot,\cdot,\cdot)}\|\hat{\PP}_h^k(\cdot|s,a,b)- \PP_h(\cdot|s,a,b)\|_{\TV}^2 \\
&\qquad \leq 2\mathbb{E}_{(s,a,b)\sim \tilde{\rho}_h^k(\cdot,\cdot,\cdot)}\|\hat{\PP}_h^k(\cdot|s,a,b)- \hat{g}_h^k(s,a,b,\cdot) \|_{\TV}^2 \\
&\qquad\quad + 2\mathbb{E}_{(s,a,b)\sim \tilde{\rho}_h^k(\cdot,\cdot,\cdot)}\|\hat{g}_h^k(s,a,b,\cdot) - \PP_h(\cdot|s,a,b)\|_{\TV}^2  \\
&\qquad \leq 2\mathbb{E}_{(s,a,b)\sim \tilde{\rho}_h^k(\cdot,\cdot,\cdot)}\|\hat{\PP}_h^k(\cdot|s,a,b)- \hat{g}_h^k(s,a,b,\cdot) \|_{\TV}^2 \\
&\qquad\quad + 16d/(\CS)^2\cdot\log (2kH|\cF|/\delta)/k,  
\end{aligned}
\end{align}
where the first inequality is by $(x+y)^2\leq 2x^2 + 2y^2$ and the last inequality is by \eqref{eq:init-P-diff-1-mg}. Moreover, we have
\begin{align*}
    &\mathbb{E}_{(s,a,b)\sim \tilde{\rho}_h^k(\cdot,\cdot,\cdot)}\|\hat{\PP}_h^k(\cdot|s,a,b)- \hat{g}_h^k(s,a,b,\cdot) \|_{\TV}^2\\
    &\qquad = \mathbb{E}_{(s,a,b)\sim \tilde{\rho}_h^k(\cdot,\cdot,\cdot)}\left\|\frac{\hat{g}_h^k(s,a,b,s')}{\|\hat{g}_h^k(s,a,b,\cdot)\|_1}- \hat{g}_h^k(s,a,b,\cdot)\right\|_{\TV}^2\\
    &\qquad = \frac{1}{4}\mathbb{E}_{(s,a,b)\sim \tilde{\rho}_h^k(\cdot,\cdot,\cdot)}\left(\|\hat{g}_h^k(s,a,b,\cdot)\|_1- 1\right)^2\\
    &\qquad \leq \frac{1}{4}\mathbb{E}_{(s,a,b)\sim \tilde{\rho}_h^k(\cdot,\cdot,\cdot)}\left(\|\hat{g}_h^k(s,a,b,\cdot)-\PP_h(\cdot|s,a,b)\|_1 + \|\PP_h(\cdot|s,a,b)\|_1- 1\right)^2\\
    &\qquad \leq \frac{1}{4}\mathbb{E}_{(s,a,b)\sim\tilde{\rho}_h^k(\cdot,\cdot,\cdot)}\|\hat{g}_h^k(s,a,b,\cdot)-\PP_h(\cdot|s,a,b)\|_1^2 \\ &\qquad = \mathbb{E}_{(s,a,b)\sim \tilde{\rho}_h^k(\cdot,\cdot,\cdot)}\|\hat{g}_h^k(s,a,b,\cdot)-\PP_h(\cdot|s,a,b)\|_{\TV}^2\leq 8d/(\CS)^2\cdot\log (2kH|\cF|/\delta)/k.
\end{align*}
Combining the above inequality with \eqref{eq:P-diff0}, we eventually obtain
\begin{align*}
\mathbb{E}_{(s,a,b)\sim \tilde{\rho}_h^k(\cdot,\cdot,\cdot)}\|\hat{\PP}_h^k(\cdot|s,a,b)- \PP_h(\cdot|s,a,b)\|_{\TV}^2 \leq 32d/(\CS)^2\cdot\log (2kH|\cF|/\delta)/k, \quad \forall h\geq 1.
\end{align*}
Thus, we similarly have
\begin{align*}
&\mathbb{E}_{(s,a,b)\sim\breve{\rho}_h^k(\cdot,\cdot,\cdot)}\|\hat{\PP}_h^k(\cdot|s,a,b)- \PP_h(\cdot|s,a,b)\|_{\TV}^2 \leq 32d/(\CS)^2\cdot\log (2kH|\cF|/\delta)/k, \quad \forall h\geq 2.
\end{align*}
The above three inequalities hold with probability at least $1-2\delta$. This completes the proof.
\end{proof}

\subsection{Proof of Theorem \ref{thm:main-mg}} \label{sec:proof-thm-main-mg}
\begin{proof} We define two auxiliary MGs respectively by reward function $r+\beta^k$ and transition model $\hat{\PP}^k$, and $r-\beta^k$, $\hat{\PP}^k$. Then for any joint policy $\sigma$, let $\overline{V}_{k,h}^{\sigma}$ and $\underline{V}_{k,h}^{\sigma}$ be the associated value functions on the two auxiliary MGs respectively. We first decompose the instantaneous regret term $V_1^{\mathrm{br}(\nu^k), \nu^k}(s_1) -  V_1^{\pi^k, \mathrm{br}(\pi^k)}(s_1)$ as follows
\begin{align}
\begin{aligned} \label{eq:decomp-mg-init}
&    V_1^{\mathrm{br}(\nu^k), \nu^k}(s_1) -  V_1^{\pi^k, \mathrm{br}(\pi^k)}(s_1) \\
&\qquad =  \underbrace{V_1^{\mathrm{br}(\nu^k), \nu^k}(s_1) - \overline{V}_{k,1}^{\mathrm{br}(\nu^k), \nu^k}(s_1)}_{(i)} + \underbrace{\overline{V}_{k,1}^{\mathrm{br}(\nu^k), \nu^k}(s_1) - \overline{V}_1^k(s_1)}_{(ii)} + \underbrace{\overline{V}_1^k(s_1) - \underline{V}_1^k(s_1)}_{(iii)} \\
&\qquad \quad + \underbrace{\underline{V}_1^k(s_1)-  \underline{V}_{k,1}^{\pi^k, \mathrm{br}(\pi^k)}(s_1)}_{(iv)} + \underbrace{\underline{V}_{k,1}^{\pi^k, \mathrm{br}(\pi^k)}(s_1) - V_1^{\pi^k, \mathrm{br}(\pi^k)}(s_1)}_{(v)}.
\end{aligned}
\end{align}
Terms $(ii)$ and $(iv)$ depict the planning error on two auxiliary Markov games. According to Lemma \ref{lem:plan-mg}, we have
\begin{align*}
&\overline{V}_{k,1}^{\mathrm{br}(\nu^k), \nu^k}(s_1) \leq \overline{V}_1^k(s_1)+H\iota_k, \\
&\underline{V}_{k,1}^{\pi^k, \mathrm{br}(\pi^k)}(s_1) \geq \underline{V}_1^k(s_1)-H\iota_k,
\end{align*}
where $\iota_k$ is the learning accuracy of CCE. Thus, together with \eqref{eq:decomp-mg-init}, we have
\begin{align}
\begin{aligned}\label{eq:decomp-mg-init-2}
&V_1^{\mathrm{br}(\nu^k), \nu^k}(s_1) -  V_1^{\pi^k, \mathrm{br}(\pi^k)}(s_1) \\
&\qquad =  \underbrace{V_1^{\mathrm{br}(\nu^k), \nu^k}(s_1) - \overline{V}_{k,1}^{\mathrm{br}(\nu^k), \nu^k}(s_1)}_{(i)}  + \underbrace{\overline{V}_1^k(s_1) - \underline{V}_1^k(s_1)}_{(iii)} \\
&\qquad\quad + \underbrace{\underline{V}_{k,1}^{\pi^k, \mathrm{br}(\pi^k)}(s_1) - V_1^{\pi^k, \mathrm{br}(\pi^k)}(s_1)}_{(v)} + 2H\iota_k.
\end{aligned}
\end{align}
Thus, to bound the term $V_1^{\mathrm{br}(\nu^k), \nu^k}(s_1) -  V_1^{\pi^k, \mathrm{br}(\pi^k)}(s_1)$, we only need to bound the terms $(i)$, $(iii)$, and $(v)$ as in \eqref{eq:decomp-mg-init-2}.

To bound term $(i)$, by Lemma \ref{lem:diff1-mg}, we have
\begin{align}
\begin{aligned}\label{eq:term-i-decomp-mg}
   (i)&= V_1^{\mathrm{br}(\nu^k), \nu^k}(s_1) - \overline{V}_{k,1}^{\mathrm{br}(\nu^k), \nu^k}(s_1)\\
   &= \EE \left[ \sum_{h=1}^H \left(-\beta^k_h(s_h,a_h,b_h) + (\PP_h - \hat{\PP}^k_h )V^{\mathrm{br}(\nu^k), \nu^k}_{h+1}(s_h,a_h,b_h)\right) \Bigggiven\mathrm{br}(\nu^k), \nu^k, \hat{\PP}^k \right]\\
    &\leq  \EE \left[ \sum_{h=1}^H \left(-\beta^k_h(s_h,a_h,b_h) + H \|\PP_h(\cdot|s_h,a_h,b_h) - \hat{\PP}^k_h(\cdot|s_h,a_h,b_h) \|_1\right) \Bigggiven \mathrm{br}(\nu^k), \nu^k, \hat{\PP}^k \right], 
\end{aligned}
\end{align}
where the first inequality is by the fact $\sup_{s\in\cS} \big|V^{\mathrm{br}(\nu^k), \nu^k}_{h+1}(s)\big| \leq H$. Next, we bound $\EE \big[ \sum_{h=1}^H  \|\PP_h(\cdot|s_h,a_h,b_h) \allowbreak- \hat{\PP}^k_h(\cdot|s_h,a_h,b_h) \|_1 \biggiven \mathrm{br}(\nu^k), \nu^k, \hat{\PP}^k \big]$. Note that for $\|\PP_h(\cdot|s_h,a_h,b_h) - \hat{\PP}^k_h(\cdot|s_h,a_h,b_h)  \|_1$, we have a trivial bound $\|\PP_h(\cdot|s_h,a_h,b_h) - \hat{\PP}^k_h(\cdot|s_h,a_h,b_h)\|_1 \leq 2$. Furthermore, by Lemma \ref{lem:expand1-mg}, we have
\begin{align*}
&\EE \left[ \sum_{h=1}^H \|\PP_h(\cdot|s_h,a_h,b_h) - \hat{\PP}^k_h(\cdot|s_h,a_h,b_h)\|_1 \Bigggiven\mathrm{br}(\nu^k), \nu^k, \hat{\PP}^k \right] \\
& \qquad = \sum_{h=1}^H \EE_{(s_h, a_h, b_h)\sim d_h^{\mathrm{br}(\nu^k), \nu^k, \hat{\PP}^k}(\cdot,\cdot,\cdot)} [  \|\PP_h(\cdot|s_h,a_h,b_h) - \hat{\PP}^k_h(\cdot|s_h,a_h,b_h) \|_1] \\
& \qquad= \sum_{h=2}^H \sqrt{8k  \zeta^k_{h-1} + 2k  |\cA||\cB|\xi_h^k + 4\lambda_k  d/(\CS)^2}  \EE_{d^{\mathrm{br}(\nu^k), \nu^k, \hat{\PP}^k}_{h-1}}\left\|\hat{\phi}^k_{h-1}\right\|_{\Sigma_{\tilde{\rho}^k_{h-1}, \hat{\phi}^k_{h-1}}^{-1}}  + \sqrt{|\cA||\cB|  \zeta_1^k },
\end{align*}
where the last equation is by the below definitions for all $(h,k)\in [H]\times[K]$,
\begin{align}
\begin{aligned} \label{eq:tran-err-def-mg}
&\zeta_h^k:=\EE_{(s,a,b)\sim \tilde{\rho}_h^k(\cdot,\cdot,\cdot)}[\|\PP_1(\cdot|s
,a,b) - \hat{\PP}^k_1(\cdot|s,a,b) \|_1^2],\\
&\xi_h^k:=\EE_{(s,a,b)\sim\breve{\rho}^k_h(\cdot,\cdot,\cdot)}[ \|\PP_h(\cdot|s,a,b) - \hat{\PP}^k_h(\cdot|s,a,b) \|_1^2],
\end{aligned}
\end{align}
whose upper bound will be characterized later. Thus, the above results imply that
\begin{small}
\begin{align*}
&\EE \left[ \sum_{h=1}^H H \|\PP_h(\cdot|s_h,a_h,b_h) - \hat{\PP}^k_h(\cdot|s_h,a_h,b_h) \|_1 \Bigggiven \mathrm{br}(\nu^k), \nu^k, \hat{\PP}^k \right] \\
&\leq \min\bigg\{ H\sqrt{|\cA||\cB|  \zeta_1^k } + \sum_{h=2}^H H\sqrt{8k  \zeta^k_{h-1} + 2k  |\cA||\cB|\xi_h^k + 4\lambda_k  d/(\CS)^2}\cdot \EE_{d^{\mathrm{br}(\nu^k), \nu^k, \hat{\PP}^k}_{h-1}}\left\|\hat{\phi}^k_{h-1}\right\|_{\Sigma_{\tilde{\rho}^k_{h-1}, \hat{\phi}^k_{h-1}}^{-1}}  , ~~ 2H^2\bigg\}.
\end{align*}
\end{small}
On the other hand, we bound $\EE [ \sum_{h=1}^H -\beta^k_h(s_h,a_h,b_h)  \given \mathrm{br}(\nu^k), \nu^k, \hat{\PP}^k ]$ in \eqref{eq:term-i-decomp-mg}. We have 
\begin{align}
\begin{aligned}\label{eq:bonus-mg-1}
&\EE \left[ \sum_{h=1}^H -\beta^k_h(s_h,a_h,b_h)  \bigggiven \mathrm{br}(\nu^k), \nu^k, \hat{\PP}^k \right]\\ 
&\qquad  =\EE \left[ \sum_{h=1}^H - \min\{\gamma_k \|\hat{\phi}^k_h(s_h,a_h,b_h)\|_{(\hat{\Sigma}_h^k)^{-1}}, 2H\} \bigggiven \mathrm{br}(\nu^k), \nu^k, \hat{\PP}^k \right]\\
&\qquad  \leq\EE \left[ \sum_{h=1}^H - \min\left\{\frac{3}{5}\gamma_k \|\hat{\phi}^k_h(s_h,a_h,b_h)\|_{\Sigma_{\tilde{\rho}^k_h, \hat{\phi}^k_h}^{-1}}, 2H\right\} \bigggiven \mathrm{br}(\nu^k), \nu^k, \hat{\PP}^k \right]\\
&\qquad  = - \min\left\{\frac{3}{5}\gamma_k \sum_{h=1}^H \EE_{ d^{\mathrm{br}(\nu^k), \nu^k, \hat{\PP}^k}_h}\|\hat{\phi}^k_h\|_{\Sigma_{\tilde{\rho}^k_h, \hat{\phi}^k_h}^{-1}}, 2H^2\right\} \\
&\qquad  \leq - \min\left\{\frac{3}{5}\gamma_k \sum_{h=1}^{H-1} \EE_{d^{\mathrm{br}(\nu^k), \nu^k, \hat{\PP}^k}_h}\|\hat{\phi}^k_h\|_{\Sigma_{\tilde{\rho}^k_h, \hat{\phi}^k_h}^{-1}}, 2H^2\right\} ,
\end{aligned}
\end{align} 
when  $\lambda_k\geq c_0 d \log(H|\Phi|k/\delta)$ with probability at least $1-\delta$. The first inequality is by Lemma \ref{lem:con-inverse} for all $h\in [H]$. Thus, plugging in the above results into \eqref{eq:term-i-decomp-mg}, for a sufficient large $c_0$, setting 
\begin{align}
\lambda_k\geq  c_0 d \log(H|\Phi|k/\delta), \qquad
\gamma_k\geq  \frac{5}{3}H\sqrt{8k  \zeta^k_{h-1} + 2k  |\cA||\cB|\xi_h^k + 4\lambda_k  d/(\CS)^2}, \label{eq:extra-condition-mg}
\end{align}
we have that
\begin{align}
(i) =  V_1^{\mathrm{br}(\nu^k), \nu^k}(s_1) - \overline{V}_1^{\mathrm{br}(\nu^k), \nu^k}(s_1)  &\leq \sqrt{|\cA||\cB|  \zeta_1^k }, \label{eq:near-optim-mg-1}
\end{align}
where the inequality is due to  $\min\{x+y,2H^2\}-\min\{y, 2H^2\} \leq x, \forall x,y\geq 0$. 

On the other hand, we prove the upper bound for term $(v)$. Specifically, by Lemma \ref{lem:diff1-mg}, we have
\begin{align}
\begin{aligned}\label{eq:term-v-decomp-mg}
   (v)&= \underline{V}_1^{\pi^k, \mathrm{br}(\pi^k)}(s_1)-V_1^{\pi^k, \mathrm{br}(\pi^k)}(s_1)\\
   &= \EE \left[ \sum_{h=1}^H \left(-\beta^k_h(s_h,a_h,b_h) - (\PP_h - \hat{\PP}^k_h )V^{\pi^k, \mathrm{br}(\pi^k)}_{h+1}(s_h,a_h,b_h)\right) \Bigggiven \pi^k, \mathrm{br}(\pi^k), \hat{\PP}^k \right]\\
    &\leq  \EE \left[ \sum_{h=1}^H \left(-\beta^k_h(s_h,a_h,b_h) + H \|\PP_h(\cdot|s_h,a_h,b_h) - \hat{\PP}^k_h(\cdot|s_h,a_h,b_h) \|_1\right) \Bigggiven \pi^k, \mathrm{br}(\pi^k), \hat{\PP}^k \right],
\end{aligned}
\end{align}
where the first inequality is by the fact $\sup_{s\in\cS} \big|V^{\pi^k, \mathrm{br}(\pi^k)}_{h+1}(s)\big| \leq H$. Next, for $\|\PP_h(\cdot|s_h,a_h,b_h) - \hat{\PP}^k_h(\cdot|s_h,a_h,b_h)  \|_1$, we have a trivial bound $\|\PP_h(\cdot|s_h,a_h,b_h) - \hat{\PP}^k_h(\cdot|s_h,a_h,b_h)\|_1 \leq 2$. In addition, by Lemma \ref{lem:expand1-mg}, we obtain
\begin{align*}
&\EE \left[ \sum_{h=1}^H \|\PP_h(\cdot|s_h,a_h,b_h) - \hat{\PP}^k_h(\cdot|s_h,a_h,b_h)\|_1 \Bigggiven\pi^k, \mathrm{br}(\pi^k), \hat{\PP}^k \right] \\
& = \sum_{h=1}^H \EE_{(s_h, a_h, b_h)\sim d_h^{\pi^k, \mathrm{br}(\pi^k), \hat{\PP}^k}(\cdot,\cdot,\cdot)} [  \|\PP_h(\cdot|s_h,a_h,b_h) - \hat{\PP}^k_h(\cdot|s_h,a_h,b_h) \|_1] \\
& = \sum_{h=2}^H \sqrt{8k  \zeta^k_{h-1} + 2k  |\cA||\cB|\xi_h^k + 4\lambda_k  d} \cdot \EE_{d^{\pi^k, \mathrm{br}(\pi^k), \hat{\PP}^k}_{h-1}}\left\|\hat{\phi}^k_{h-1}\right\|_{\Sigma_{\tilde{\rho}^k_{h-1}, \hat{\phi}^k_{h-1}}^{-1}}  + \sqrt{|\cA||\cB|  \zeta_1^k },
\end{align*}
where the last equation is by the definitions of $\zeta_h^k$ and $\xi_h^k$ in \eqref{eq:tran-err-def-mg}. Thus, the above results imply that
\begin{align*}
&\EE \left[ \sum_{h=1}^H H \|\PP_h(\cdot|s_h,a_h,b_h) - \hat{\PP}^k_h(\cdot|s_h,a_h,b_h) \|_1 \Bigggiven \pi^k, \mathrm{br}(\pi^k), \hat{\PP}^k \right] \\
& \leq \min\bigg\{ H\sqrt{|\cA||\cB|  \zeta_1^k } + \sum_{h=2}^H H\sqrt{8k  \zeta^k_{h-1} + 2k  |\cA||\cB|\xi_h^k + 4\lambda_k  d}\cdot \EE_{d^{\pi^k, \mathrm{br}(\pi^k), \hat{\PP}^k}_{h-1}}\left\|\hat{\phi}^k_{h-1}\right\|_{\Sigma_{\tilde{\rho}^k_{h-1}, \hat{\phi}^k_{h-1}}^{-1}}  , ~~ 2H^2\bigg\}.
\end{align*}
On the other hand, we bound $\EE [ \sum_{h=1}^H -\beta^k_h(s_h,a_h,b_h)  \given \pi^k, \mathrm{br}(\pi^k), \hat{\PP}^k ]$ in \eqref{eq:term-v-decomp-mg}. Analogous to \eqref{eq:bonus-mg-1}, we can obtain 
\begin{align*}
\EE \left[ \sum_{h=1}^H -\beta^k_h(s_h,a_h,b_h)  \bigggiven \pi^k, \mathrm{br}(\pi^k), \hat{\PP}^k \right]  \leq - \min\left\{\frac{3}{5}\gamma_k \sum_{h=1}^{H-1} \EE_{d^{\pi^k, \mathrm{br}(\pi^k), \hat{\PP}^k}_h}\|\hat{\phi}^k_h\|_{\Sigma_{\tilde{\rho}^k_h, \hat{\phi}^k_h}^{-1}}, 2H^2\right\} ,
\end{align*} 
when  $\lambda_k\geq c_0 d \log(H|\Phi|k/\delta)$ with probability at least $1-\delta$.  Thus, plugging in the above results into \eqref{eq:term-v-decomp-mg}, setting $\lambda_k$ and $\gamma_k$ as in \eqref{eq:extra-condition-mg},
we have
\begin{align}
(v) =  \underline{V}_1^{\pi^k, \mathrm{br}(\pi^k)}(s_1) - V_1^{\pi^k, \mathrm{br}(\pi^k)}(s_1)  &\leq \sqrt{|\cA||\cB|  \zeta_1^k }, \label{eq:near-optim-mg-2}
\end{align}
where the inequality is due to  $\min\{x+y,2H^2\}-\min\{y, 2H^2\} \leq x, \forall x,y\geq 0$. 

Now we have proved the near-optimism and near-pessimism in \eqref{eq:near-optim-mg-1} and \eqref{eq:near-optim-mg-2} respectively, which extends the related result for single-agent MDPs.

Next, we show the upper bound of the term $(iii)$ in \eqref{eq:decomp-mg-init}. By Lemma \ref{lem:diff2-mg}, we have
\begin{small}
\begin{align}
\begin{aligned}\label{eq:term-iii-decomp}
(iii) &= \overline{V}_1^k(s_1) - \underline{V}_1^k(s_1)  =   \EE \left[ \sum_{h=1}^H 2 \beta^k_h(s_h,a_h,b_h) + (\hat{\PP}_h^k - \PP_h)  \big(\overline{V}^k_{h+1} - \underline{V}^k_{h+1} \big) (s_h,a_h,b_h)\Bigggiven \sigma^k, \PP \right]\\
&\leq  2\sum_{h=1}^H \EE_{(s,a,b)\sim d_h^{\sigma^k, \PP}(\cdot,\cdot,\cdot)}  [\beta^k_h(s,a,b) ] +  6H^2\sum_{h=1}^H \EE_{(s,a,b)\sim d_h^{\sigma^k, \PP}(\cdot,\cdot,\cdot)}  [\|\PP^k_h(\cdot|s,a,b) - \PP_h(\cdot|s,a,b)\|_1 ]
\end{aligned}
\end{align}
\end{small}
where the above inequality is due to $\sup_{s\in\cS}|\overline{V}^k_h(s)| \leq H(1+2H)\leq 3H^2$ and $\sup_{s\in\cS}|\underline{V}^k_h(s)| \leq H(1+2H)\leq 3H^2$. By Lemma \ref{lem:expand2-mg}, since $\sup_{s\in\cS, a\in\cA, b\in\cB}\beta_h^k(s,a,b)\leq 2H$ according to the definition of $\beta_h^k$ in Algorithm \ref{alg:contrastive-mg}, we have
\begin{align*}
&\sum_{h=1}^H \EE_{(s,a,b)\sim d_h^{\sigma^k, \PP}(\cdot,\cdot,\cdot)} [ \beta^k_h(s,a,b) ]\\
&\leq   \sqrt{|\cA||\cB|  \EE_{(a,b)\sim \tilde{\rho}_1^k(s_1,\cdot,\cdot)}[\beta^k_1(s_1,a,b)^2] } \\
&\quad + \sum_{h=2}^H \sqrt{k  |\cA||\cB|  \EE_{(s,a,b)\sim\tilde{\rho}^k_h(\cdot,\cdot,\cdot)}[ \beta^k_h(s,a,b)^2] + 4H^2\lambda_k  d} ~ \EE_{d^{\sigma^k, \PP}_{h-1}}\left\|\phi^*_{h-1}\right\|_{\Sigma_{\rho^k_{h-1}, \phi^*_{h-1}}^{-1}}\\
&\leq   \sqrt{|\cA||\cB|\gamma_k^2  \EE_{(a,b)\sim \tilde{\rho}_1^k(s_1,\cdot,\cdot)} \|\hat{\phi}^k_1(s_1,a,b)\|_{(\hat{\Sigma}_1^k)^{-1}}^2 } \\
&\quad + \sum_{h=2}^H \sqrt{k  |\cA||\cB| \gamma_k^2 \EE_{(s,a,b)\sim\tilde{\rho}^k_h(\cdot,\cdot,\cdot)} \|\hat{\phi}^k_h(s,a,b)\|_{(\hat{\Sigma}_h^k)^{-1}}^2 + 4H^2\lambda_k  d} ~ \EE_{d^{\sigma^k, \PP}_{h-1}}\left\|\phi^*_{h-1}\right\|_{\Sigma_{\rho^k_{h-1}, \phi^*_{h-1}}^{-1}}
,
\end{align*}
where the second inequality is due to $\beta_h^k(s,a,b) \leq \|\hat{\phi}^k_h(s,a,b)\|_{(\hat{\Sigma}_h^k)^{-1}}$. Furthermore, we have that with $\lambda_k \geq c_0 d\log(H|\Phi|k/\delta)$, with probability at least $1-\delta$, for all $h\in [H]$, 
\begin{align*}
\EE_{(s,a,b)\sim\tilde{\rho}^k_h(\cdot,\cdot,\cdot)} \|\hat{\phi}^k_h(s,a,b)\|_{(\hat{\Sigma}_h^k)^{-1}}^2 &\leq 3\EE_{(s,a,b)\sim\tilde{\rho}^k_h(\cdot,\cdot,\cdot)} \|\hat{\phi}^k_h(s,a,b)\|_{\Sigma_{\tilde{\rho}^k_h, \hat{\phi}^k_h}^{-1}}^2\\
&=3\EE_{\tilde{\rho}^k_h}\left[ \hat{\phi}^k_h{}^\top \left(k\EE_{\tilde{\rho}^k_h}[ \hat{\phi}^k_h (\hat{\phi}^k_h)^\top] + \lambda_k I \right)^{-1}\hat{\phi}^k_h \right]\\
&=\frac{3}{k}\tr\left\{ k\EE_{\tilde{\rho}^k_h}[ \hat{\phi}^k_h\hat{\phi}^k_h{}^\top] \left(k\EE_{\tilde{\rho}^k_h}[ \hat{\phi}^k_h (\hat{\phi}^k_h)^\top] + \lambda_k I \right)^{-1}  \right\}\leq \frac{3}{k}\tr(I) = \frac{3d}{k},
\end{align*}
where the first inequality is by Lemma \ref{lem:con-inverse}. Combining the above results, we have the following inequality holds with probability at least $1-\delta$,
\begin{align}
\begin{aligned}\label{eq:term-iii-decomp-1}
&\sum_{h=1}^H \EE_{(s,a,b)\sim d_h^{\sigma^k, \PP}(\cdot,\cdot,\cdot)} [ \beta^k_h(s,a,b) ]\\
&\qquad \leq   \sqrt{3d|\cA||\cB|\gamma_k^2  / k } + \sum_{h=2}^H \sqrt{3d  |\cA||\cB| \gamma_k^2  + 4H^2\lambda_k  d} ~ \EE_{d^{\sigma^k, \PP}_{h-1}}\left\|\phi^*_{h-1}\right\|_{\Sigma_{\rho^k_{h-1}, \phi^*_{h-1}}^{-1}}.
\end{aligned}
\end{align}
Further by Lemma \ref{lem:expand2-mg}, due to $\|\PP_h(\cdot|s,a,b) - \hat{\PP}^k_h(\cdot|s,a,b) \|_1 \leq 2$, we have
\begin{align}
&\sum_{h=1}^H \EE_{(s,a,b)\sim d_h^{\sigma^k, \PP}(\cdot,\cdot,\cdot)} [ \|\PP_h(\cdot|s,a,b) - \hat{\PP}^k_h(\cdot|s,a,b) \|_1 ]\nonumber\\
&\quad =   \sqrt{|\cA| |\cB| \zeta_1^k }  + \sum_{h=2}^H \sqrt{k  |\cA| |\cB| \zeta_h^k + 4\lambda_k  d} ~ \EE_{d^{\sigma^k, \PP}_{h-1}}\left\|\phi^*_{h-1}\right\|_{\Sigma_{\rho^k_{h-1}, \phi^*_{h-1}}^{-1}}. \label{eq:term-iii-decomp-2}
\end{align}
Therefore, combining \eqref{eq:term-iii-decomp}, \eqref{eq:term-iii-decomp-1}, and \eqref{eq:term-iii-decomp-2}, we obtain 
\begin{align}
(iii) &\leq \left( 2\sqrt{3d|\cA||\cB|\gamma_k^2  / k }  + 6H^2 \sqrt{|\cA| |\cB| \zeta_1^k } \right) \label{eq:term-iii-bound} \\
&\quad  + \sum_{h=2}^H  \left( 2\sqrt{3d  |\cA||\cB| \gamma_k^2  + 4H^2\lambda_k  d} +  6H^2\sqrt{k  |\cA| |\cB| \zeta_h^k + 4\lambda_k  d} \right)~ \EE_{ d^{\sigma^k, \PP}_{h-1}}\left\|\phi^*_{h-1}\right\|_{\Sigma_{\rho^k_{h-1}, \phi^*_{h-1}}^{-1}}.\nonumber
\end{align}
We characterize the upper bound of $\zeta_h^k$ and $\xi_h^k$ as defined in \eqref{eq:tran-err-def-mg}.  According to Lemma \ref{lem:stat-err-mg}, we have with probability at least $1-2\delta$,
\begin{align}
\begin{aligned}\label{eq:stat-err-bound-mg}
&\zeta_h^k \leq  32d/(\CS)^2\cdot\log (2kH|\cF|/\delta)/k, \quad \forall h\geq 1,\\
& \xi_h^k  \leq 32d/(\CS)^2\cdot\log (2kH|\cF|/\delta)/k, \quad \forall h\geq 2,
\end{aligned}
\end{align}
Plugging \eqref{eq:stat-err-bound-mg} and \eqref{eq:extra-condition-mg} into \eqref{eq:near-optim-mg-1},\eqref{eq:near-optim-mg-2}, and \eqref{eq:term-iii-bound}, we obtain
\begin{align*}
&(i) =  V_1^{\mathrm{br}(\nu^k), \nu^k}(s_1) - \overline{V}_{k,1}^{\mathrm{br}(\nu^k), \nu^k}(s_1)  \lesssim \sqrt{d|\cA||\cB| /(\CS)^2\cdot\log (KH|\cF|/\delta)/k }, \\
&(v) =  \underline{V}_{k,1}^{\pi^k, \mathrm{br}(\pi^k)}(s_1) - V_1^{\pi^k, \mathrm{br}(\pi^k)}(s_1) \lesssim \sqrt{d|\cA||\cB| /(\CS)^2\cdot\log (KH|\cF|/\delta)/k }, \\
&(iii) =   \overline{V}_1^k(s_1) - \underline{V}_1^k(s_1) \lesssim \sqrt{C_1 \log(H|\cF|K/\delta)/k}  +  \sqrt{(C_1 +C_2)\log(H|\cF|K/\delta)}\sum_{h=1}^{H-1} \EE_{ d^{\sigma^k, \PP}_h}\left\|\phi^*_h \right\|_{\Sigma_{\rho^k_h, \phi^*_h}^{-1}},
\end{align*}
where we let $C_1 = H^2d^3|\cA||\cB|/(\CS)^2 +  H^2d^2|\cA|^2|\cB|^2/(\CS)^2+  H^4d|\cA||\cB|/(\CS)^2 $ and $C_2 = H^4d^2$. Further by \eqref{eq:decomp-mg-init-2}, we have
\begin{align*}
\frac{1}{K}\sum_{k=1}^K \left[ V_1^{\mathrm{br}(\nu^k), \nu^k}(s_1)- V_1^{\pi^k, \mathrm{br}(\pi^k)}(s_1)\right]&\lesssim    \sqrt{(C_1 +C_2)\log(H|\cF|K/\delta)}/K\cdot \sum_{h=1}^{H-1} \sum_{k=1}^K\EE_{ d^{\sigma^k, \PP}_h}\left\|\phi^*_h\right\|_{\Sigma_{\rho^k_h, \phi^*_h}^{-1}}\\
&\quad +\sqrt{C_1 \log(H|\cF|K/\delta)/K} + \frac{H}{K}\sum_{k=1}^K\iota_k.
\end{align*}
Moreover, we have
\begin{align*}
&\frac{1}{K}\sum_{k=1}^K\EE_{(s,a,b)\sim d^{\sigma^k, \PP}_h(\cdot,\cdot,\cdot)}\left\|\phi^*_h(s,a,b)\right\|_{\Sigma_{\rho^k_h, \phi^*_h}^{-1}} \\
&\qquad\leq \sqrt{\frac{1}{K}\sum_{k=1}^K\EE_{(s,a,b)\sim d^{\sigma^k, \PP}_h(\cdot,\cdot,\cdot)}\left\|\phi^*_h(s,a,b)\right\|^2_{\Sigma_{\rho^k_h, \phi^*_h}^{-1}}}\\
&\qquad= \sqrt{\frac{1}{K}\sum_{k=1}^K\tr\left(\EE_{(s,a,b)\sim d^{\sigma^k, \PP}_h(\cdot,\cdot,\cdot)}\left(\phi^*_h(s,a,b)\phi^*_h(s,a,b)^\top\right)\Sigma_{\rho^k_h, \phi^*_h}^{-1}\right) }\\
&\qquad\leq \sqrt{d \log(1+kd/\lambda_k)/K} \leq \sqrt{d \log(1+c_1 K)/K}.
\end{align*}
where the first inequality is by Jensen's inequality and the second inequality is by Lemma \ref{lem:logdet-tele} with $c_1$ being some absolute constant. Thus, we have
\begin{align*}
\frac{1}{K}\sum_{k=1}^K \left[ V_1^{\mathrm{br}(\nu^k), \nu^k}(s_1)- V_1^{\pi^k, \mathrm{br}(\pi^k)}(s_1)\right]&\lesssim   \sqrt{H^2/K} + \sqrt{H^2d(C_1 +C_2)\log(H|\cF|K/\delta)  \log(1+c_1 K)/K }.
\end{align*}
Taking union bound for all events in this proof, due to $|\cF| \geq |\Phi|$, letting 
\begin{align*}
\lambda_k= c_0 d \log(H|\cF|k/\delta), \quad
\gamma_k=  4H \big( 12\sqrt{  |\cA||\cB|d} + \sqrt{c_0} d\big)/\CS\cdot \sqrt{\log (2Hk|\cF|/\delta)}, \quad \iota_k \leq \cO( \sqrt{1/k}),
\end{align*}
we have with probability at least $1-3\delta$,
\begin{align*}
\frac{1}{K}\sum_{k=1}^K \left[ V_1^{\mathrm{br}(\nu^k), \nu^k}(s_1)- V_1^{\pi^k, \mathrm{br}(\pi^k)}(s_1)\right]\lesssim \sqrt{C\log(H|\cF|K/\delta) \log(c_0' K)/K },
\end{align*}
where $C = H^4d^4|\cA||\cB|/(\CS)^2 +  H^4d^3|\cA|^2|\cB|^2/(\CS)^2+  H^6d^2|\cA||\cB|/(\CS)^2 + H^6d^3$ and $c_0,c_0'$ are absolute constants.
This completes the proof.
\end{proof}

\section{Other Supporting Lemmas}
\begin{lemma}[Concentration of Inverse Covariances \citep{zanette2021cautiously}]\label{lem:con-inverse} Let $\mu_i$ be the conditional distribution of $\phi$ given the sampled $\phi_1, \cdots, \phi_{i-1}$ with $\|\phi_i\|_2 \leq 1$ holding for $\phi_i$ as the realization of $\phi$. Let $\Lambda = \frac{1}{k}\sum_{i=1}^k\EE_{\phi\sim \mu_i}[\phi \phi^\top]$. Then there exists an absolute constant $c_0>0$. If $\lambda \geq c_0 d \log(|\Phi| k/\delta)$, we have with probability at least $1-\delta$, for all $k\geq 1$,
\begin{align*}
\frac{3}{5} (k\Lambda + \lambda I)^{-1} \preceq \left( \sum_{i=1}^k \phi_i \phi_i^\top + \lambda I \right)^{-1} \preceq 3 (k\Lambda + \lambda I)^{-1}.
\end{align*}
\end{lemma}
\begin{proof}
The proof of this lemma is adapted from Lemma 39 in \citet{zanette2021cautiously}.  Further applying Lemma 39 of \citet{zanette2021cautiously} to all the elements in the function class $\Phi$, we obtain Lemma \ref{lem:con-inverse}. This completes the proof.
\end{proof}

\begin{lemma}[\citet{agarwal2020flambe}]\label{lem:recover-mle} Let $\cF$ be a function class with $|\cF| < \infty$ and $f^* \in \cF$ where 
\begin{align*}
f^*(x, z) = P^*(z|x) 
\end{align*} 
is some conditional distribution. Given a dataset $\cD := \{(x_i,z_i)\}_{i=0}^{k-1}$, let $\cT_i$ be some distribution that is dependent on $\{(x_{i'},z_{i'})\}_{i'=0}^{i-1}$ for all $i \leq k$. Suppose $x_i\sim\cT_i$ and $z_i\sim P^*(\cdot|x) = f^*(x,\cdot)$ for all $i \leq k$. Then, we have with probability at least $1-\delta$,
\begin{align*}
\sum_{i=0}^{k-1}  \EE_{x\sim \cT_i} \|\hat{f}(x,\cdot) - f^*(x,\cdot) \|_{\TV}^2 \leq 2\log (k|\cF|/\delta),
\end{align*}
where
\begin{align*}
\hat{f}:=\argmax_{f\in \cF} \sum_{(x,z)\in \cD} \log f(x, z).
\end{align*}
\end{lemma}

\begin{lemma}[\citet{uehara2021representation,jin2020provably}]\label{lem:logdet-tele}  For $i = 1, \ldots , k$, $\Sigma_i := \Sigma_{i-1} + G_i$ where $\Sigma_0 = \lambda I$ with $\lambda > 0$ and $G_i\in \RR^{d\times d}$ is
a positive semidefinite matrix with eigenvalues upper bounded by $1$ and $\tr(G_i) \leq  C^2$ for some $C>0$. Then, we have the following inequality
\begin{align*}
\sum_{i=1}^k \tr(G_i\Sigma_{i-1}^{-1}) \leq 2 \log\det(\Sigma_k)  - 2 \log\det(\lambda I)  \leq d \log(1 + kC^2d/\lambda).
\end{align*} 
\end{lemma}

\section{Additional Experimental Results}\label{app:exp}
In this section, we present the additional experimental results. In Table \ref{tab:res}, we report the human normalized scores for all the algorithms under all the tasks of Atari 100K. In Figure \ref{fig:aggregates}, we follow \citet{agarwal2021deep} and report the stratified bootstrap of experiments, which consists of the $95\%$ confidence intervals (CIs) of median, interquartile mean (IQM), mean, and optimality gap, over the $26$ Atari 100K tasks. Here IQM is the $25\%$ trimmed mean obtained by discarding the top and bottom $25\%$ score and calculating the mean. See \citet{agarwal2021deep} for details.
According to Figure \ref{fig:aggregates}, our proposed SPR-UCB performs similarly to SPR on average, without the top $5\%$ scores. Nevertheless, we remark that SPR-UCB outperforms SPR significantly on some hard exploration tasks \citep{Taiga2020On}, including \emph{PrivateEye}, \emph{Frostbite}, and \emph{Freeway}, as shown in Table \ref{tab:res}.

\begin{figure}[h]
\centering
\includegraphics[width=1.05\columnwidth]{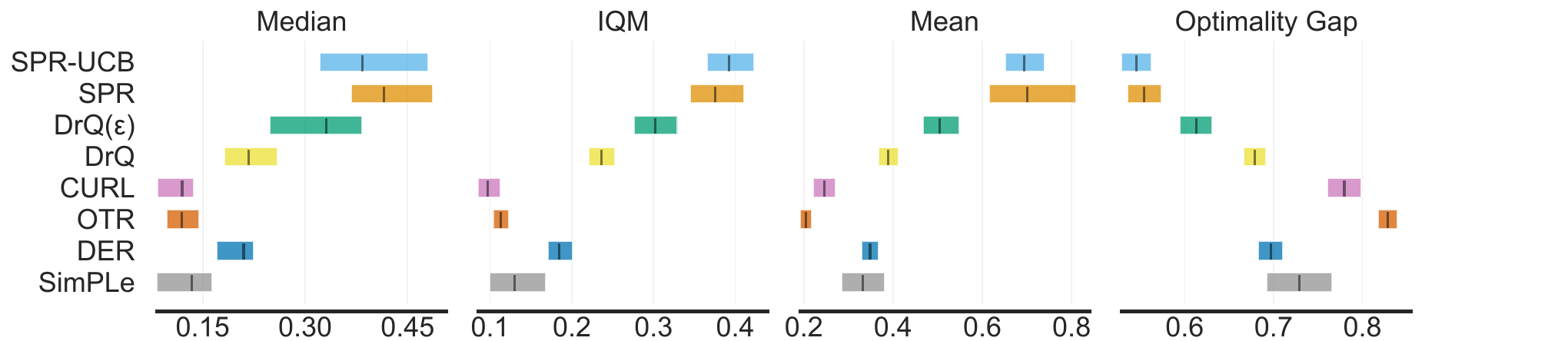}
\vspace{-0.6cm}
\caption{Stratified Bootstrap \citep{agarwal2021deep} of experiments, with $95\%$ confidence intervals (CIs) based on $26$ Atari 100K tasks. Higher mean, median, interquartile mean (IQM), and lower optimality gap a better. See \citet{agarwal2021deep} for details. The results for baseline algorithms are collected from the report by \citet{agarwal2021deep}. The results for SPR-UCB  are based on $10$ runs per game.} 
\label{fig:aggregates}
\end{figure}

\begin{table}[h]
\vspace{-0.39cm}
\caption{Table of the comparison of human normalized scores over tasks of Atari 100K. The scores of baselines are adopted from \citet{agarwal2021deep}, which runs each method over 100 seeds. We follow \citet{agarwal2021deep} and evaluate the scores of SPR-UCB by evaluating the final policy obtained by SPR-UCB over 100 episodes. Highlighted scores are the highest and second highest among all algorithms.}
\label{tab:res}
\vskip 0.15in
\begin{center}
\begin{footnotesize}
\begin{sc}
\begin{tabular}{lllllllr@{\hspace{-0.5pt}}lr@{\hspace{-0.5pt}}l}
\toprule
               & CURL            & OTR             & DER             & SimPLe          & DrQ             & DrQ($\epsilon$)    & \multicolumn{2}{c}{SPR}       & \multicolumn{2}{c}{\textbf{SPR-UCB}}  \\
\midrule
Alien          & 0.0700          & 0.0497          & 0.0833          & 0.0564          & 0.0734          & \textbf{0.0924} & 0.0890          & {\color[HTML]{525252} $\pm$0.03}          & \textbf{0.0997} & {\color[HTML]{525252} $\pm$0.02}        \\
Amidar         & 0.0630          & 0.0420          & 0.0701          & 0.0399          & 0.0516          & 0.0770          & \textbf{0.1015} & {\color[HTML]{525252} $\pm$0.02}          & \textbf{0.0973} & {\color[HTML]{525252} $\pm$0.02}        \\
Assault        & 0.5360          & 0.2088          & 0.6525          & 0.5866          & 0.4949          & \textbf{0.6875} & 0.6605          & {\color[HTML]{525252} $\pm$0.11}          & \textbf{0.6729} & {\color[HTML]{525252} $\pm$0.07}        \\
Asterix        & 0.0431          & 0.0150          & 0.0392          & \textbf{0.1107} & 0.0393          & 0.0668          & 0.0907          & {\color[HTML]{525252} $\pm$0.02}          & \textbf{0.0965} & {\color[HTML]{525252} $\pm$0.01}        \\
BankHeist      & 0.0692          & 0.0552          & 0.2318          & 0.0271          & 0.1884          & 0.2960          & \textbf{0.4483} & {\color[HTML]{525252} $\pm$0.29}          & \textbf{0.3011} & {\color[HTML]{525252} $\pm$0.36}        \\
BattleZone     & 0.1906          & 0.0798          & 0.1900          & 0.0480          & 0.2355          & 0.2241          & \textbf{0.3582} & {\color[HTML]{525252} $\pm$0.14}          & \textbf{0.3663} & {\color[HTML]{525252} $\pm$0.09}        \\
Boxing         & 0.0708          & 0.1284          & -0.0340         & 0.6375          & 0.5443          & 0.7452          & \textbf{2.9667} & {\color[HTML]{525252} $\pm$1.19}          & \textbf{3.4332} & {\color[HTML]{525252} $\pm$0.94}        \\
Breakout       & 0.0297          & 0.2216          & 0.2609          & 0.5099          & 0.4759          & \textbf{0.6272} & 0.6208          & {\color[HTML]{525252} $\pm$0.46}          & \textbf{0.7245} & {\color[HTML]{525252} $\pm$0.47}        \\
ChopperCommand & -0.0042         & 0.0003          & 0.0175          & \textbf{0.0256} & -0.0028         & 0.0051          & \textbf{0.0206} & {\color[HTML]{525252} $\pm$0.04}          & 0.0041          & {\color[HTML]{525252} $\pm$0.04}        \\
CrazyClimber   & -0.0649         & 0.1684          & 0.9473          & \textbf{2.0681} & 0.4476          & 0.4295          & 1.0348          & {\color[HTML]{525252} $\pm$0.48}          & \textbf{1.2936} & {\color[HTML]{525252} $\pm$0.62}        \\
DemonAttack    & 0.2718          & 0.2911          & 0.2614          & 0.0308          & \textbf{0.5445} & \textbf{0.6429} & 0.2010          & {\color[HTML]{525252} $\pm$0.07}          & 0.2214          & {\color[HTML]{525252} $\pm$0.10}        \\
Freeway        & \textbf{0.9550} & 0.3877          & 0.7046          & 0.5637          & 0.6006          & 0.6843          & 0.6512          & {\color[HTML]{525252} $\pm$0.47}          & \textbf{0.9592} & {\color[HTML]{525252} $\pm$0.11}        \\
Frostbite      & \textbf{0.2720} & 0.0374          & 0.1887          & 0.0402          & 0.1037          & 0.2223          & 0.2589          & {\color[HTML]{525252} $\pm$0.26}          & \textbf{0.5591} & {\color[HTML]{525252} $\pm$0.15}        \\
Gopher         & 0.0665          & 0.1308          & 0.0972          & 0.1574          & 0.1673          & \textbf{0.1689} & \textbf{0.1870} & {\color[HTML]{525252} $\pm$0.11}          & 0.1666          & {\color[HTML]{525252} $\pm$0.05}        \\
Hero           & 0.1329          & 0.1654          & \textbf{0.1745} & 0.0547          & 0.0905          & 0.1054          & 0.1621          & {\color[HTML]{525252} $\pm$0.07}          & \textbf{0.2096} & {\color[HTML]{525252} $\pm$0.09}        \\
Jamesbond      & 1.1032          & 0.2156          & 0.9009          & 0.2610          & 0.8136          & 1.1691          & \textbf{1.2326} & {\color[HTML]{525252} $\pm$0.23}          & \textbf{1.2124} & {\color[HTML]{525252} $\pm$0.20}        \\
Kangaroo       & 0.2307          & 0.0994          & 0.1776          & -0.0003         & 0.3092          & 0.3474          & \textbf{1.1952} & {\color[HTML]{525252} $\pm$1.08}          & \textbf{1.0553} & {\color[HTML]{525252} $\pm$0.96}        \\
Krull          & 1.3595          & 1.9278          & 1.5540          & 0.5685          & 2.3732          & \textbf{2.6268} & 1.9519          & {\color[HTML]{525252} $\pm$0.43}          & \textbf{2.4225} & {\color[HTML]{525252} $\pm$0.23}        \\
KungFuMaster   & 0.3513          & 0.2848          & 0.2812          & \textbf{0.6497} & 0.3068          & 0.4987          & 0.6462          & {\color[HTML]{525252} $\pm$0.32}          & \textbf{0.8126} & {\color[HTML]{525252} $\pm$0.27}        \\
MsPacman       & 0.1139          & 0.0904          & 0.1325          & \textbf{0.1765} & 0.1047          & 0.1371          & 0.1522          & {\color[HTML]{525252} $\pm$0.05}          & \textbf{0.1557} & {\color[HTML]{525252} $\pm$0.06}        \\
Pong           & 0.0627          & \textbf{0.5168} & 0.3113          & \textbf{0.9495} & 0.1827          & 0.3298          & 0.4331          & {\color[HTML]{525252} $\pm$0.30}          & 0.4007          & {\color[HTML]{525252} $\pm$0.23}        \\
PrivateEye     & 0.0008          & 0.0005          & 0.0007          & 0.0001          & 0.0000          & -0.0003         & \textbf{0.0009} & {\color[HTML]{525252} $\pm$0.00}          & \textbf{0.0011} & {\color[HTML]{525252} $\pm$0.00}        \\
Qbert          & 0.0424          & 0.0292          & \textbf{0.1211} & 0.0846          & 0.0580          & \textbf{0.1239} & 0.0528          & {\color[HTML]{525252} $\pm$0.04}          & 0.0606          & {\color[HTML]{525252} $\pm$0.04}        \\
RoadRunner     & 0.6376          & 0.3313          & 1.5104          & 0.7186          & 1.1123          & 1.4297          & \textbf{1.5576} & {\color[HTML]{525252} $\pm$0.64}          & \textbf{2.0051} & {\color[HTML]{525252} $\pm$0.53}        \\
Seaquest       & 0.0059          & 0.0049          & 0.0056          & \textbf{0.0146} & 0.0058          & 0.0068          & 0.0117          & {\color[HTML]{525252} $\pm$0.00}          & \textbf{0.0134} & {\color[HTML]{525252} $\pm$0.00}        \\
UpNDown        & 0.1893          & 0.1611          & 0.2277          & 0.2524          & 0.2765          & 0.3397          & \textbf{0.9253} & {\color[HTML]{525252} $\pm$1.44}          & \textbf{0.6941} & {\color[HTML]{525252} $\pm$0.59}        \\
\midrule
\textbf{Average}           & 0.2615          & 0.2171          & 0.3503          & 0.3320          & 0.3691          & 0.4647          & \textbf{0.6158} & {\color[HTML]{525252} $\pm$0.32} & \textbf{0.6938} & {\color[HTML]{525252} $\pm$0.24}\\
\bottomrule
\end{tabular}
\end{sc}
\end{footnotesize}
\end{center}
\vskip -0.1in
\end{table}

\end{appendices}

\end{document}